\documentclass{article}

\usepackage[final]{arxiv}

\usepackage[utf8]{inputenc} % allow utf-8 input
\usepackage[T1]{fontenc}    % use 8-bit T1 fonts
\usepackage{hyperref}       % hyperlinks
\usepackage{url}            % simple URL typesetting
\usepackage{booktabs}       % professional-quality tables
\usepackage{amssymb}
\usepackage{algpseudocode}
\usepackage{algorithm}
\usepackage{amsfonts}       % blackboard math symbols
\usepackage{nicefrac}       % compact symbols for 1/2, etc.
\usepackage{microtype}      % microtypography
\usepackage{xcolor}         % colors
\usepackage[table]{xcolor}
\usepackage{graphicx}
\usepackage{booktabs}
\usepackage{multirow}
\usepackage{diagbox}
\usepackage{makecell}
\usepackage{amsmath, amssymb}

\usepackage{amsthm}
\newtheorem{theorem}{Theorem}[section]
\newtheorem{corollary}{Corollary}[theorem]
\newtheorem{lemma}[theorem]{Lemma}
\newtheorem{proposition}[theorem]{Proposition}
\newtheorem{remark}[theorem]{Remark}

\theoremstyle{remark}

\title{A Principled Self-Referenced Early Stopping Approach for Deep Image Prior}

\author{Chaoyan Huang\textsuperscript{1,2}, Cheng-Han Huang\textsuperscript{1}, Ismail R. Alkhouri\textsuperscript{3,4}, ~Rongrong Wang\textsuperscript{1,5} \\ 
  \textsuperscript{1}Department of Computational Mathematics, Science, \& Engineering, Michigan State University \\
  \textsuperscript{2}Department of Electrical Engineering and Computer Science, University of Michigan\\
  \textsuperscript{3}X Computational Physics Division, Los Alamos National Laboratory \\
 \textsuperscript{4}Michigan Institute for Computational Discovery \& Engineering, University of Michigan\\
 \textsuperscript{5}Mathematical Sciences, Michigan State University \\
}

\begin{document}
\maketitle

\begin{abstract}

Recently, Deep Image Prior (DIP) has demonstrated strong capabilities for solving inverse imaging problems (IIPs) by optimizing a randomly initialized convolutional neural network in a training-data-free regime. However, DIP suffers from overfitting to noisy measurements due to network over-parameterization, making early stopping (ES) essential. The most successful ES method tracks fluctuations in the running variance of the network output to detect overfitting. However, in many applications, these fluctuations may appear prematurely, leading to unstable reconstructions. In this paper, we first show that nearly optimal DIP early stopping can be achieved when two independent noisy copies of the degraded image are available. Motivated by this observation, and since obtaining two fully independent copies is infeasible, we propose an overfitting detection framework based on constructing pseudo self-referenced images, resulting in three IIP-specific algorithms. Our approach is further supported by theoretical results on single-reference validation, pseudo-validation estimation, and the impact of shared noise. Across different IIPs, ranging from natural image restoration to medical image reconstruction, and under varying noise levels and noise types, our methods consistently outperform existing DIP early stopping approaches, all without requiring an accurate estimate of the noise level. %without requiring knowledge of the measurement noise characteristics. %\textcolor{red}{broad applicability...}

%based on channel similarity reference, mask reference, and augmented-channel reference. 
%, demonstrating the effectiveness of the proposed framework.

  % Recently, Deep Image Prior (DIP) has demonstrated strong capabilities for imaging inverse problems (IIPs) by optimizing a randomly initialized convolutional neural network in a training-data-free regime. However, DIP suffers from overfitting to noisy measurements due to network over-parameterization. One approach to mitigate this issue is early stopping (ES), which relies on the running variance of the network output to define an ``inexact'' stopping criterion. This approach has two limitations: (i) the stopping rule is inexact, and (ii) compared to regularization methods, ES may limit exploration of the search space. To address these limitations, we propose the Normalized Residual Spectral Moment (NRSM), a metric that can exactly detect overfitting. Leveraging NRSM, computed at each iteration, we introduce the Fourier-Aware Resetting Algorithm (FARA-DIP), which employs a simple yet effective input resetting strategy to improve exploration. FARA-DIP is computationally efficient, requiring neither network reparameterization nor explicit regularization, introduces fewer hyper-parameters (as resets are automatically determined via spectral detection), and only requires a frequency transform of the network output per iteration. Across two linear IIPs (denoising and MRI) and one non-linear IIP (non-linear deblurring), FARA-DIP achieves competitive performance in both reconstruction quality and runtime.
\end{abstract}

%%%%%%%%%%%%%%%%%%%%%%%%%%%%%%%%%%%%%%%%%%%%%%%%%%%%%%%%%%%%%%%%%%%%%%%%%%%%%%%%%%%%%%%%%%%%%%%%%%%%%%%%%%%%%%%%%%%%%%%%
\section{Introduction}

Inverse Imaging Problems (IIPs) appear in several science and engineering applications, including medical imaging \citep{chung2024decomposed,arridge2019solving}, remote sensing \citep{levis2022gravitationally,zhu2017deep}, and computer vision \cite{wang2024dmplug,zhang2017beyond}. IIPs are often ill-posed \cite{engl1996regularization,mccann2017convolutional,bertero1998introduction}, making exact recovery challenging and requiring advanced reconstruction techniques.

In recent years, several neural network (NN)-based approaches have been developed for solving IIPs \cite{mccann2017convolutional}. Depending on the availability and type of training data, these methods can be broadly categorized into (\textit{i}) training-data-intensive methods, including supervised \cite{Yiasemis_2022_CVPR}, self-supervised \cite{liu2020rare}, and generative approaches \cite{chung2025diffusion}; and (\textit{ii}) training-free methods such as Implicit Neural Representations (INRs) \cite{essakinewe} and the Deep Image Prior (DIP) \cite{ulyanov2018deep}. Training-data-free methods operate without pre-trained models, making them particularly useful in scenarios where data is scarce, severely degraded, or unlabeled, preventing supervised or generative approaches from providing reliable priors.

INRs provide continuous neural representations of signals and have been successfully applied to several reconstruction tasks \cite{shen2022nerp, vyas2025learning,de2023deep,tewari2022advances,molaei2023implicit}. In this paper, however, we focus on DIP methods, which have demonstrated strong performance across a wide range of IIPs \cite{heckel2019deep,sapienza2022deep,liang2025watermark}, in some cases even outperforming training-data-intensive approaches \cite{li2023deep,alkhouriNeuIPS24}. 

\begin{figure}[h]
\begin{minipage}{0.45\linewidth}
\centerline{\includegraphics[width=2.8in]{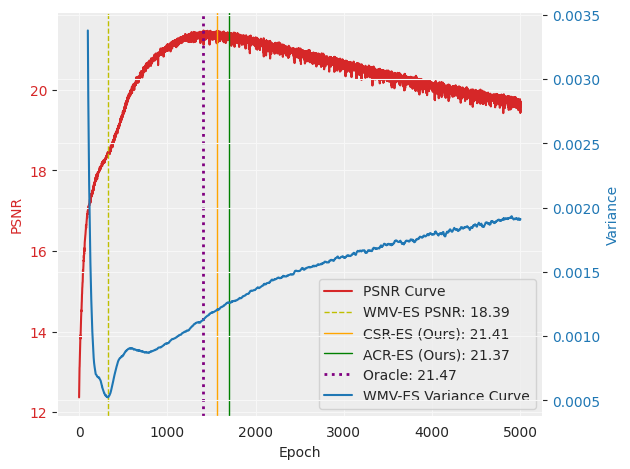}}\vspace{-0.1in}
\centerline{\scriptsize (a) Natural image}
\end{minipage}\hspace{0.35in}
\begin{minipage}{0.45\linewidth}
\centerline{\includegraphics[width=2.8in]{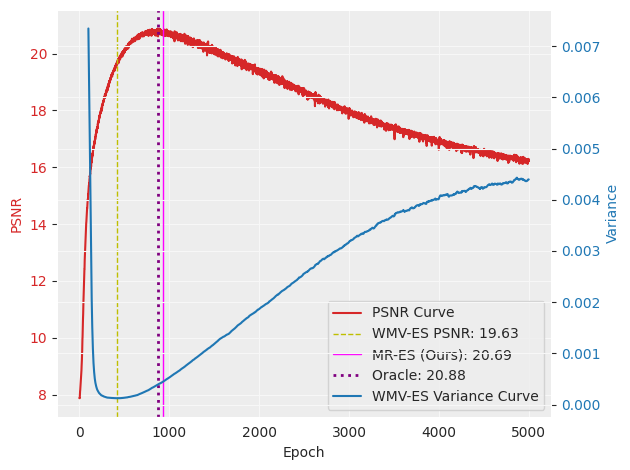}}\vspace{-0.1in}
\centerline{\scriptsize (b) CT image}
\end{minipage}
\label{fig:limit_variance}\vspace{-0.08in}
\caption{Denoising result on Gaussian noise with noise level $0.26$. The oracle is defined as the iteration that achieves the maximum PSNR along the DIP trajectory. 
The variance-based WMV-ES \cite{wang2021early} reaches its global minimum too early and therefore selects a stopping point far from the oracle. 
In contrast, our CSR-ES, ACR-ES, and MR-ES select iterations much closer to the oracle.}
\end{figure} 

In the original DIP study, the authors in \cite{ulyanov2018deep} demonstrated that an untrained convolutional neural network (CNN) is able to capture low-frequency image statistics even before optimization. Despite its strong empirical performance \cite{alkhouri2025DIPSPM}, DIP remains vulnerable to overfitting due to network over-parameterization, often fitting measurement noise or null-space artifacts of the forward model, as shown for MRI in \cite{liang2024analysis}. Over the past few years, several approaches have been proposed to mitigate this issue which, as surveyed in \cite{alkhouri2025DIPSPM}, can be broadly categorized into regularization, network re-parameterization, and early stopping.

Regularization methods, ranging from total variation DIP (TV-DIP) \cite{liu2019image} to auto-encoding sequential DIP (aSeqDIP) \cite{alkhouriNeuIPS24}, aim to reduce overfitting. However, these methods can be computationally expensive due to additional optimization terms and variables, and may also be sensitive to regularization hyperparameters.

Early Stopping (ES) \cite{wang2021early}, which is the focus of this paper, defines stopping criteria that attempt to detect when overfitting begins. Existing DIP ES approaches typically estimate overfitting from either reconstruction trajectory statistics or unbiased risk surrogates. Variance-based methods \cite{wang2021early} monitor fluctuations in the DIP output and can be effective in many settings. However, the resulting signal depends not only on overfitting, but also on optimization dynamics and intrinsic image structure. For high-resolution images or images containing substantial high-frequency details, such as textured natural images or medical images with fine anatomical structures, oscillations may appear even before severe overfitting occurs, making the variance curve ambiguous or unstable (see Fig.~\ref{fig:limit_variance}). 

Stein’s Unbiased Risk Estimator (SURE) \cite{fourdrinier2018shrinkage} can be used for ES, providing a more principled estimate of reconstruction risk. However, SURE-based methods require knowledge of the measurement noise level and are typically more reliable in expectation over multiple realizations. In the per-instance DIP setting, where stopping decisions must be made from a single reconstruction trajectory, these estimates can exhibit high variance and lead to unreliable stopping decisions. These limitations motivate validation-based criteria that directly compare the DIP trajectory against a reference signal not used during fitting. Toward addressing these limitations, we make the following contributions:

\begin{itemize}
    \item We demonstrate that if two fully independent noisy copies of the degraded measurements are available, then near-optimal overfitting detection becomes achievable. This is supported by theoretical results on single-reference validation, pseudo-validation estimation, and the effect of shared noise.

    \item Since obtaining \textit{fully} independent degraded measurements is generally infeasible, we propose three task-specific ES algorithms based on pseudo-reference construction: channel-similarity reference ES (CSR-ES) for RGB images, mask-reference ES (MR-ES) for grayscale images such as MRI and CT, and augmented-channel reference ES (ACR-ES) for handling different noise types, including Poisson noise. 

    %\item \textcolor{red}{broad applicability...}

    \item Extensive evaluation across natural image restoration and medical image reconstruction tasks, under varying noise levels and noise types, demonstrates that the proposed methods more reliably detect overfitting compared to existing variance-based and SURE-based DIP ES methods, all without requiring an accurate estimate of the noise level. %without requiring knowledge of the measurement noise characteristics.
\end{itemize}

\subsection{Preliminaries}
\paragraph{DIP setup and notation.}
We consider a noisy measurement \(y=A(x)+\eta\), where \(x\) is the unknown clean
image and \(A\) is the forward operator. DIP reconstructs \(x\) by optimizing an
untrained network \(f_\theta(z)\) to fit the measurement,
\begin{equation}\nonumber
    \theta_t \approx \arg\min_\theta \mathcal L(A(f_\theta(z)),y),
\end{equation}
and we denote the resulting reconstruction at iteration \(t\) by
\(\widehat x_t(y)=f_{\theta_t}(z)\). Since continued optimization eventually fits
measurement noise or null-space artifacts, early stopping chooses a time
\(\widehat t\) and returns \(\widehat x_{\widehat t}(y)\).

% The inverse problem: In IIPs, the main task is to estimate an unknown image or signal $\mathbf{x}\in \mathbb{R}^n$ from vector $\mathbf{y}\in \mathbb{R}^m$ that represents sub-sampled measurements or a degraded image. The process is often corrupted by noise. Mathematically, the unknown signal and the measurements are related as 
% %
% \begin{equation}\nonumber\label{eqn: forward model}\nonumber
%     \mathbf{y} = \mathcal{A}(\mathbf{x}) + \mathbf{n}\:,
% \end{equation}
% %
% where $\mathcal{A}(\cdot) : \mathbb{R}^n\rightarrow \mathbb{R}^m$ (with $m\leq n$) is the linear or non-linear forward operator that represents the measurement process, and $\mathbf{n}\in \mathbb{R}^m$ denotes the noise in the measurement domain, e.g., it can be assumed to be sampled from a Gaussian distribution $\mathcal{N}(\mathbf{0},\sigma^2_{\mathbf{y}}\mathbf{I})$, where $\sigma_{\mathbf{y}}>0$ denotes the measurement noise level. NOTE ABOUT MRI...

% The original DIP... 

% The main example of noise over-fitting with cat images to show the problem... Lets produce a similar plot or re-use the one from SPM with accreditation ...

% This is also important to show that ES-DIP is not exact...

%%%%%%%%%%%%%%%%%%%%%%%%%%%%%%%%%%%%%%%%%%%%%%%%%%%%%%%%%%%%%%%%%%%%%%%%%%%%%%%%%%%%%%%%%%%%%%%%%%%%%%%%%%%%%%%%%%%%%%%%

\section{Reference-Based Early Stopping}
\label{sec:reference_principle}

DIP early stopping can be viewed as a validation problem. Let
$y_1=x+\eta_1$ be the noisy image used for fitting, and suppose for the moment
that an independent noisy reference $y_2=x+\eta_2$ were available. If
$\{\hat x_t(y_1)\}_{t=1}^T$ denotes the DIP trajectory trained only on $y_1$, then
\begin{equation}\nonumber
\mathbb{E}_{\eta_2}\!\left[
    \frac1n\|\hat x_t(y_1)-y_2\|^2 \,\middle|\, y_1
\right]
=
\frac1n\|\hat x_t(y_1)-x\|^2+\sigma^2 .
\end{equation}
Thus, the reference loss has the same minimizer as the clean-image MSE, up to an
additive constant. Moreover, because the loss averages over all pixels, a single
independent reference is already sufficient to obtain a concentrated validation
signal.

\begin{proposition}[Single-reference validation, informal]
\label{prop:single_reference_informal}
Assume that the reference noise is independent, zero-mean, and sub-Gaussian. Then,
with probability at least $1-\delta$, uniformly over $t\le T$,
\begin{equation}\nonumber
\frac1n\|\hat x_t(y_1)-y_2\|^2
=
\frac1n\|\hat x_t(y_1)-x\|^2+\sigma^2
+
O\!\left(\sqrt{\frac{\log(T/\delta)}{n}}\right).
\end{equation}
% Consequently, the stopping time minimizing the reference loss attains near-oracle
% MSE along the DIP trajectory.
\end{proposition}

The formal statement, constants, proof, and the corresponding oracle inequality
are given in Appendix~\ref{app:reference_validation}.
In standard DIP, however, only one noisy observation is available. We therefore
construct a pseudo-reference $\tilde y$ from the observed image $y$. Since
$\tilde y$ is not fully independent of the image used for fitting, we decompose
$
y=x+s+\eta, 
\tilde y=x+s+\tilde\eta,
$
where $s$ denotes the component shared by fitting and validation, while
$\eta$ and $\tilde\eta$ denote non-shared components. For the pseudo-validation
loss
\begin{equation}\nonumber
\tilde V_t := \frac1n\|\hat x_t(y)-\tilde y\|^2 ,
\end{equation}
we have
\begin{equation}\nonumber
\mathbb{E}_{\tilde\eta}\!\left[
    \tilde V_t \,\middle|\, y,s
\right]
=
\frac1n\|\hat x_t(y)-(x+s)\|^2+\sigma_{\tilde\eta}^2 .
\end{equation}
Thus pseudo-validation is exact for the effective target $x+s$. Its usefulness
for clean-image early stopping depends on whether the shared component $s$ is
small, harmless, or not preferentially fitted before the remaining noise.

This gives the main design principle of the paper: a pseudo-reference need not be
perfectly independent, but it should reduce the overlap between the component
fitted by DIP and the component used for validation. In the noise-fitting
model
\begin{equation}\nonumber
\hat x_t(y) \approx x+\alpha_t(s+\eta),
\qquad \alpha_t\in[0,1],
\end{equation}
which is consistent with DIP's spectral bias when $s$ and $\eta$ have similar
statistics, the expected pseudo-validation loss is minimized near
\begin{equation}\nonumber
\alpha^\star
=
\frac{\|s\|^2/n}{\|s\|^2/n+\sigma_\eta^2}.
\end{equation}
If the construction makes the shared energy a $1/r$ fraction of the
non-shared variance, the excess MSE relative to the oracle is
$O(1/r^2)$, quadratically tighter than the worst-case
$O(1/r)$ bound implied by tracking the effective target $x+s$
above. Modest decorrelation of the pseudo-reference therefore suffices
for near-oracle stopping. The full derivation of the noise-fitting model is in Appendix~\ref{app:pseudo_reference}. The
algorithms in Section~\ref{sec:algorithms} instantiate this principle using
channel separation, held-out masks, and auxiliary corrupted channels.

\section{Pseudo-Reference Algorithms}
\label{sec:algorithms}

Section~\ref{sec:reference_principle} suggests that ES can be guided by a target that is not fully independent, but whose validation component is sufficiently decorrelated from the component fitted by DIP. We instantiate this principle through three pseudo-reference constructions, chosen according to the available image structure: cross-channel references for RGB images, masks for grayscale images, and auxiliary corrupted images for more challenging noise models. Full algorithmic pseudocodes are given in Appendix~\ref{app:algorithm_details}.

\subsection{Channel-similarity reference for RGB images}
\label{sec:csr}

For RGB images, different color channels often share the same image structure while containing partially decorrelated noise. CSR-ES trains DIP on the original RGB image using the standard fitting loss, and uses cross-channel consistency as a post-hoc stopping criterion.
% CSR-ES therefore uses one channel as the DIP fitting target and another structurally similar channel as the pseudo-reference. 
Writing
\begin{equation}\nonumber
    y_c=x_c+N_c,\quad c\in\{R,G,B\}, \quad x_c \ \text{the clean channel and} \ N_c \ \text{the channel noise},
\end{equation}
we select the closest channel pair from the noisy image,
\begin{equation}\nonumber
    (i^\star,j^\star)
    =
    \arg\min_{i\ne j}
    \frac1n\|y_i-y_j\|^2,
\end{equation}
and use the validation curve
\begin{equation}\nonumber
    V_t^{\rm CSR}
    =
    \frac1n\|\hat x_{t,i^\star}-y_{j^\star}\|^2,
    \qquad
    \hat t_{\rm CSR}
    =
    \arg\min_t V_t^{\rm CSR}.
\end{equation}
The intuition is that selecting a structurally similar reference channel reduces clean cross-channel bias, while the channel-specific noise remains less correlated with the fitted channel. %A controlled-bias decomposition and an idealized tracking result are given in Appendix~\ref{app:csr_theory}.

\subsection{Mask-reference early stopping for grayscale images}
\label{sec:mr}

CSR-ES requires multiple channels and is therefore not directly applicable to grayscale images. MR-ES instead constructs a pseudo-reference by holding out a small subset of pixels. Let $M\in\{0,1\}^n$ be a high-retention mask, for example $\mathbb P(M_i=1)=0.98$, and let $H=1-M$ be the held-out mask. DIP is trained with the masked loss
$
    \|M\odot(f_\theta(z)-y)\|^2,
$ 
so the held-out pixels are excluded from the training loss rather than treated as zero-valued targets. The stopping curve is then
\begin{equation}\nonumber
    V_t^{\rm MR}
    =
    \frac1{\sum_i H_i}
    \|H\odot(\hat x_t-y)\|^2,
    \qquad
    \hat t_{\rm MR}
    =
    \arg\min_t V_t^{\rm MR}.
\end{equation}
This construction trades a small loss of training information for a validation signal that is not directly fitted by the network. The held-out validation theorem, its transfer condition to full-image risk, and the implementation requirement are given in Appendix~\ref{app:mask_reference}.

\subsection{Augmented-channel reference for flexible pseudo-validation}
\label{sec:acr}

Mask references can be less reliable when a representative mask is difficult to choose, for example under signal-dependent noise. ACR-ES instead augments the output space of the DIP network. Given the noisy observation $y$, generate two pseudo-copies
\begin{equation}\nonumber
    y^{(1)}=y+\zeta_1,\qquad
    y^{(2)}=y+\zeta_2,
    %\qquad
    %\zeta_1,\zeta_2\sim\mathcal N(0,\tau^2I),
\end{equation}
where $\zeta_1$ and $\zeta_2$ are random noise, 
and train a $2C$-channel DIP network to fit
$
    Y=\operatorname{concat}(y,y^{(1)}).
$ 
The original output channels provide the reconstruction, while the auxiliary channels are monitored against the independent pseudo-copy $y^{(2)}$, whose added
perturbation \(\zeta_2\) is independent of the perturbation \(\zeta_1\) used to
construct \(y^{(1)}\).

In practice we use a frequency-bias score rather than raw MSE on the auxiliary channel:
\begin{equation}\nonumber
B_t^{\rm ACR}
=
\frac{
\sum_\omega \|\omega\|^2
\left|
\mathcal F(\hat X_{t,{\rm aux}}-y^{(2)})(\omega)
\right|^2
}{
\sum_\omega
\left|
\mathcal F(\hat X_{t,{\rm aux}}-y^{(2)})(\omega)
\right|^2
},
\quad
\hat t_{\rm ACR}
=
\arg\max_{t\ge t_0} B_t^{\rm ACR} \ \text{after a brief burn-in}\ t_0.
\end{equation}
This score measures the second moment of the residual spectrum and is more sensitive to the frequency shift that appears when DIP begins fitting noise. 
For the spectral-bias variant of ACR-ES, we select the stopping point by the maximum of the average-frequency score rather than its minimum.
This choice is motivated by the typical spectral evolution of DIP training.
The detailed discussion is given in Appendix~\ref{app:algorithm_details}. 
Additional discussion of the augmented-channel construction, the role of $\tau$, and the score choice is given in Appendix~\ref{app:acr_theory}.

\iffalse
\paragraph{Applicability.}
CSR-ES is lightweight and post-hoc but relies on useful cross-channel similarity; MR-ES applies to grayscale images but requires the held-out region to be representative; ACR-ES is more flexible but introduces an auxiliary perturbation scale. These are predictable trade-offs of the corresponding pseudo-reference constructions. The same pseudo-reference principle extends to regularization-parameter selection beyond DIP; see Appendix~\ref{appen: beyond DIP}.
\fi

\section{Experimental Results}
\label{sec:experiments}

We evaluate whether the proposed reference-based early-stopping criteria can identify stopping points close to the oracle iteration of DIP optimization.
Our experiments are conducted on both natural and medical images.
For natural images, we use Set3, Set14, Set18, CBSD68, and Urban100.
For medical images, we use MRI15, CT10, PD20, and T2-20.
We consider three representative corruption models: Gaussian with noise level $0.26$, Poisson with noise level $10$, and impulse with noise level $0.1$. 
For Poisson and impulse corruptions, we use their zero-mean versions so that the perturbations do not introduce an additional intensity bias.

The baseline methods include WMV-ES~\cite{wang2021early} and SURE-based early stopping.
For WMV-ES, we follow the setting in~\cite{wang2021early} and select the stopping iteration according to the local minimum of the windowed variance curve. 
Specifically, for a DIP trajectory of $5{,}000$ iterations, we use a window size of $1{,}000$ iterations to compute the variance statistic.
For SURE, we provide the true noise level used to generate the corrupted observation. 
This gives SURE a favorable setting, since the noise variance is assumed to be known exactly.

We report two quantities: the PSNR at the selected stopping iteration and the corresponding PSNR Gap.
The PSNR Gap is defined as
$
\mathrm{Gap}
=
\mathrm{PSNR}_{\mathrm{oracle}}
-
\mathrm{PSNR}_{\mathrm{ES}},
$ 
where $\mathrm{PSNR}_{\mathrm{ES}}$ is the PSNR obtained at the iteration selected by an early-stopping criterion, and $\mathrm{PSNR}_{\mathrm{oracle}}$ is the maximum PSNR along the corresponding DIP trajectory.
A smaller gap indicates that the stopping criterion selects an iteration closer to the oracle stopping point.
Unless otherwise specified, results are reported as mean $\pm$ standard deviation across images.

Our experiments are organized according to the applicability of the proposed criteria.
CSR-ES is a post-hoc criterion based on multi-channel consistency, and is therefore most naturally suited for color images.
MR-ES constructs a mask-based reference signal, making the reference-based principle applicable to single-channel grayscale images.
ACR-ES uses auxiliary corrupted references and provides a more flexible stopping signal for challenging noise models, especially when the reference signal is less reliable.
Accordingly, we first evaluate CSR-ES and ACR-ES on natural color images, then evaluate MR-ES on grayscale medical images, then analyze method-specific oracle behavior, and finally test whether the proposed stopping principle generalizes beyond denoising.

For stopping criteria evaluated on the same DIP trajectory, such as WMV-ES and CSR-ES, the oracle PSNR is shared.
For criteria that modify the optimization trajectory, such as MR-ES and ACR-ES, we additionally report method-specific oracle values.
This distinction is important because different trajectories may have different attainable oracle PSNR values.

Table~\ref{tab:average_all} compares the full practical pipelines: baseline ES methods on the original DIP trajectory versus our ES method on the modified trajectory.
Although the proposed modification may slightly reduce the oracle peak PSNR, it greatly reduces the gap between the selected stopping point and the oracle best point.
The reduction in stopping gap compensates for the small loss in peak PSNR, leading
to a higher final reconstruction PSNR.

Due to space limitations, the main paper focuses on four representative quantitative studies and two trajectory-level visualizations.
Additional results in Appendix~\ref{app:add_exp} provide per-image oracle comparisons, all RGB channel-reference pairs used by CSR-ES, ablations of MSE-based and bias-based reference scores, sensitivity to auxiliary noise levels, and additional qualitative examples.

\begin{table}[t]%\vspace{-0.1in}
\centering
\caption{Average early-stopping performance on natural image datasets under different noise types. 
All stopping criteria are evaluated on the same ACR trajectory, so the oracle PSNR is shared across WMV-ES, CSR-ES, and ACR-ES. 
PSNR is reported as mean $\pm$ standard deviation, and PSNR Gap denotes the difference from the shared oracle. \textbf{This table isolates the stopping rule rather than comparing the original full pipelines. For the latter, please refer to Table \ref{tab:average_all}}}
\label{tab:natural_es_results}
\small
\setlength{\tabcolsep}{6pt}
\renewcommand{\arraystretch}{1.18}
\resizebox{\linewidth}{!}{
\begin{tabular}{llcccccccc}
\toprule
\multirow{2}{*}{Noise} 
& \multirow{2}{*}{Dataset}
& \multicolumn{2}{c}{WMV-ES}
& \multicolumn{2}{c}{CSR-ES}
& \multicolumn{2}{c}{ACR-ES}
& \multicolumn{1}{c}{\cellcolor{gray!10}Oracle} \\
\cmidrule(lr){3-4}
\cmidrule(lr){5-6}
\cmidrule(lr){7-8}
\cmidrule(lr){9-9}
&
& PSNR $\uparrow$ 
& Gap $\downarrow$
& PSNR $\uparrow$ 
& Gap $\downarrow$
& PSNR $\uparrow$ 
& Gap $\downarrow$
& \cellcolor{gray!10}PSNR $\uparrow$ \\
\midrule

\multirow{4}{*}{Gaussian}
& Set14
& 21.86($\pm$2.18)
& 0.64($\pm$0.56)
& \textbf{22.27($\pm$1.71)}
& 0.24($\pm$0.21)
& 22.26($\pm$1.71)
& 0.25($\pm$0.13)
& \cellcolor{gray!10}22.50($\pm$1.79) \\
& Set18
& 21.37($\pm$2.57)
& 1.27($\pm$1.85)
& 22.38($\pm$1.59)
& 0.26($\pm$0.16)
&\textbf{ 22.41($\pm$1.58)}
& 0.23($\pm$0.12)
& \cellcolor{gray!10}22.64($\pm$1.65) \\
& CBSD68
& 21.40($\pm$2.79)
& 1.24($\pm$1.04)
& \textbf{22.48($\pm$2.17)}
& \textbf{0.16($\pm$0.12)}
& 22.35($\pm$1.97)
& 0.30($\pm$0.28)
& \cellcolor{gray!10}22.64($\pm$2.18) \\
& Urban100
& 18.82($\pm$2.99)
& 2.75($\pm$2.40)
& \textbf{21.38($\pm$1.33)}
& \textbf{0.18($\pm$0.14)}
& 21.36($\pm$1.31)
& 0.20($\pm$0.10)
& \cellcolor{gray!10}21.56($\pm$1.35) \\
\addlinespace[2pt]
\midrule

\multirow{4}{*}{Poisson}
& Set14
& 23.81($\pm$2.93)
& 0.83($\pm$0.93)
& 24.30($\pm$2.12)
& 0.33($\pm$0.41)
& \textbf{24.52($\pm$2.35)}
& \textbf{0.11($\pm$0.05)}
& \cellcolor{gray!10}24.63($\pm$2.34) \\
& Set18
& 25.62($\pm$3.41)
& 1.08($\pm$1.95)
& 25.99($\pm$2.02)
& 0.71($\pm$1.33)
& \textbf{26.55($\pm$2.33)}
& \textbf{0.15($\pm$0.10)}
& \cellcolor{gray!10}26.70($\pm$2.32) \\
& CBSD68
& 23.06($\pm$3.33)
& 1.74($\pm$1.79)
& 24.59($\pm$2.29)
& 0.20($\pm$0.15)
& \textbf{24.66($\pm$2.33)}
& \textbf{0.14($\pm$0.09)}
& \cellcolor{gray!10}24.80($\pm$2.33) \\
& Urban100
& 21.32($\pm$4.00)
& 2.94($\pm$3.40)
& 24.07($\pm$1.62)
& 0.19($\pm$0.13)
& \textbf{24.08($\pm$1.66)}
& \textbf{0.18($\pm$0.12)}
& \cellcolor{gray!10}24.26($\pm$1.67) \\
\addlinespace[2pt]
\midrule

\multirow{4}{*}{Impulse}
& Set14
& 23.33($\pm$2.53)
& 0.82($\pm$0.83)
& \textbf{23.77($\pm$1.79)}
& \textbf{0.22($\pm$0.16)}
& 23.65($\pm$1.74)
& 0.49($\pm$0.23)
& \cellcolor{gray!10}24.14($\pm$1.88) \\
& Set18
& 23.30($\pm$2.64)
& 1.31($\pm$2.13)
& \textbf{24.20($\pm$1.55)}
& \textbf{0.51($\pm$0.34)}
& 24.09($\pm$1.52)
& 0.53($\pm$0.21)
& \cellcolor{gray!10}24.61($\pm$1.62) \\
& CBSD68
& 22.67($\pm$3.01)
& 1.39($\pm$1.29)
& \textbf{23.84($\pm$2.17)}
& \textbf{0.23($\pm$0.19)}
& 23.63($\pm$2.02)
& 0.44($\pm$0.25)
& \cellcolor{gray!10}24.07($\pm$2.19) \\
& Urban100
& 20.05($\pm$3.68)
& 3.15($\pm$3.13)
& \textbf{22.96($\pm$1.46)}
& \textbf{0.24($\pm$0.19)}
& 22.86($\pm$1.42)
& 0.34($\pm$0.17)
& \cellcolor{gray!10}23.21($\pm$1.51) \\

\bottomrule
\end{tabular}}
\end{table}

\subsection{Early Stopping on Natural Color Images}

We first evaluate the proposed criteria on natural color images. 
%In this experiment, WMV-ES, CSR-ES, and ACR-ES are all evaluated on the same ACR-generated DIP trajectory. 
Table \ref{tab:natural_es_results} is not intended as a full-pipeline comparison. Instead, it compares the stopping methods by evaluating WMV-ES, CSR-ES, and ACR-ES on the same ACR-induced trajectory. Likewise, Table \ref{tab:medical_es_results} compares WMV-ES and MR-ES by evaluating them on the same MR-induced trajectory. Full practical pipeline comparisons, where each method is evaluated on its own natural trajectory, are reported in Table \ref{tab:average_all}.
% Therefore, the three stopping criteria share the same oracle PSNR for each image. 
This setting isolates the effect of the stopping criterion itself, rather than differences caused by different training trajectories.

Table~\ref{tab:natural_es_results} reports the average early-stopping performance on Set14, Set18, CBSD68, and Urban100 under Gaussian, Poisson, and impulse noise. Compared with WMV-ES, both CSR-ES and ACR-ES substantially reduce the gap to the shared oracle stopping point across datasets and noise types. 
The improvement is especially pronounced on structured images, such as Urban100, where WMV-ES often stops too early and leaves a large gap to the oracle.
In contrast, the proposed criteria select iterations much closer to the oracle region.
CSR-ES is particularly effective under Gaussian and impulse noise, while ACR-ES is more favorable under several Poisson settings, suggesting that auxiliary corrupted references provide a useful stopping signal under signal-dependent corruptions.

Figure~\ref{fig:closeACR} provides a trajectory-level view of this behavior.
The variance-based WMV score often reaches its minimum before the PSNR peak, leading to premature stopping.
By contrast, CSR-ES and ACR-ES select iterations near the oracle region, showing better alignment with the actual overfitting dynamics of DIP.

\begin{figure}
\begin{minipage}{0.45\linewidth}
\centerline{\includegraphics[width=2.8in]{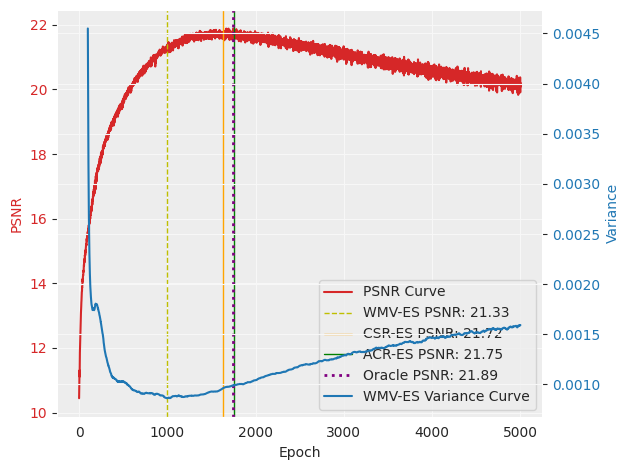}}
\centerline{\scriptsize (a) Gaussian iteration: %WMV-ES (995),
CSR-ES (1625), ACR-ES (1759), Oracle (1748)
}
\end{minipage}\hspace{0.35in}
\begin{minipage}{0.45\linewidth}
\centerline{\includegraphics[width=2.8in]{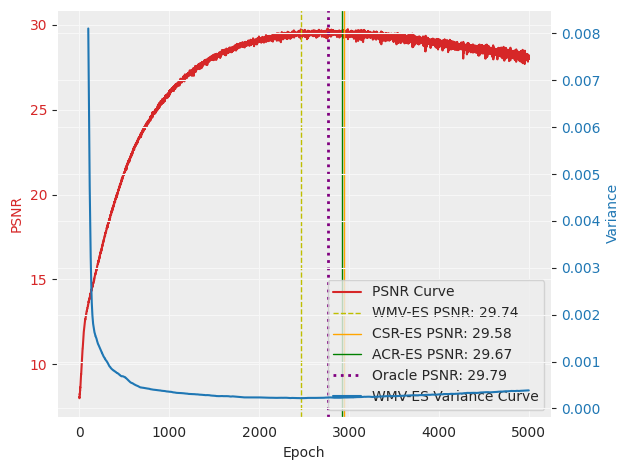}}
\centerline{\scriptsize (b) Poisson iteration: %WMV-ES (2463), 
CSR-ES (2940), ACR-ES (2926), Oracle (2761)}
\end{minipage}
\caption{Color image denoising results. Image corrupted by Gaussian noise with noise level $0.26$; Poisson noise with noise level $10$.  Optimizer: Adam with LR $10^{-4}$. The number in parentheses after each method indicates the selected iteration. The variance-based
criterion often reaches its global minimum too early, leading to premature
stopping.}\label{fig:closeACR}
\end{figure}

Appendix~\ref{app:csr_grid} further analyzes all RGB channel-reference pairs used by CSR-ES.
The automatically selected pair, computed from the noisy observation only, yields near-oracle stopping across Gaussian, Poisson, and impulse noise.
This supports the use of channel similarity as a pseudo-validation signal without access to the clean image.
Additional results are provided in Appendix~\ref{app:additional_curves}.

\begin{table}[t]
\centering
\caption{Average early-stopping performance on medical datasets under different noise types. 
PSNR is reported as mean $\pm$ standard deviation across images, and PSNR Gap denotes the difference from the oracle best stopping point.}
\label{tab:medical_es_results}
\small
\setlength{\tabcolsep}{6pt}
\renewcommand{\arraystretch}{1.18}
% \resizebox{\linewidth}{!}
{
\begin{tabular}{llcccccc}
\toprule
\multirow{2}{*}{Noise} 
& \multirow{2}{*}{Dataset}
& \multicolumn{2}{c}{WMV-ES}
& \multicolumn{2}{c}{MR-ES}
& \multicolumn{1}{c}{\cellcolor{gray!10}Oracle} \\
\cmidrule(lr){3-4}
\cmidrule(lr){5-6}
\cmidrule(lr){7-7}
& 
& PSNR $\uparrow$ 
& Gap $\downarrow$
& PSNR $\uparrow$ 
& Gap $\downarrow$
& \cellcolor{gray!10}PSNR $\uparrow$ \\
\midrule

\multirow{4}{*}{Gaussian}
& MRI15
& 21.41($\pm$0.54) & 0.60($\pm$0.38)
& \textbf{21.60($\pm$0.69)} & \textbf{0.41($\pm$0.29)}
& \cellcolor{gray!10}22.01($\pm$0.58) \\
& CT10
& 19.63($\pm$0.60) & 0.94($\pm$0.34)
& \textbf{20.22($\pm$0.47)} & \textbf{0.35($\pm$0.28)}
& \cellcolor{gray!10}20.57($\pm$0.43) \\
& PD20 
& 18.12($\pm$0.92) & 2.15($\pm$0.86)
& \textbf{19.94($\pm$0.36)} & \textbf{0.33($\pm$0.24)}
& \cellcolor{gray!10}20.27($\pm$0.15) \\
& T2-20
& 17.11($\pm$0.55) & 2.67($\pm$0.51)
& \textbf{19.62($\pm$0.19)} & \textbf{0.16($\pm$0.09)}
& \cellcolor{gray!10}19.78($\pm$0.12) \\
\addlinespace[2pt]
\midrule

\multirow{4}{*}{Poisson}
& MRI15
& \textbf{25.37($\pm$0.88)} & \textbf{0.76($\pm$0.34)}
& 24.78($\pm$1.27) & 1.35($\pm$1.04)
& \cellcolor{gray!10}26.13($\pm$0.71) \\
& CT10
& 23.45($\pm$0.10) & 3.15($\pm$1.09)
& \textbf{24.56($\pm$1.30)} & \textbf{2.03($\pm$0.85)}
& \cellcolor{gray!10}26.60($\pm$1.14) \\
& PD20 
& \textbf{22.71($\pm$2.35)} & \textbf{2.93($\pm$2.41)}
& 21.15($\pm$0.37) & 4.49($\pm$0.37)
& \cellcolor{gray!10}25.63($\pm$0.43) \\
& T2-20
& 18.75($\pm$1.22) & 7.56($\pm$1.40)
& \textbf{23.26($\pm$0.47)} & \textbf{3.05($\pm$0.36)}
& \cellcolor{gray!10}26.32($\pm$0.55) \\
\addlinespace[2pt]
\midrule

\multirow{4}{*}{Impulse}
& MRI15
& 23.29($\pm$0.76) & 0.65($\pm$0.48)
& \textbf{23.49($\pm$0.52)} & \textbf{0.46($\pm$0.47)}
& \cellcolor{gray!10}23.94($\pm$0.45) \\
& CT10
& 21.55($\pm$0.81) & 1.58($\pm$0.36)
& \textbf{22.82($\pm$0.70)} & \textbf{0.31($\pm$0.17)}
& \cellcolor{gray!10}23.13($\pm$0.68) \\
& PD20 
& 21.55($\pm$1.78) & 1.83($\pm$1.81)
& \textbf{23.12($\pm$0.28)} & \textbf{0.26($\pm$0.19)}
& \cellcolor{gray!10}23.38($\pm$0.21) \\
& T2-20
& 17.61($\pm$0.72) & 4.59($\pm$0.78)
& \textbf{21.89($\pm$0.26)} & \textbf{0.31($\pm$0.16)}
& \cellcolor{gray!10}22.20($\pm$0.21) \\
\bottomrule
\end{tabular}}
\end{table}

\subsection{Early Stopping on Grayscale Medical Images}%\vspace{-0.in}
\label{subsec:medical_results}

CSR-ES relies on RGB channel consistency and is therefore not directly applicable to single-channel grayscale images.
To test whether the proposed reference-based principle can go beyond color images, we first evaluate MR-ES on four grayscale medical datasets.
MR-ES constructs a small held-out mask as a reference signal, allowing early stopping without requiring multi-channel information.

Table~\ref{tab:medical_es_results} shows that MR-ES substantially reduces the PSNR Gap under Gaussian and impulse noise across all medical datasets.
For example, on the more challenging PD and T2 images, WMV-ES can stop far from the oracle, whereas MR-ES selects iterations much closer to the best attainable point along the trajectory.
This demonstrates that the proposed reference-based framework is not limited to RGB images and can be adapted to grayscale medical reconstruction tasks.

However, MR-ES can be less stable under some Poisson settings.
This is expected because Poisson noise is signal-dependent, making a simple held-out mask reference less reliable.
To address this limitation, we further evaluate ACR-ES, which uses auxiliary corrupted references to construct a more flexible stopping signal.
Table~\ref{tab:ct_t2_acr_es_results} reports ACR-ES results on CT10 and T2-20 under Gaussian, Poisson, and impulse noise.
Across all settings, ACR-ES substantially improves over WMV-ES and selects stopping points much closer to the oracle.
The improvement is especially pronounced under Poisson noise: on CT10, ACR-ES reduces the average gap from $4.46$ dB to $0.80$ dB, and on T2-20, from $6.41$ dB to $1.42$ dB.
Under Gaussian and impulse noise, ACR-ES also consistently reduces the gap, showing that auxiliary corrupted references provide a reliable stopping signal for grayscale medical images.

\begin{table}[t]
\centering
\caption{
Average early-stopping performance of ACR-ES on CT10 and T2-20 under different noise types.
PSNR is reported as mean $\pm$ standard deviation across images, and PSNR Gap denotes the difference from the oracle best stopping point.
}
\label{tab:ct_t2_acr_es_results}
\small
\setlength{\tabcolsep}{5.5pt}
\renewcommand{\arraystretch}{1.18}
{
\begin{tabular}{llcccccc}
\toprule
\multirow{2}{*}{Noise} 
& \multirow{2}{*}{Dataset}
& \multicolumn{2}{c}{WMV-ES}
& \multicolumn{2}{c}{ACR-ES}
& \multicolumn{1}{c}{\cellcolor{gray!10}Oracle} \\
\cmidrule(lr){3-4}
\cmidrule(lr){5-6}
\cmidrule(lr){7-7}
&
& PSNR $\uparrow$
& Gap $\downarrow$
& PSNR $\uparrow$
& Gap $\downarrow$
& \cellcolor{gray!10}PSNR $\uparrow$ \\
\midrule

\multirow{2}{*}{Gaussian}
& CT10
& 19.61($\pm$0.68) & 1.00($\pm$0.43)
& \textbf{20.38($\pm$0.51)} & \textbf{0.23($\pm$0.09)}
& \cellcolor{gray!10}20.61($\pm$0.48) \\
& T2-20
& 16.61($\pm$0.72) & 3.11($\pm$0.71)
& \textbf{19.50($\pm$0.23)} & \textbf{0.21($\pm$0.15)}
& \cellcolor{gray!10}19.71($\pm$0.16) \\

\midrule

\multirow{2}{*}{Poisson}
& CT10
& 23.81($\pm$1.24) & 4.46($\pm$1.51)
& \textbf{27.47($\pm$1.28)} & \textbf{0.80($\pm$0.54)}
& \cellcolor{gray!10}28.27($\pm$1.58) \\
& T2-20
& 21.13($\pm$3.73) & 6.41($\pm$3.62)
& \textbf{26.12($\pm$0.80)} & \textbf{1.42($\pm$0.64)}
& \cellcolor{gray!10}27.54($\pm$0.67) \\

\midrule

\multirow{2}{*}{Impulse}
& CT10
& 20.82($\pm$1.00) & 2.07($\pm$0.47)
& \textbf{22.28($\pm$0.85)} & \textbf{0.60($\pm$0.37)}
& \cellcolor{gray!10}22.88($\pm$0.69) \\
& T2-20
& 17.35($\pm$1.30) & 4.62($\pm$1.19)
& \textbf{21.39($\pm$0.34)} & \textbf{0.59($\pm$0.37)}
& \cellcolor{gray!10}21.98($\pm$0.29) \\

\bottomrule
\end{tabular}}
\end{table}
\begin{table}[t]
\centering
\caption{Average PSNR and oracle gap across image groups. 
For each method, the value in parentheses is the average peak PSNR. 
The $\pm$ values denote standard deviations across images.}
\label{tab:average_all}
\resizebox{\textwidth}{!}{
\begin{tabular}{llcc cc cc cc cc}
\toprule
\multirow{2}{*}{Noise}
& \multirow{2}{*}{Image group}
& \multicolumn{2}{c}{SURE-ES}
& \multicolumn{2}{c}{WMV-ES}
& \multicolumn{2}{c}{MR-ES}
& \multicolumn{2}{c}{CSR-ES}
& \multicolumn{2}{c}{ACR-ES} \\
\cmidrule(lr){3-4}
\cmidrule(lr){5-6}
\cmidrule(lr){7-8}
\cmidrule(lr){9-10}
\cmidrule(lr){11-12}
&
& \makecell{ES PSNR\\(Oracle) $\uparrow$}
& Gap $\downarrow$
& \makecell{ES PSNR\\(Oracle) $\uparrow$}
& Gap $\downarrow$
& \makecell{ES PSNR\\(Oracle) $\uparrow$}
& Gap $\downarrow$
& \makecell{ES PSNR\\(Oracle) $\uparrow$}
& Gap $\downarrow$
& \makecell{ES PSNR\\(Oracle) $\uparrow$}
& Gap $\downarrow$ \\
\midrule

\multirow{2}{*}{Gaussian}
& Set3/Set14
& 21.09(22.24)$\pm$1.43 & 1.16$\pm$1.05
& 20.84(21.91)$\pm$1.86 & 1.07$\pm$0.83
& 21.46(21.79)$\pm$1.26 & 0.33$\pm$0.19
& \textbf{21.62(21.78)$\pm$1.39} & \textbf{0.17$\pm$0.08}
& 21.61(21.78)$\pm$1.29 & 0.18$\pm$0.11 \\
& CT10/T2-20
& 17.37(20.39)$\pm$2.03 & 3.02$\pm$2.08
& 18.39(20.31)$\pm$1.90 & 1.92$\pm$1.37
& 19.81(20.13)$\pm$0.41 & 0.32$\pm$0.29
& \multicolumn{2}{c}{--}
& \textbf{20.01(20.18)$\pm$0.57} & \textbf{0.17$\pm$0.05} \\

\midrule

\multirow{2}{*}{Poisson}
& Set3/Set14
& 22.30(23.31)$\pm$2.34 & 1.00$\pm$0.93
& 22.06(22.86)$\pm$2.42 & 0.80$\pm$0.82
& 22.62(22.86)$\pm$1.92 & 0.24$\pm$0.15
& 23.52(23.66)$\pm$1.99 & 0.14$\pm$0.08
& \textbf{23.53(23.66)$\pm$2.03} & \textbf{0.12$\pm$0.05} \\
& CT10/T2-20
& 21.68(26.02)$\pm$2.40 & 4.34$\pm$2.47
& 21.00(26.39)$\pm$2.76 & 5.38$\pm$2.62
& 23.77(26.01)$\pm$0.77 & 2.24$\pm$0.88
& \multicolumn{2}{c}{--}
& \textbf{26.55(27.39)$\pm$1.17} & \textbf{0.84$\pm$0.42} \\

\midrule

\multirow{2}{*}{Impulse}
& Set3/Set14
& 20.35(23.91)$\pm$2.12 & 3.56$\pm$2.55
& 22.45(23.42)$\pm$2.39 & 0.97$\pm$0.97
& 22.94(23.26)$\pm$1.64 & 0.32$\pm$0.17
& \textbf{23.16(23.31)$\pm$1.68} & \textbf{0.15$\pm$0.09}
& 22.89(23.31)$\pm$1.50 & 0.42$\pm$0.28 \\
& CT10/T2-20
& 14.59(22.83)$\pm$0.47 & 8.24$\pm$0.44
& 20.68(22.83)$\pm$1.86 & 2.16$\pm$1.29
& \textbf{22.27(22.60)$\pm$0.70} & \textbf{0.33$\pm$0.17}
& \multicolumn{2}{c}{--}
& 21.92(22.35)$\pm$0.82 & 0.42$\pm$0.40 \\

\bottomrule
\end{tabular}}
\end{table}

% 
% Together, Tables~\ref{tab:medical_es_results} and~\ref{tab:ct_t2_acr_es_results} show complementary strengths of the proposed medical-image criteria.
% MR-ES provides a simple mask-reference stopping rule that works well for Gaussian and impulse noise, while ACR-ES offers a more robust auxiliary-reference criterion for challenging signal-dependent noise.
Appendix~\ref{app:oracle_full} provides full per-image method-specific oracle comparisons, where the same trends hold across representative CT and T2 images.
\begin{table}[t]
\centering
\caption{
Extension to super-resolution and inpainting on Urban10. 
WMV-ES and CSR-ES are evaluated on the same standard DIP trajectory; therefore, they share the same oracle PSNR. 
PSNR and Gap are averaged over 10 images, where Gap denotes Oracle PSNR minus ES PSNR.
}
\label{tab:urban10_extension_wmv_csr}
\small
\setlength{\tabcolsep}{5.5pt}
\renewcommand{\arraystretch}{1.15}
\resizebox{0.95\linewidth}{!}{
\begin{tabular}{lllccccc}
\toprule
\multirow{2}{*}{Task}
& \multirow{2}{*}{Setting}
& \multirow{2}{*}{Noise}
& \multicolumn{2}{c}{WMV-ES}
& \multicolumn{2}{c}{CSR-ES}
& \cellcolor{gray!10}\multirow{2}{*}{Oracle} \\
\cmidrule(lr){4-5}
\cmidrule(lr){6-7}
&
&
& PSNR $\uparrow$ & Gap $\downarrow$
& PSNR $\uparrow$ & Gap $\downarrow$
& \cellcolor{gray!10}PSNR $\uparrow$ \\
\midrule

\multirow{6}{*}{SR}
& \multirow{3}{*}{$\times4$}
& Gaussian 
& 16.10($\pm$2.79) 
& 1.20($\pm$1.59) 
& \textbf{17.24($\pm$1.32)} 
& \textbf{0.07($\pm$0.08)} 
& \cellcolor{gray!10}17.31($\pm$1.38) \\
&
& Poisson  
& 16.29($\pm$2.85) 
& 1.60($\pm$1.70) 
& \textbf{17.84($\pm$1.25)} 
& \textbf{0.05($\pm$0.05)} 
& \cellcolor{gray!10}17.89($\pm$1.29) \\
&
& Impulse  
& 16.47($\pm$3.05) 
& 1.69($\pm$1.84) 
& \textbf{18.11($\pm$1.37)} 
& \textbf{0.05($\pm$0.04)} 
& \cellcolor{gray!10}18.16($\pm$1.38) \\
\cmidrule(lr){2-8}
& \multirow{3}{*}{$\times8$}
& Gaussian 
& 15.29($\pm$2.45) 
& 0.36($\pm$0.39) 
& \textbf{15.60($\pm$2.11)} 
& \textbf{0.05($\pm$0.05)} 
& \cellcolor{gray!10}15.65($\pm$2.15) \\
&
& Poisson  
& 15.71($\pm$2.60) 
& 0.39($\pm$0.53) 
& \textbf{16.03($\pm$2.15)} 
& \textbf{0.07($\pm$0.06)} 
& \cellcolor{gray!10}16.09($\pm$2.17) \\
&
& Impulse  
& 15.81($\pm$2.67) 
& 0.45($\pm$0.61) 
& \textbf{16.17($\pm$2.18)} 
& \textbf{0.09($\pm$0.08)} 
& \cellcolor{gray!10}16.26($\pm$2.21) \\

\midrule

\multirow{6}{*}{Inpainting}
& \multirow{3}{*}{50\%}
& Gaussian 
& 17.07($\pm$3.41) 
& 3.63($\pm$3.61) 
& \textbf{19.53($\pm$0.70)} 
& \textbf{1.16($\pm$0.92)} 
& \cellcolor{gray!10}20.69($\pm$1.03) \\
&
& Poisson  
& 17.29($\pm$3.52) 
& 4.30($\pm$4.00) 
& \textbf{21.16($\pm$0.91)} 
& \textbf{0.43($\pm$0.22)} 
& \cellcolor{gray!10}21.59($\pm$0.92) \\
&
& Impulse  
& 17.44($\pm$3.79) 
& 4.61($\pm$4.11) 
& \textbf{21.11($\pm$1.16)} 
& \textbf{0.94($\pm$0.57)} 
& \cellcolor{gray!10}22.05($\pm$1.27) \\
\cmidrule(lr){2-8}
& \multirow{3}{*}{70\%}
& Gaussian 
& 16.57($\pm$3.08) 
& 3.21($\pm$3.11) 
& \textbf{17.56($\pm$0.72)} 
& \textbf{2.22($\pm$1.17)} 
& \cellcolor{gray!10}19.78($\pm$0.90) \\
&
& Poisson  
& 16.79($\pm$3.21) 
& 3.74($\pm$3.41) 
& \textbf{19.78($\pm$0.87)} 
& \textbf{0.75($\pm$0.25)} 
& \cellcolor{gray!10}20.53($\pm$0.81) \\
&
& Impulse  
& 17.22($\pm$3.58) 
& 3.65($\pm$3.58) 
& \textbf{18.85($\pm$0.88)} 
& \textbf{2.03($\pm$0.68)} 
& \cellcolor{gray!10}20.88($\pm$1.09) \\
\bottomrule
\end{tabular}}
\end{table}

\subsection{Method-specific Oracle Analysis}
\label{subsec:oracle_analysis}

For criteria that modify the reference construction or induce a different optimization trajectory, a small PSNR Gap alone is not sufficient.
The gap could become small simply because the attainable oracle PSNR of that trajectory is degraded.
We therefore perform a method-specific oracle analysis, where each method is evaluated against the oracle PSNR of its own trajectory.

Table~\ref{tab:average_all} summarizes the average ES PSNR, method-specific oracle PSNR, and oracle gap on representative natural and medical image groups.
Each PSNR entry is reported as ES PSNR (Oracle PSNR), so the selected reconstruction quality and the best attainable quality along the corresponding trajectory can be compared simultaneously.
On natural images, the proposed CSR-ES and ACR-ES achieve much smaller gaps than SURE-ES and WMV-ES while maintaining oracle PSNR values comparable to the competing trajectories.
For example, under Gaussian noise, CSR-ES reduces the average gap on randomly selected images from Set3 and Set14 to $0.17$ dB, compared with $1.16$ dB for SURE-ES and $1.07$ dB for WMV-ES.
Under Poisson noise, ACR-ES achieves the smallest average gap, and under impulse noise CSR-ES gives the best stopping accuracy.

On medical images, CSR-ES is not applicable because it relies on RGB channel consistency.
MR-ES and ACR-ES instead provide reference-based stopping signals for grayscale images.
The results show that MR-ES is highly effective under impulse noise, while ACR-ES is particularly strong under Gaussian and Poisson noise.
For example, on randomly selected images from CT10 and T2-20 under Poisson noise, ACR-ES reduces the average gap to $0.84$ dB, compared with $4.34$ dB for SURE-ES, $5.38$ dB for WMV-ES, and $2.24$ dB for MR-ES.

These results show that the proposed criteria do not obtain smaller gaps by sacrificing the attainable oracle quality.
Instead, they select stopping points closer to the best attainable point along their corresponding trajectories.
The full per-image method-specific oracle comparisons for natural and medical images are provided in Appendix~\ref{app:oracle_full}, where the same conclusion holds across individual images and noise types.

\begin{figure}
\begin{minipage}{0.45\linewidth}
\centerline{\includegraphics[width=2.8in]{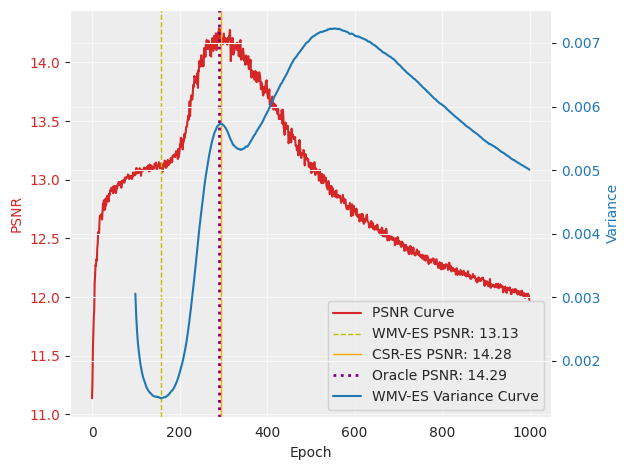}}%\vspace{-0.05in}
\centerline{\scriptsize (a) Super-resolution $\times 8$ with Impulse noise}
\end{minipage}
\hspace{0.35in}
\begin{minipage}{0.45\linewidth}
\centerline{\includegraphics[width=2.8in]{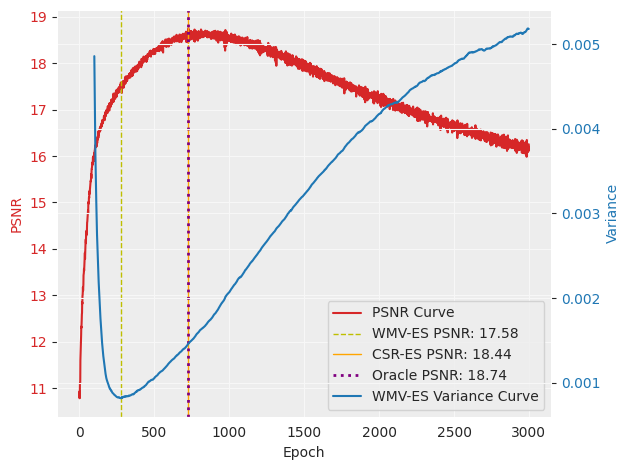}}\vspace{-0.05in}
\centerline{\scriptsize (b) Inpainting $70\%$ with Gaussian noise}
\end{minipage}
\caption{Image restoration results of the WMV and our CSR method. %\cy{maybe we can just show (a) and (b) in here, I will find another better case for our MR and ACR. Or should I put (b) to figure 1? } 
}\label{fig:sr_inp}
\end{figure}

\subsection{Extension to Super-resolution and Inpainting}

Finally, we evaluate whether the proposed stopping principle is specific to denoising or can transfer to other DIP-based inverse problems.
We consider super-resolution and inpainting on Urban10.
For this extension study, we focus on CSR-ES because it is the most directly transferable variant: it is post-hoc, requires no modification of the DIP optimization objective, and can be applied to a standard DIP trajectory.
WMV-ES is included as a trajectory-based baseline.

Since WMV-ES and CSR-ES are evaluated on the same standard DIP trajectory, they share the same oracle PSNR.
Table~\ref{tab:urban10_extension_wmv_csr} reports results for SR $\times4$, SR $\times8$, and inpainting under Gaussian, Poisson, and impulse noise.
CSR-ES consistently achieves smaller gaps to the shared oracle than WMV-ES across both tasks.
For super-resolution, CSR-ES almost reaches the oracle stopping point, with very small gaps across noise types.
For inpainting, the problem is more challenging and the oracle gap is larger, but CSR-ES still substantially improves over WMV-ES.

Figure~\ref{fig:sr_inp} shows representative stopping curves for SR $\times 8$ and inpainting with random $70\%$ missing.
The SR example reveals a failure mode of WMV-ES: the maximum PSNR occurs near a local maximum, rather than a local minimum, of their variance curve, contradicting the variance-minimum intuition used by WMV-ES.
Although this is a single representative case, it shows that variance minima are not always reliable indicators of the oracle stopping region.
In contrast, CSR-ES selects an iteration very close to the maximum-PSNR point.
For inpainting, WMV-ES fails in a different way: its global minimum occurs far before the PSNR peak, leading to premature stopping and a suboptimal reconstruction.
CSR-ES remains better aligned with the oracle region in both examples, suggesting that the proposed consistency-based signal is more robust than variance-based stopping across different inverse problems.

% Together, the quantitative results in Table~\ref{tab:urban10_extension_wmv_csr} and the visual examples in Figure~\ref{fig:sr_inp} suggest that consistency-based early stopping is not tied to denoising, but can serve as a lightweight and transferable criterion for broader DIP-based inverse problems.

\paragraph{Ablations and robustness.}
The appendix provides additional evidence supporting the design choices of the proposed criteria.
Appendix~\ref{app:score_ablation} compares MSE-based and bias-based reference scores.
The results justify using bias-based ACR-ES for natural color images, where cross-reference bias provides useful information beyond direct MSE consistency, and MSE-based MR-ES for grayscale medical images, where direct mask-reference consistency is more stable.
Appendix~\ref{app:aux_sensitivity} studies the auxiliary corrupted references used by ACR-ES.
The results show that ACR-ES remains stable under moderate changes of the auxiliary noise magnitude and under mismatched auxiliary noise levels, denoted as \(\tau_1\neq \tau_2\), indicating that it does not require precise tuning of the auxiliary corruption strength.

Overall, the experiments show that the proposed reference-based early-stopping criteria select stopping points much closer to the oracle than variance-based stopping across natural images, medical images, and additional inverse problems.
CSR-ES is simple and transferable because it is post-hoc and requires no change to DIP training.
MR-ES extends the reference-based principle to single-channel medical images where channel consistency is unavailable.
ACR-ES provides a more flexible auxiliary-reference signal, especially useful for challenging or signal-dependent corruptions.

\section{Conclusion}
DIP early stopping can be reduced to a validation problem. To address the
absence of clean validation data, we proposed self-referenced criteria that
construct pseudo-references from the noisy observation itself. Three concrete
methods instantiate this principle: CSR-ES, MR-ES
, and ACR-ES. Across experiments under multiple image datasets and several noise types for imaging inverse problems, these criteria consistently select stopping points closer to the oracle than existing baselines.

\section{Acknowledgments}
This work was supported in part by the Defense Advanced Research Projects Agency (DARPA) under Cooperative Agreement No. HR0011-25-2-0021.

\bibliography{refs}
%%%%%%%%%%%%%%%%%%%%%%%%%%%%%%%%%%%%%%%%%%%%%%%%%%%%%%%%%%%%
\newpage
\appendix
\onecolumn
\par\noindent\rule{\textwidth}{1pt}
\begin{center}
{\Large \bf Appendix}
\end{center}
\vspace{-0.1in}
\par\noindent\rule{\textwidth}{1pt}
\appendix

%%%%%%%%%%%%%%%%%%%%%%%%%%%%%%%%%%%%%%%%%%%%%%%%%%%%%%%%%%%%%%%%%%%%%%%%%%%%%%%
%%%%%%%%%%%%%%%%%%%%%%%%%%%%%%%%%%%%%%%%%%%%%%%%%%%%%%%%%%%%%%%%%%%%%%%%%%%%%%%

\section{Reference Validation Theory}
\label{app:reference_validation}

This appendix formalizes the validation principle used in
Section~\ref{sec:reference_principle}. Throughout this section, the DIP trajectory
$\{\hat x_t(y_1)\}_{t=1}^T$ is trained only on
\begin{equation}\nonumber
y_1=x+\eta_1,
\end{equation}
and the reference image is
\begin{equation}\nonumber
y_2=x+\eta_2,
\end{equation}
where $\eta_2$ is independent of $y_1$ and of the fitted trajectory conditional
on $y_1$.

\subsection{Setup and notation}

For each iteration $t$, define the clean risk
\begin{equation}\nonumber
R(t):=\frac1n\|\hat x_t(y_1)-x\|^2
\end{equation}
and the reference validation loss
\begin{equation}\nonumber
V_t:=\frac1n\|\hat x_t(y_1)-y_2\|^2.
\end{equation}
Let
\begin{equation}\nonumber
\hat t\in\arg\min_{1\le t\le T} V_t,
\qquad
t^\star\in\arg\min_{1\le t\le T} R(t).
\end{equation}
The goal is to compare $R(\hat t)$ to the oracle risk $R(t^\star)$.

\subsection{Single-reference concentration}

\begin{theorem}[Single-reference validation]
\label{thm:single_reference_validation}
Assume that the entries of $\eta_2$ are independent, zero-mean,
sub-Gaussian random variables with variance proxy $\sigma^2$ and
$\mathbb{E}\eta_{2,i}^2=\sigma^2$. Then there exists a universal constant
$C>0$ such that, with probability at least $1-\delta$, uniformly over
$1\le t\le T$,
\begin{equation}\nonumber
\left|
V_t - R(t)-\sigma^2
\right|
\le
C\left[
M_T\sqrt{\frac{\log(T/\delta)}{n}}
+
\sigma^2\frac{\log(T/\delta)}{n}
\right],
\end{equation}
where
\begin{equation}\nonumber
M_T
:=
\sigma\sup_{1\le t\le T}\sqrt{R(t)}+\sigma^2 .
\end{equation}
\end{theorem}

\begin{proof}
Condition on $y_1$ and hence on the entire DIP trajectory
$\{\hat x_t(y_1)\}_{t=1}^T$. For each fixed $t$, write
\begin{equation}\nonumber
a_t:=\hat x_t(y_1)-x.
\end{equation}
Since $y_2=x+\eta_2$,
\begin{equation}\nonumber
V_t
=
\frac1n\|a_t-\eta_2\|^2
=
\frac1n\|a_t\|^2
-\frac{2}{n}\langle a_t,\eta_2\rangle
+\frac1n\|\eta_2\|^2.
\end{equation}
Therefore
\begin{equation}\nonumber
V_t-R(t)-\sigma^2
=
-\frac{2}{n}\langle a_t,\eta_2\rangle
+
\frac1n\left(\|\eta_2\|^2-n\sigma^2\right).
\end{equation}

The first term is a centered sub-Gaussian linear functional. Conditional on
$a_t$, standard sub-Gaussian concentration gives, with probability at least
$1-\delta/(2T)$,
\begin{equation}\nonumber
\left|
\frac{2}{n}\langle a_t,\eta_2\rangle
\right|
\le
C\sigma\sqrt{R(t)}
\sqrt{\frac{\log(T/\delta)}{n}}.
\end{equation}
The second term is a centered average of sub-exponential random variables
$\eta_{2,i}^2-\sigma^2$. Bernstein or Hanson--Wright concentration gives,
with probability at least $1-\delta/(2T)$,
\begin{equation}\nonumber
\left|
\frac1n\left(\|\eta_2\|^2-n\sigma^2\right)
\right|
\le
C\sigma^2
\left[
\sqrt{\frac{\log(T/\delta)}{n}}
+
\frac{\log(T/\delta)}{n}
\right].
\end{equation}
Combining the two inequalities and taking a union bound over
$t=1,\ldots,T$ yields the stated result.
\end{proof}

\subsection{Oracle inequality}

\begin{corollary}[Near-oracle stopping]
\label{cor:reference_oracle}
Under the event in Theorem~\ref{thm:single_reference_validation}, define
\begin{equation}\nonumber
\varepsilon_n
:=
C\left[
M_T\sqrt{\frac{\log(T/\delta)}{n}}
+
\sigma^2\frac{\log(T/\delta)}{n}
\right].
\end{equation}
Then the stopping time $\hat t\in\arg\min_t V_t$ satisfies
\begin{equation}\nonumber
R(\hat t)
\le
R(t^\star)+2\varepsilon_n.
\end{equation}
\end{corollary}

\begin{proof}
On the event of Theorem~\ref{thm:single_reference_validation},
\begin{equation}\nonumber
R(t)\le V_t-\sigma^2+\varepsilon_n
\quad\text{and}\quad
V_t-\sigma^2\le R(t)+\varepsilon_n
\end{equation}
for all $t$. Hence
\begin{equation}\nonumber
R(\hat t)
\le
V_{\hat t}-\sigma^2+\varepsilon_n
\le
V_{t^\star}-\sigma^2+\varepsilon_n
\le
R(t^\star)+2\varepsilon_n.
\end{equation}
\end{proof}

\subsection{PSNR interpretation}

If image intensities are normalized so that the peak value is $I_{\max}$, then
\begin{equation}\nonumber
\operatorname{PSNR}(t)
=
10\log_{10}\left(\frac{I_{\max}^2}{R(t)}\right).
\end{equation}
If $R(\hat t)\le R(t^\star)+2\varepsilon_n$, then
\begin{equation}\nonumber
\operatorname{PSNR}(t^\star)-\operatorname{PSNR}(\hat t)
=
10\log_{10}\left(\frac{R(\hat t)}{R(t^\star)}\right)
\le
10\log_{10}\left(1+\frac{2\varepsilon_n}{R(t^\star)}\right).
\end{equation}
Thus the PSNR loss is small whenever the oracle MSE is bounded away from zero and
the validation error $\varepsilon_n$ is small.

\section{Pseudo-Reference Validation}
\label{app:pseudo_reference}

This appendix justifies the pseudo-reference discussion in
Section~\ref{sec:reference_principle}. The purpose is not to claim that a
pseudo-reference is fully independent, but to identify the bias introduced by
the component shared between fitting and validation.

\subsection{Risk estimated by pseudo-validation}

Assume that the observed image and pseudo-reference can be decomposed as
\begin{equation}\nonumber
y=x+s+\eta,
\qquad
\tilde y=x+s+\tilde\eta,
\end{equation}
where $s$ is shared between fitting and validation, while $\tilde\eta$ is
zero-mean and independent of the DIP trajectory conditional on $y$ and $s$.
Define
\begin{equation}\nonumber
\tilde V_t
:=
\frac1n\|\hat x_t(y)-\tilde y\|^2.
\end{equation}

\begin{proposition}[Pseudo-validation estimates effective-target risk]
\label{prop:pseudo_effective_target}
Conditioned on $y$ and $s$,
\begin{equation}\nonumber
\mathbb{E}_{\tilde\eta}\!\left[
\tilde V_t \mid y,s
\right]
=
\frac1n\|\hat x_t(y)-(x+s)\|^2
+
\frac1n\mathbb{E}\|\tilde\eta\|^2.
\end{equation}
In particular, if $\tilde\eta$ has per-pixel variance
$\sigma_{\tilde\eta}^2$, then
\begin{equation}\nonumber
\mathbb{E}_{\tilde\eta}\!\left[
\tilde V_t \mid y,s
\right]
=
\frac1n\|\hat x_t(y)-(x+s)\|^2+\sigma_{\tilde\eta}^2.
\end{equation}
\end{proposition}

\begin{proof}
Let
\begin{equation}\nonumber
b_t:=\hat x_t(y)-(x+s),
\end{equation}
then
\begin{equation}\nonumber
\tilde V_t
=
\frac1n\|b_t-\tilde\eta\|^2
=
\frac1n\|b_t\|^2
-\frac{2}{n}\langle b_t,\tilde\eta\rangle
+
\frac1n\|\tilde\eta\|^2.
\end{equation}
Conditioned on $y$ and $s$, $b_t$ is fixed and
$\mathbb{E}\tilde\eta=0$, so the cross term has mean zero. This proves the
claim.
\end{proof}

Proposition~\ref{prop:pseudo_effective_target} gives a conservative
interpretation: pseudo-validation tracks risk to the effective target $x+s$,
not necessarily to the clean image $x$. Therefore, the construction of
$\tilde y$ should make the shared component $s$ small or ensure that $s$ is not
preferentially fitted before the remaining noise.

\subsection{Average-case behavior under DIP noise fitting}

The worst-case in Proposition~\ref{prop:pseudo_effective_target} is often overly pessimistic in DIP for the following reason. In practice, the shared component $s$ and the non-shared component $\eta$ may have the same statistical structure, e.g., both are Gaussian-like noise components. In this case, DIP has no mechanism to preferentially fit $s$ before fitting $\eta$. Instead, the typical optimization trajectory first captures the structured image component $x$, and only later gradually fits the entire noise component $s+\eta$.

Assume that, after the oracle stopping region, the DIP trajectory gradually fits
the noise in $y=x+s+\eta$ according to
\begin{equation}\nonumber
\hat x_t(y)\approx x+\alpha_t(s+\eta),
\qquad
\alpha_t\in[0,1].
\end{equation}
Assume also that $\eta$ and $\tilde\eta$ are independent, zero-mean, and have the
same per-pixel variance $\sigma_\eta^2$. Let
\begin{equation}\nonumber
S:=\frac1n\|s\|^2.
\end{equation}
Then
\begin{equation}\nonumber
\hat x_t(y)-\tilde y
\approx
(\alpha_t-1)s+\alpha_t\eta-\tilde\eta.
\end{equation}
Taking expectation over the non-shared noises conditional on $s$ gives
\begin{equation}\nonumber
\mathbb{E}[\tilde V_t\mid s]
\approx
(1-\alpha_t)^2S
+
\alpha_t^2\sigma_\eta^2
+
\sigma_\eta^2
=
(1-\alpha_t)^2S+(\alpha_t^2+1)\sigma_\eta^2.
\end{equation}

\begin{proposition}[Effect of shared noise in the stylized model]
\label{prop:shared_noise_model}
As a function of $\alpha\in[0,1]$, the expression
\begin{equation}\nonumber
(1-\alpha)^2S+(\alpha^2+1)\sigma_\eta^2
\end{equation}
is minimized at
\begin{equation}\nonumber
\alpha^\star
=
\frac{S}{S+\sigma_\eta^2}.
\end{equation}
\end{proposition}

\begin{proof}
Differentiate:
\begin{equation}\nonumber
\frac{d}{d\alpha}
\left[
(1-\alpha)^2S+(\alpha^2+1)\sigma_\eta^2
\right]
=
2(\alpha-1)S+2\alpha\sigma_\eta^2.
\end{equation}
Setting the derivative to zero gives
\begin{equation}\nonumber
\alpha(S+\sigma_\eta^2)=S,
\end{equation}
and therefore
\begin{equation}\nonumber
\alpha^\star=\frac{S}{S+\sigma_\eta^2}.
\end{equation}
The second derivative is $2S+2\sigma_\eta^2>0$, so this is the minimizer.
\end{proof}

\paragraph{Quadratic improvement over the worst case.}
Suppose the pseudo-reference is constructed so that the shared energy
is a $1/m$ fraction of the non-shared variance,
$S = \sigma_\eta^2/m$ with $m\ge 1$.
Proposition~\ref{prop:shared_noise_model} then gives
$\alpha^\star = 1/(m+1)$, and the expected MSE to the clean image
$x$ at the model-implied stopping time is
\[
    \mathbb{E}\!\left[
        \tfrac{1}{n}\|\hat x_{t^\star}(y)-x\|^2 \,\Big|\, s
    \right]
    \;\approx\;
    (\alpha^\star)^2\,(S+\sigma_\eta^2)
    \;=\;
    \frac{\sigma_\eta^2}{m(m+1)}
    \;=\;
    \mathcal{O}(1/m^2),
\]
which is quadratically tighter than the worst-case excess
$\|s\|^2/n = \sigma_\eta^2/m = \mathcal{O}(1/m)$ implied by
Proposition~\ref{prop:pseudo_effective_target}. Modest decorrelation
($m$ moderately large) is therefore sufficient for near-oracle MSE
under the noise-fitting model.

\subsection{Interpretation and limitations}

The formula for $\alpha^\star$ explains the role of pseudo-reference design.
If $S\ll\sigma_\eta^2$, then $\alpha^\star\approx 0$, so the validation curve is
minimized near the beginning of the noise-fitting phase. This corresponds to a
useful early-stopping signal. If $S\gg\sigma_\eta^2$, then
$\alpha^\star\approx 1$, so the validation minimum may occur late, after DIP has
already fitted much of the noise.

This model is only a stylized description of DIP dynamics. It assumes that the
shared and non-shared components have comparable statistical structure (e.g. both are Gaussian noise), and that
the network fits them at similar rates. It should therefore be read as a design
principle rather than a complete dynamical theory of DIP optimization.

\section{Algorithmic Details}
\label{app:algorithm_details}

This appendix records the pseudocode for the three pseudo-reference early stopping criteria introduced in Section~\ref{sec:algorithms}. The main text gives the defining validation curves; the algorithms below make explicit the inputs, selected stopping time, and returned reconstruction.

\begin{algorithm}[h]
\caption{Channel-Similarity Reference Early Stopping (CSR-ES)}
\label{alg:rgb-similarity}
\begin{algorithmic}[1]
\Require Noisy RGB image $y=(y_R,y_G,y_B)$; DIP trajectory $\{\hat x_t\}_{t=1}^T$
\State Compute pairwise channel distances
\begin{equation}\nonumber
    d_{ij}=\frac1n\|y_i-y_j\|^2,
    \qquad i,j\in\{R,G,B\},\ i\ne j.
\end{equation}
\State Select the closest channel pair
\begin{equation}\nonumber
    (i^\star,j^\star)=\arg\min_{i\ne j}d_{ij}.
\end{equation}
\State Evaluate the cross-channel validation curve
\begin{equation}\nonumber
    V_t^{\rm CSR}=\frac1n\|\hat x_{t,i^\star}-y_{j^\star}\|^2.
\end{equation}
\State Select
\begin{equation}\nonumber
    \hat t_{\rm CSR}=\arg\min_{1\le t\le T}V_t^{\rm CSR}.
\end{equation}
\State \Return $\hat x_{\hat t_{\rm CSR}}$.
\end{algorithmic}
\end{algorithm}

\begin{algorithm}[h]
\caption{Mask-Reference Early Stopping (MR-ES)}
\label{alg:mask-reference}
\begin{algorithmic}[1]
\Require Noisy image $y$; retention probability $p$; DIP iterations $T$
\State Sample or choose a mask $M\in\{0,1\}^n$ with high retention rate, e.g., $\mathbb P(M_i=1)=p=0.98$, and set $H=1-M$.
\State Train DIP using the masked loss
\begin{equation}\nonumber
    \min_\theta \|M\odot(f_\theta(z)-y)\|^2.
\end{equation}
\State Along the resulting trajectory $\{\hat x_t\}_{t=1}^T$, evaluate only on held-out pixels:
\begin{equation}\nonumber
    V_t^{\rm MR}
    =
    \frac1{\sum_iH_i}
    \|H\odot(\hat x_t-y)\|^2.
\end{equation}
\State Select
\begin{equation}\nonumber
    \hat t_{\rm MR}=\arg\min_{1\le t\le T}V_t^{\rm MR}.
\end{equation}
\State \Return $\hat x_{\hat t_{\rm MR}}$.
\end{algorithmic}
\end{algorithm}

\begin{algorithm}[h]
\caption{Augmented-Channel Reference Early Stopping (ACR-ES)}
\label{alg:aug-channel}
\begin{algorithmic}[1]
\Require Noisy image $y$ with $C$ channels; auxiliary noise scale $\tau$; burn-in $t_0$; DIP iterations $T$
\State 
% Generate two pseudo-copies
% \begin{equation}\nonumber
%     y^{(1)}=y+\zeta_1,
%     \qquad
%     y^{(2)}=y+\zeta_2,
%     \qquad
%     \zeta_1,\zeta_2\sim\mathcal N(0,\tau^2I).
% \end{equation}
Generate two independent corrupted measurements
\begin{equation}\nonumber
    y^{(1)} = \mathrm{Corrupt}(y;\zeta_1),
    \qquad
    y^{(2)} = \mathrm{Corrupt}(y;\zeta_2),
\end{equation}
where $\zeta_1$ and $\zeta_2$ with noise level $\tau$ are user-specified auxiliary independent perturbations following the same corruption model. 
\State Construct the augmented fitting target
\begin{equation}\nonumber
    Y=\operatorname{concat}(y,y^{(1)}).
\end{equation}
\State Train a DIP network with $2C$ output channels to fit $Y$, producing outputs $\hat X_t=f_{\theta_t}(z)$.
\State Let $\hat X_{t,{\rm aux}}$ denote the auxiliary output channels $C+1:2C$. Evaluate
\begin{equation}\nonumber
B_t^{\rm ACR}
=
\frac{
\sum_\omega \|\omega\|^2
\left|
\mathcal F(\hat X_{t,{\rm aux}}-y^{(2)})(\omega)
\right|^2
}{
\sum_\omega
\left|
\mathcal F(\hat X_{t,{\rm aux}}-y^{(2)})(\omega)
\right|^2
}.
\end{equation}
\State Select
\begin{equation}\nonumber
    \hat t_{\rm ACR}=\arg\max_{t\ge t_0}B_t^{\rm ACR}.
\end{equation}
\State \Return the primary reconstruction $\hat x_{\hat t_{\rm ACR}}=\hat X_{\hat t_{\rm ACR},1:C}$.
\end{algorithmic}
\end{algorithm}

Unlike MSE-based stopping rules that search for a minimum, the spectral-bias ACR-ES Algorithm \ref{alg:aug-channel} uses an $\arg\max$ criterion.
This is because, due to the spectral bias of DIP, the average-frequency score often follows an increase-then-decrease pattern during DIP training.
In early iterations, the network mainly reconstructs low-frequency structures; as training progresses, meaningful higher-frequency details are recovered and the score increases.
After the near-oracle region, further optimization tends to fit noise or unstable artifacts, reducing the reference-consistent spectral concentration.
Thus, the peak of the spectral-bias score marks the transition from useful detail recovery to overfitting and provides an effective stopping indicator.
\iffalse 
\input{sections/App_Theory_for_Channel-Similarity_Reference}
\fi
\section{Mask-Reference Validation}
\label{app:mask_reference}

This appendix gives a partial justification for the mask-reference early stopping
criterion used by MR-ES. The key point is that, if held-out pixels are excluded
from the training loss, then the noise at those held-out pixels is independent of
the fitted trajectory at those locations.

\subsection{Setup}

Let
\begin{equation}\nonumber
y=x+\xi,
\end{equation}
where $\xi\in\mathbb{R}^n$ has independent, zero-mean, sub-Gaussian entries with
variance proxy $\sigma^2$ and per-pixel variance $\sigma^2$. Let
$M\in\{0,1\}^n$ be a mask independent of $\xi$, and let
\begin{equation}\nonumber
H:=1-M
\end{equation}
be the held-out indicator. Define
\begin{equation}\nonumber
m:=\sum_{i=1}^n H_i.
\end{equation}
Assume that the DIP trajectory is trained using only the retained pixels:
\begin{equation}\nonumber
\theta_t
\approx
\arg\min_\theta
\left\|
M\odot(f_\theta(z)-y)
\right\|^2.
\end{equation}
Thus the held-out pixels do not participate in fitting. Define the held-out
validation loss
\begin{equation}\nonumber
V_t^{\rm mask}
:=
\frac1m
\left\|
H\odot(\hat x_t-y)
\right\|^2
\end{equation}
and the held-out clean risk
\begin{equation}\nonumber
R_H(t)
:=
\frac1m
\left\|
H\odot(\hat x_t-x)
\right\|^2.
\end{equation}

\subsection{Held-out independence}

\begin{lemma}[Independence at held-out positions]
\label{lem:heldout_independence}
For every pixel $i$ with $H_i=1$, the noise entry $\xi_i$ is independent of
$\hat x_t(i)$.
\end{lemma}

\begin{proof}
The training loss uses only pixels $j$ with $M_j=1$. Hence, conditional on the
fixed input $z$ and the mask $M$, the parameters $\theta_t$ are functions only of
the retained observations $\{y_j:M_j=1\}$, equivalently of
$\{x_j+\xi_j:M_j=1\}$. If $H_i=1$, then $M_i=0$, so $\xi_i$ is not used in the
training loss. Since the noise entries are independent and the mask is
independent of the noise, $\xi_i$ is independent of the retained noise variables
and hence independent of $\theta_t$. Therefore $\xi_i$ is independent of
$\hat x_t(i)=f_{\theta_t}(z)_i$.
\end{proof}

\subsection{Held-out validation theorem}

\begin{theorem}[Validation on the held-out region]
\label{thm:mask_validation}
Under the setup above, there exists a universal constant $C>0$ such that, with
probability at least $1-\delta$, uniformly over $1\le t\le T$,
\begin{equation}\nonumber
\left|
V_t^{\rm mask}-R_H(t)-\sigma^2
\right|
\le
C\left[
M_H\sqrt{\frac{\log(T/\delta)}{m}}
+
\sigma^2\frac{\log(T/\delta)}{m}
\right],
\end{equation}
where
\begin{equation}\nonumber
M_H
:=
\sigma\sup_{1\le t\le T}\sqrt{R_H(t)}+\sigma^2.
\end{equation}
Consequently, if
\begin{equation}\nonumber
\hat t\in\arg\min_t V_t^{\rm mask},
\qquad
t_H^\star\in\arg\min_t R_H(t),
\end{equation}
then
\begin{equation}\nonumber
R_H(\hat t)
\le
R_H(t_H^\star)
+
2C\left[
M_H\sqrt{\frac{\log(T/\delta)}{m}}
+
\sigma^2\frac{\log(T/\delta)}{m}
\right].
\end{equation}
\end{theorem}

\begin{proof}
Using $y=x+\xi$,
\begin{equation}\nonumber
V_t^{\rm mask}
=
\frac1m\|H\odot(\hat x_t-x-\xi)\|^2.
\end{equation}
Expanding gives
\begin{equation}\nonumber
V_t^{\rm mask}-R_H(t)-\sigma^2
=
-\frac{2}{m}
\langle H\odot(\hat x_t-x),\xi\rangle
+
\frac1m
\left(
\|H\odot\xi\|^2-m\sigma^2
\right).
\end{equation}
By Lemma~\ref{lem:heldout_independence}, the held-out noise entries are
independent of the fitted trajectory at the held-out pixels. Therefore the cross
term is a sub-Gaussian linear functional conditional on the trajectory and mask,
with deviation of order
\begin{equation}\nonumber
\sigma\sqrt{R_H(t)}
\sqrt{\frac{\log(T/\delta)}{m}}.
\end{equation}
The quadratic term is controlled by standard sub-exponential concentration for
squares of sub-Gaussian variables, with deviation of order
\begin{equation}\nonumber
\sigma^2
\left[
\sqrt{\frac{\log(T/\delta)}{m}}
+
\frac{\log(T/\delta)}{m}
\right].
\end{equation}
A union bound over $t=1,\ldots,T$ yields the uniform concentration inequality.
The oracle inequality follows by the same comparison argument as in
Corollary~\ref{cor:reference_oracle}.
\end{proof}
\vspace{0.2in}
\subsection{Transfer to full-image risk}

Theorem~\ref{thm:mask_validation} controls the held-out risk $R_H(t)$, whereas
the desired reconstruction metric is usually the full-image risk
\begin{equation}\nonumber
R(t):=\frac1n\|\hat x_t-x\|^2.
\end{equation}
A transfer condition is therefore needed.

\begin{corollary}[Transfer from held-out risk to full-image risk]
\label{cor:mask_full_risk}
Suppose that the held-out region is representative of the full image along the
trajectory in the sense that
\begin{equation}\nonumber
\sup_{1\le t\le T}|R_H(t)-R(t)|\le\kappa.
\end{equation}
Then the MR-ES stopping time satisfies
\begin{equation}\nonumber
R(\hat t)
\le
\min_{1\le t\le T}R(t)
+
2\kappa
+
2C\left[
M_H\sqrt{\frac{\log(T/\delta)}{m}}
+
\sigma^2\frac{\log(T/\delta)}{m}
\right].
\end{equation}
\end{corollary}

\begin{proof}
Let $$t^\star\in\arg\min_t R(t),$$ by the representativeness assumption,
\begin{equation}\nonumber
R(\hat t)
\le
R_H(\hat t)+\kappa.
\end{equation}
By Theorem~\ref{thm:mask_validation},
\begin{equation}\nonumber
R_H(\hat t)
\le
R_H(t_H^\star)+2\varepsilon_m
\le
R_H(t^\star)+2\varepsilon_m,
\end{equation}
where
\begin{equation}\nonumber
\varepsilon_m
=
C\left[
M_H\sqrt{\frac{\log(T/\delta)}{m}}
+
\sigma^2\frac{\log(T/\delta)}{m}
\right].
\end{equation}
Again using representativeness,
\begin{equation}\nonumber
R_H(t^\star)\le R(t^\star)+\kappa.
\end{equation}
Combining the three inequalities gives the result.
\end{proof}

\subsection{When is the transfer term small?}

For random masks with independent Bernoulli entries and with the mask independent
of the noise, $R_H(t)$ is an unbiased estimate of $R(t)$ over the randomness of
the mask:
\begin{equation}\nonumber
\mathbb{E}_M[R_H(t)]=R(t),
\end{equation}
up to the usual normalization by the random held-out size. Concentration over the
mask then suggests that $R_H(t)$ is close to $R(t)$ when the held-out set is large
and representative. The mask ratio therefore creates a trade-off: retaining more
pixels preserves the DIP training signal, but holding out fewer pixels increases
the variance of the validation estimate.

\subsection{Implementation requirement}

The validation interpretation requires that held-out pixels do not enter the
training loss. Thus MR-ES should train with
\begin{equation}\nonumber
\|M\odot(f_\theta(z)-y)\|^2,
\end{equation}
or an equivalent loss that ignores held-out pixels. It should not train against
the zero-filled image through
\begin{equation}\nonumber
\|f_\theta(z)-M\odot y\|^2,
\end{equation}
because that objective explicitly supervises the network to output zero at
held-out positions. In that case the held-out pixels are no longer unused, and
Lemma~\ref{lem:heldout_independence} no longer supports the validation
interpretation.

\subsection{Failure modes}

The MR-ES justification has two main limitations. First, for signal-dependent
noise such as Poisson noise, the validation offset is no longer a constant
$\sigma^2$ across pixels. Instead, the noise variance depends on the underlying
intensity. If the held-out region has a different intensity distribution from the
full image, the held-out validation curve may be biased.

Second, deterministic or structured masks, such as masks that mostly select
background regions in medical images, require a domain-specific
representativeness assumption. The held-out validation theorem may still hold on
the selected region, but the transfer from $R_H(t)$ to the full-image risk $R(t)$
depends on the size of $\kappa$.

\section{Augmented-Channel Reference Remarks}
\label{app:acr_theory}

ACR-ES follows the pseudo-reference principle of Appendix~\ref{app:pseudo_reference}, but uses an auxiliary output channel rather than a spatial or color-channel split. Given
\begin{equation}\nonumber
y^{(1)}=y+\zeta_1,\qquad y^{(2)}=y+\zeta_2,
\end{equation}
the two pseudo-copies share the original observation $y$ but contain independent added perturbations. If \(y=x+\xi\), then
\[
y^{(1)} = x+\xi+\zeta_1,\qquad
y^{(2)} = x+\xi+\zeta_2.
\]
Thus ACR-ES can be viewed as an instance of the pseudo-reference decomposition
in Appendix~\ref{app:pseudo_reference} with
\[
s=\xi,\qquad \eta=\zeta_1,\qquad \tilde\eta=\zeta_2.
\]
The shared component is the original observation noise, while the non-shared
components are the added auxiliary perturbations. Increasing the auxiliary perturbation scale increases the non-shared component between fitting and validation, while an excessively large perturbation can interfere with the primary reconstruction task. Thus the auxiliary noise level controls the same shared-versus-independent trade-off described in Appendix~\ref{app:pseudo_reference}.

The frequency-bias score used by ACR-ES should be viewed as a trajectory diagnostic rather than an unbiased risk estimator. It is motivated by the spectral bias of DIP: as training moves from fitting image structure to fitting noise or artifacts, the spectral distribution of the auxiliary residual changes. The empirical ablations in Section~4 compare this score with direct MSE-based alternatives.

\iffalse
\input{sections/App: Beyond DIP: Regularization-Parameter Selection}
\fi 
\section{Additional Experimental Results}
\label{app:add_exp}

This section provides additional experimental evidence supporting the main results in Section~\ref{sec:experiments}. 
The main paper focuses on four representative quantitative studies and three visualizations. 
Here, we provide complementary results that clarify the behavior of the proposed early-stopping criteria from five aspects: 
(i) analysis of the RGB channel-reference selection used by CSR-ES, 
(ii) additional trajectory and qualitative visualizations, 
(iii) per-image method-specific oracle comparisons on natural and medical images, 
(iv) ablations of MSE-based and bias-based reference scores, and 
(v) sensitivity of ACR-ES to auxiliary noise levels.

\subsection{Analysis of RGB Channel References in CSR-ES}
\label{app:csr_grid}

CSR-ES relies on the observation that different RGB channels of a natural image often share correlated structures. 
The method automatically selects a pseudo-reference channel pair from the noisy observation and uses the corresponding cross-channel consistency curve as a validation signal. 
In this section, we examine all possible RGB channel-reference pairs to verify whether the automatically selected pair is reliable.

Table~\ref{tab:mres_rgb_grid} reports the CSR-ES results over all RGB channel-reference pairs. 
Rows correspond to the pseudo-reference channel $y_j$, while columns correspond to the reconstructed channel $\hat{x}_{t,i}$. 
Each entry shows the stopping iteration selected by the validation curve 
\[
\|\hat{x}_{t,i}-y_j\|^2/n
\]
and the corresponding reconstruction PSNR. 
The highlighted entry denotes the pair automatically selected by CSR-ES using the noisy observation only. 
The oracle result is reported only for evaluation and is not used by CSR-ES.

For Gaussian noise, CSR-ES selects the pair $(\hat{x}_R,y_G)$ and stops at iteration $2373$, achieving $21.2249$ dB, exactly matching the oracle-best PSNR.
For Poisson noise, CSR-ES selects $(\hat{x}_G,y_R)$ and stops at iteration $4980$, again exactly matching the oracle-best PSNR of $22.8460$ dB.
For impulse noise, CSR-ES selects $(\hat{x}_G,y_B)$ and stops at iteration $2348$, achieving $23.5807$ dB, only $0.0427$ dB below the oracle maximum of $23.6234$ dB.
These results demonstrate that the channel similarity reference provides a reliable pseudo-validation signal for early stopping without access to the clean image.

% Required packages:
% \usepackage{booktabs}
% \usepackage{makecell}
% \usepackage[table]{xcolor}
% \usepackage{array}

\definecolor{mresblue}{RGB}{230,240,255}
\definecolor{mresline}{gray}{0.78}

% Inner table cells: vertically centered
\newcolumntype{C}[1]{>{\centering\arraybackslash}m{#1}}

% Outer three panels: top aligned
\newcolumntype{P}[1]{>{\centering\arraybackslash}p{#1}}

\newcommand{\EScell}[2]{%
\makecell[c]{%
\textbf{#2}\\[-1pt]
{\scriptsize iter. #1}%
}}

\newcommand{\ESbest}[2]{%
\cellcolor{mresblue}\makecell[c]{%
\textbf{#2}\\[-1pt]
{\scriptsize \textbf{iter. #1}}%
}}

\newcommand{\ESsummary}[4]{%
\vspace{3pt}
{\tiny
\begin{tabular}{@{}ll@{}}
Closest: &  {#1} \\
Selected: &  {#2} \\
Oracle: & #3 \\
  & \textbf{#4}
\end{tabular}
}
}

\begin{table}[t]
\centering
\caption{
CSR-ES selection over RGB channel pairs.
Rows correspond to the RGB channels of $y$, and columns correspond to the RGB channels of $\hat{x}$.
Each cell reports \textbf{PSNR} on the first line and the corresponding iteration (iter.) on the second line.
The selected channel pair is highlighted in blue.
} % urban 38,83,95 denoising
\label{tab:mres_rgb_grid}

\small
\setlength{\tabcolsep}{2pt}
\renewcommand{\arraystretch}{1.18}
\arrayrulecolor{mresline}
\centering
\begin{tabular}{@{}P{0.32\textwidth}P{0.32\textwidth}P{0.32\textwidth}@{}}

% ======================= Gaussian =======================
\begin{minipage}[t]{0.32\textwidth}
\vspace{-0.5em}
\centering
\textbf{Gaussian noise}

\vspace{2pt}
\begin{tabular}{C{0.19\linewidth}|C{0.25\linewidth}C{0.25\linewidth}C{0.25\linewidth}}
\toprule
% $y \backslash \hat{x}$
& $\hat{x}_R$ & $\hat{x}_G$ & $\hat{x}_B$ \\
\midrule
$y_R$
& --%\EScell{4957}{20.60}
& \EScell{2666}{20.70}
& \EScell{3139}{20.84} \\

$y_G$
& \ESbest{2373}{21.22}
& --%\EScell{4948}{20.35}
& \EScell{3139}{20.84} \\

$y_B$
& \EScell{32}{15.12}
& \EScell{3482}{20.89}
& --%\EScell{4959}{20.56} 
\\
\bottomrule
\end{tabular}

\ESsummary{(R, G)}
{$(y_G,\hat{x}_R)$ / $21.22$ / iter. $2373$}
{$21.22$ / iter. $2373$}
{ }
\end{minipage}
&
% ======================= Poisson =======================
\begin{minipage}[t]{0.32\textwidth}
\vspace{-0.5em}
\centering
\textbf{Poisson noise}

\vspace{2pt}
\begin{tabular}{%C{0.19\linewidth}
|C{0.25\linewidth}C{0.25\linewidth}C{0.25\linewidth}}
\toprule
% $y \backslash \hat{x}$ &
 $\hat{x}_R$ & $\hat{x}_G$ & $\hat{x}_B$ \\
\midrule
% $y_R$&
 --%\EScell{4988}{22.78}
& \ESbest{4980}{22.85}
& \EScell{4988}{22.78} \\

% $y_G$&
 \EScell{4925}{22.75}
& --%\EScell{4925}{22.75}
& \EScell{4925}{22.75} \\

% $y_B$&
 \EScell{4937}{22.72}
& \EScell{4974}{22.73}
& --%\EScell{4980}{22.85} 
\\
\bottomrule
\end{tabular}

\ESsummary{(R, G)}
{$(y_R,\hat{x}_G)$ / $22.85$ / iter. $4980$}
{$22.85$ / iter. $4980$}
{ }
\end{minipage}
&
\hspace{-0.5in}
% ======================= Impulse =======================
\begin{minipage}[t]{0.32\textwidth}
\centering
\textbf{Impulse noise}

\vspace{2pt}
\begin{tabular}{%C{0.19\linewidth}
|C{0.25\linewidth}C{0.25\linewidth}C{0.25\linewidth}}
\toprule
% $y \backslash \hat{x}$&
 $\hat{x}_R$ & $\hat{x}_G$ & $\hat{x}_B$ \\
\midrule
% $y_R$&
 -- %\EScell{4986}{22.40}
& \EScell{20}{15.48}
& \EScell{20}{15.48} \\

% $y_G$&
 \EScell{3036}{23.47}
& --%\EScell{4991}{22.41}
& \EScell{2588}{23.37} \\

% $y_B$&
 \EScell{3036}{23.47}
& \ESbest{2348}{23.58}
& --%\EScell{4986}{22.40} 
\\
\bottomrule
\end{tabular}

\ESsummary{(G, B)}
{$(y_B,\hat{x}_G)$ / $23.58$ / iter. $2348$}
{$23.62$ / iter. $2414$}
{ }
\end{minipage}

\end{tabular}

\arrayrulecolor{black}

\vspace{-0.2in}
\end{table}

\begin{figure}
\begin{minipage}{0.45\linewidth}
\centerline{\includegraphics[width=2.8in]{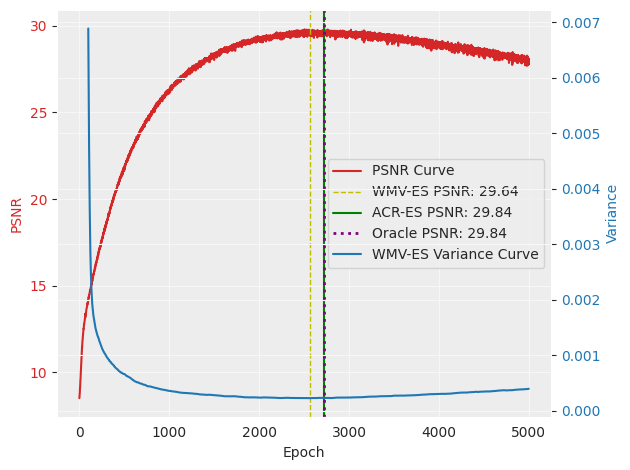}}
\centerline{\scriptsize (a) WMV-ES  curve with ACR-ES training
}
\end{minipage}\hspace{0.35in}
\begin{minipage}{0.45\linewidth}
\centerline{\includegraphics[width=2.8in]{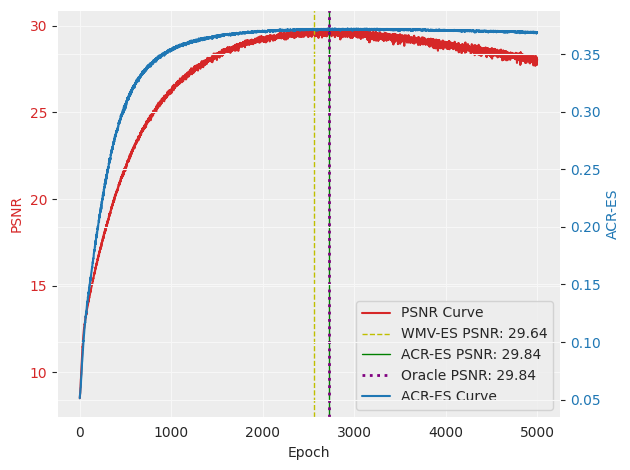}}
\centerline{\scriptsize (b) ACR-ES  curve with ACR-ES training
}
\end{minipage}

\begin{minipage}{0.45\linewidth}
\centerline{\includegraphics[width=2.8in]{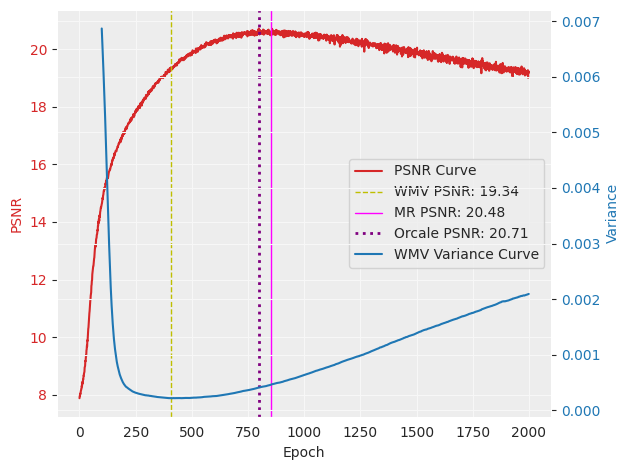}}
\centerline{\scriptsize (c) WMV-ES  curve with MR-ES training
}
\end{minipage}\hspace{0.35in}
\begin{minipage}{0.45\linewidth}
\centerline{\includegraphics[width=2.8in]{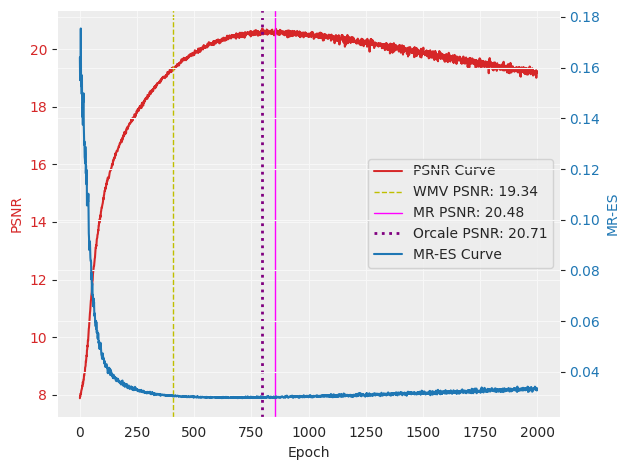}}
\centerline{\scriptsize (d) MR-ES  curve with MR-ES training
}
\end{minipage}
\caption{Early stop for image denoising. First row is the nature image from Set18 corrupted by Poisson noise, the second row is the CT image from CT10 corrupted by Gaussian noise.}\label{fig:acr_mr_curve_extra}
\end{figure}

\subsection{Additional Trajectory and Qualitative Results}
\label{app:additional_curves}

The main paper includes representative trajectory plots showing that the proposed criteria select iterations close to the oracle region. 
Here, we provide additional curves and qualitative examples to show that this behavior is consistent across different image types and corruption models.

Figure~\ref{fig:acr_mr_curve_extra} compares WMV-ES with ACR-ES and MR-ES on representative natural and medical images. 
The WMV curve may reach its minimum before the PSNR peak, leading to premature stopping. 
In contrast, the proposed reference-based scores better track the region where the reconstruction quality is close to the oracle.

Figures~\ref{fig:acr_poisson_visual} and~\ref{fig:acr_gaussian_visual} provide qualitative denoising examples under Poisson and Gaussian noise. 
Each figure shows the noisy observation, the auxiliary corrupted references, the reconstruction selected by WMV-ES, the reconstruction selected by ACR-ES, and the oracle reconstruction. 
The selected iteration of ACR-ES is closer to the oracle iteration, and the corresponding reconstruction is visually closer to the oracle-quality result.

\begin{figure}
\centering
\begin{minipage}{0.3\linewidth}
\centerline{\includegraphics[width=1.8in]{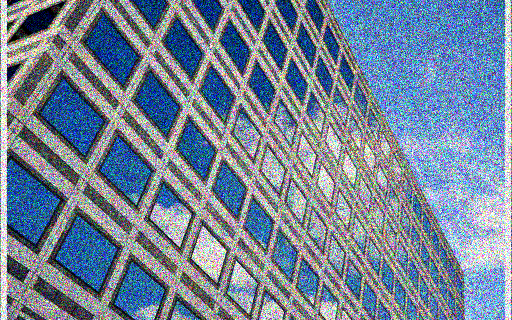}}\vspace{-0.05in}
\centerline{\scriptsize (a) Poisson noise $y$
}
\end{minipage} 
\begin{minipage}{0.3\linewidth}
\centerline{\includegraphics[width=1.8in]{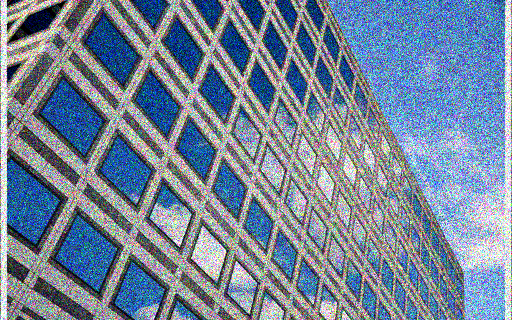}}\vspace{-0.05in}
\centerline{\scriptsize (b) Poisson noise $y_1=y+\zeta_1$}
\end{minipage}
\begin{minipage}{0.3\linewidth}
\centerline{\includegraphics[width=1.8in]{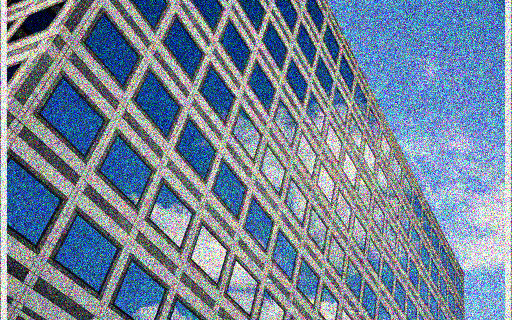}}\vspace{-0.05in}
\centerline{\scriptsize (c) Poisson noise $y_2=y+\zeta_2$}
\end{minipage}

\begin{minipage}{0.3\linewidth}
\centerline{\includegraphics[width=1.8in]{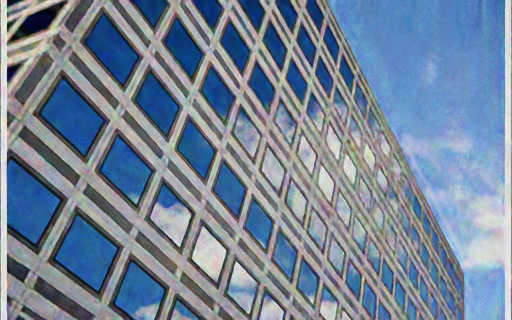}}\vspace{-0.05in}
\centerline{\scriptsize (d) WMV-ES at Iter. 3069
}
\end{minipage} 
\begin{minipage}{0.3\linewidth}
\centerline{\includegraphics[width=1.8in]{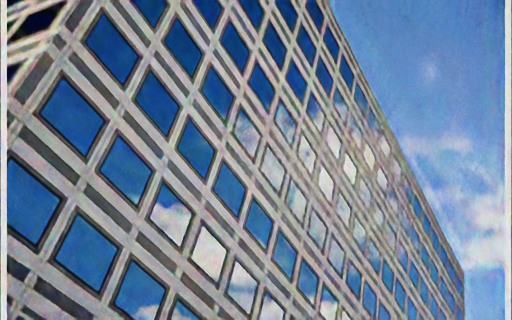}}\vspace{-0.05in}
\centerline{\scriptsize (e) ACR-ES at Iter. 2216}
\end{minipage}
\begin{minipage}{0.3\linewidth}
\centerline{\includegraphics[width=1.8in]{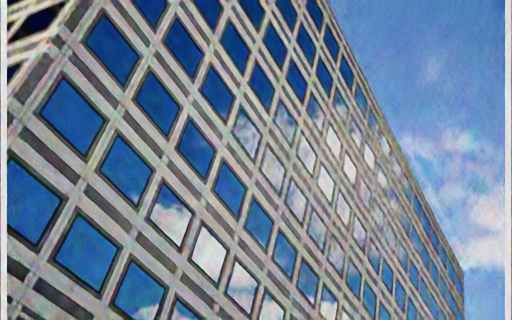}}\vspace{-0.05in}
\centerline{\scriptsize (f) Oracle at Iter. 2307}
\end{minipage}

\begin{minipage}{0.45\linewidth}
\centerline{\includegraphics[width=2.8in]{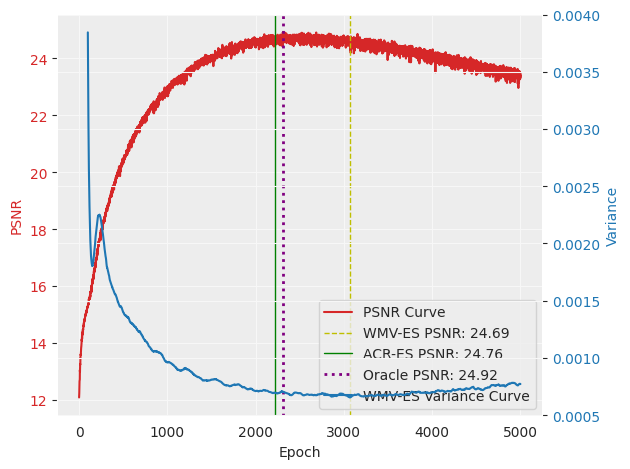}}\vspace{-0.05in}
\centerline{\scriptsize (g) WMV-ES curve with ACR-ES training}
\end{minipage}\hspace{0.35in}
\begin{minipage}{0.45\linewidth}
\centerline{\includegraphics[width=2.8in]{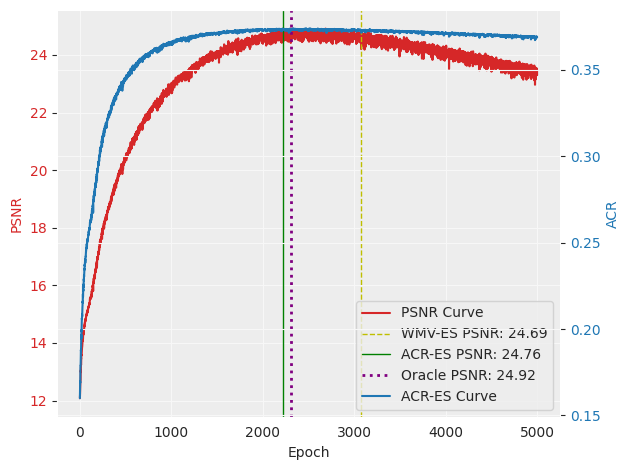}}\vspace{-0.05in}
\centerline{\scriptsize (h) ACR-ES curve with ACR-ES training}
\end{minipage}
\caption{Image denoising results on Poisson noise. Comparison of WMV-ES and our ACR-ES. }\label{fig:acr_poisson_visual}
\end{figure}
\begin{figure}
\centering
\begin{minipage}{0.3\linewidth}
\centerline{\includegraphics[width=1.6in]{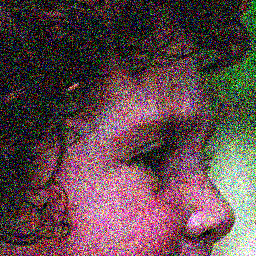}}\vspace{-0.05in}
\centerline{\scriptsize (a) Gaussian noise $y$
}
\end{minipage} 
\begin{minipage}{0.3\linewidth}
\centerline{\includegraphics[width=1.6in]{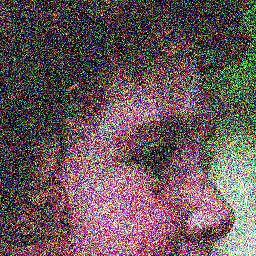}}\vspace{-0.05in}
\centerline{\scriptsize (b) Gaussian noise $y_1=y+\zeta_1$}
\end{minipage}
\begin{minipage}{0.3\linewidth}
\centerline{\includegraphics[width=1.6in]{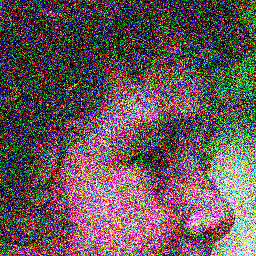}}\vspace{-0.05in}
\centerline{\scriptsize (c) Gaussian noise $y_2=y+\zeta_2$}
\end{minipage}

\begin{minipage}{0.3\linewidth}
\centerline{\includegraphics[width=1.6in]{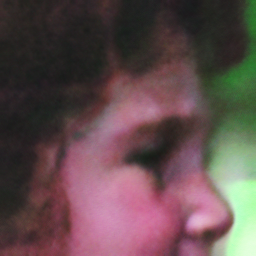}}\vspace{-0.05in}
\centerline{\scriptsize (d) WMV-ES at Iter. 534
}
\end{minipage} 
\begin{minipage}{0.3\linewidth}
\centerline{\includegraphics[width=1.6in]{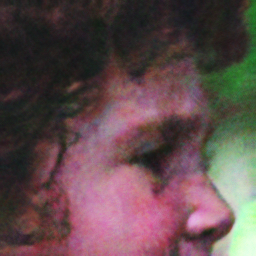}}\vspace{-0.05in}
\centerline{\scriptsize (e) ACR-ES at Iter. 1004}
\end{minipage}
\begin{minipage}{0.3\linewidth}
\centerline{\includegraphics[width=1.6in]{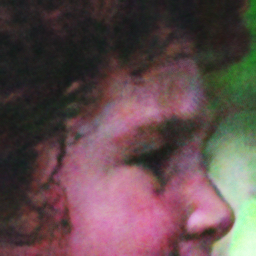}}\vspace{-0.05in}
\centerline{\scriptsize (f) Oracle at Iter. 988}
\end{minipage}

\begin{minipage}{0.45\linewidth}
\centerline{\includegraphics[width=2.8in]{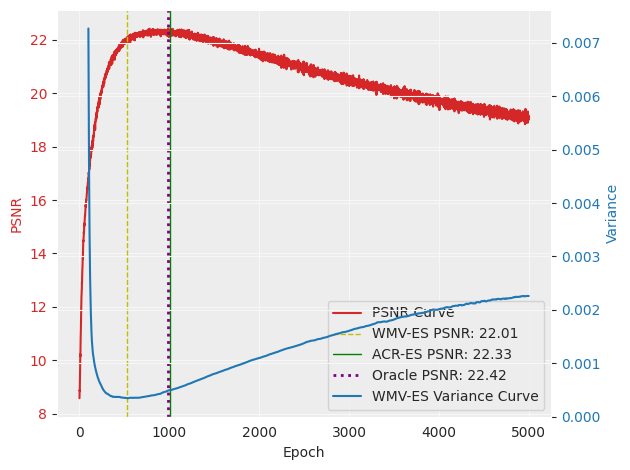}}\vspace{-0.05in}
\centerline{\scriptsize (g) WMV-ES curve with ACR-ES training}
\end{minipage}\hspace{0.35in}
\begin{minipage}{0.45\linewidth}
\centerline{\includegraphics[width=2.8in]{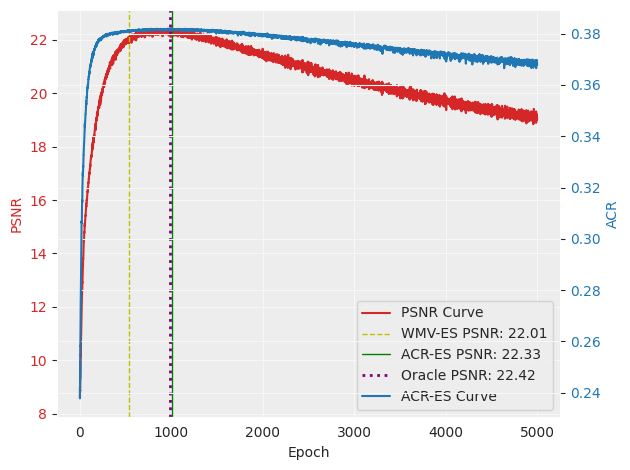}}\vspace{-0.05in}
\centerline{\scriptsize (h) ACR-ES curve with ACR-ES training}
\end{minipage}
\caption{Image denoising results on Gaussian noise. Comparison of WMV-ES and our ACR-ES. }\label{fig:acr_gaussian_visual}
\end{figure}

Figure \ref{fig:ct_acr_gaussian_visual} shows the denoising results of CT image corrupted by the Gaussian noise. We provide the noisy observation, two auxiliary corrupted references, the reconstruction selected by WMV-ES, the reconstruction selected by ACR-ES, and the oracle reconstruction. The corresponding WMV-ES and ACR-ES curves are also given. The selected iteration of ACR-ES is closer to the oracle iteration. Besides, we also present the denoising results of T2 image corrupted by the Impulse noise in Figure \ref{fig:t2_acr_gaussian_visual}. We show the noisy image, the reconstruction selected by WMV-ES, the reconstruction selected by MR-ES, and the oracle reconstruction. The corresponding WMV-ES and MR-ES curves are also given. The selected iteration of MR-ES is closer to the oracle iteration.
\begin{figure}\centering
\begin{minipage}{0.3\linewidth}
\centerline{\includegraphics[width=1.6in]{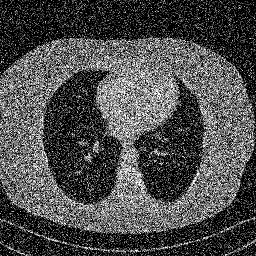}}\vspace{-0.05in}
\centerline{\scriptsize (a) Gaussian noise $y$
}
\end{minipage} 
\begin{minipage}{0.3\linewidth}
\centerline{\includegraphics[width=1.6in]{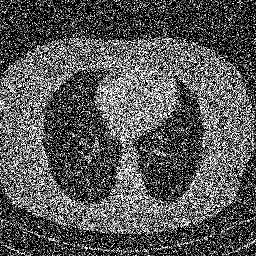}}\vspace{-0.05in}
\centerline{\scriptsize (b) Gaussian noise $y_1=y+\zeta_1$}
\end{minipage}
\begin{minipage}{0.3\linewidth}
\centerline{\includegraphics[width=1.6in]{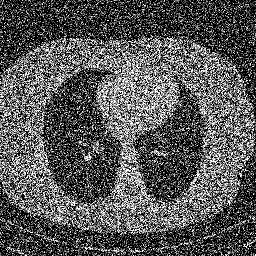}}\vspace{-0.05in}
\centerline{\scriptsize (c) Gaussian noise $y_2=y+\zeta_2$}
\end{minipage}

\begin{minipage}{0.3\linewidth}
\centerline{\includegraphics[width=1.6in]{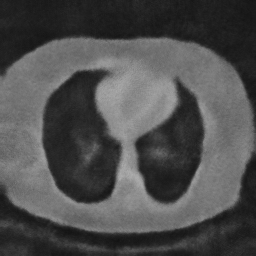}}\vspace{-0.05in}
\centerline{\scriptsize (d) WMV-ES at Iter. 486
}
\end{minipage} 
\begin{minipage}{0.3\linewidth}
\centerline{\includegraphics[width=1.6in]{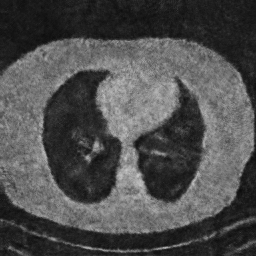}}\vspace{-0.05in}
\centerline{\scriptsize (e) ACR-ES at Iter. 954}
\end{minipage}
\begin{minipage}{0.3\linewidth}
\centerline{\includegraphics[width=1.6in]{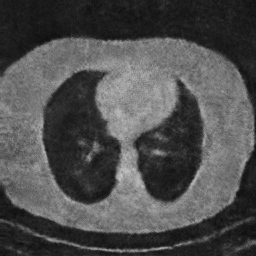}}\vspace{-0.05in}
\centerline{\scriptsize (f) Oracle at Iter. 812}
\end{minipage}

\begin{minipage}{0.45\linewidth}
\centerline{\includegraphics[width=2.8in]{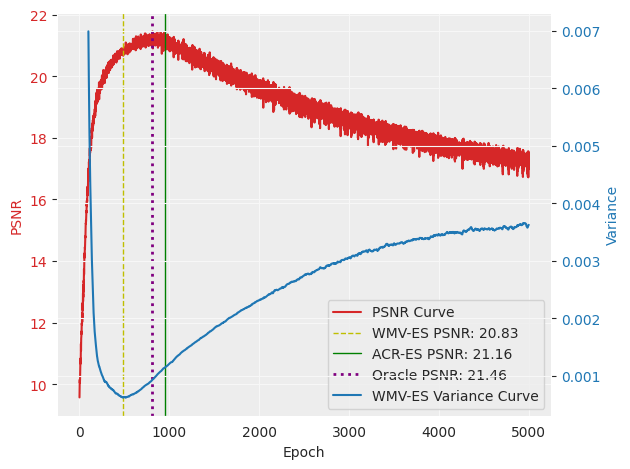}}\vspace{-0.05in}
\centerline{\scriptsize (g) WMV-ES curve with ACR-ES training}
\end{minipage}\hspace{0.35in}
\begin{minipage}{0.45\linewidth}
\centerline{\includegraphics[width=2.8in]{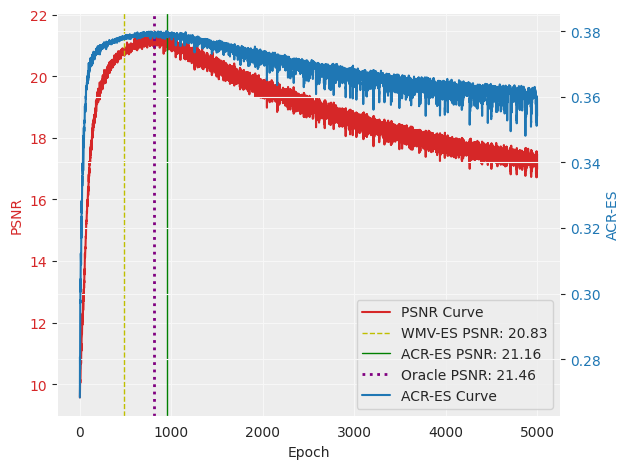}}\vspace{-0.05in}
\centerline{\scriptsize (h) ACR-ES curve with ACR-ES training}
\end{minipage}
\caption{Image denoising results on Gaussian noise. Comparison of WMV-ES and our ACR-ES. }\label{fig:ct_acr_gaussian_visual}
\end{figure}
\begin{figure}\centering
\begin{minipage}{0.22\linewidth}
\centerline{\includegraphics[width=1.4in]{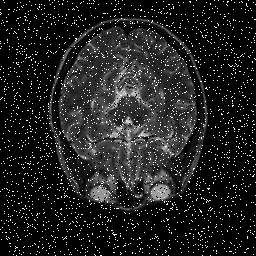}}\vspace{-0.05in}
\centerline{\scriptsize (a) Impulse noise $y$
}
\end{minipage} \hspace{0.15in}
\begin{minipage}{0.22\linewidth}
\centerline{\includegraphics[width=1.4in]{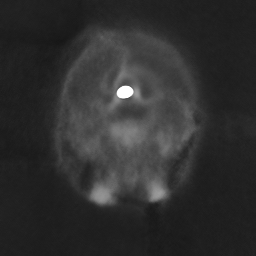}}\vspace{-0.05in}
\centerline{\scriptsize (b) WMV-ES at Iter. 382
}
\end{minipage} \hspace{0.15in}
\begin{minipage}{0.22\linewidth}
\centerline{\includegraphics[width=1.4in]{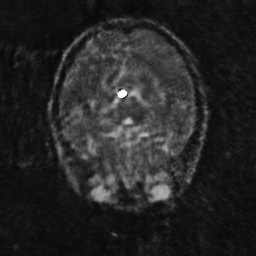}}\vspace{-0.05in}
\centerline{\scriptsize (c) MR-ES at Iter. 1090}
\end{minipage}\hspace{0.18in}
\begin{minipage}{0.22\linewidth}
\centerline{\includegraphics[width=1.4in]{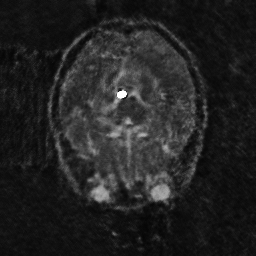}}\vspace{-0.05in}
\centerline{\scriptsize (d) Oracle at Iter. 1243}
\end{minipage}

\begin{minipage}{0.45\linewidth}
\centerline{\includegraphics[width=2.8in]{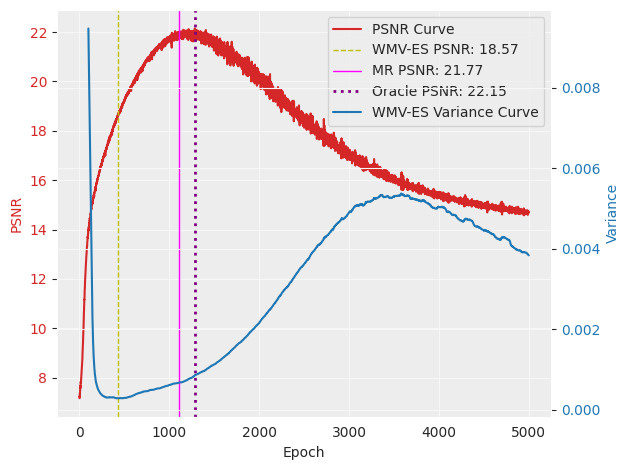}}\vspace{-0.05in}
\centerline{\scriptsize (e) WMV-ES curve with MR-ES training}
\end{minipage}\hspace{0.35in}
\begin{minipage}{0.45\linewidth}
\centerline{\includegraphics[width=2.8in]{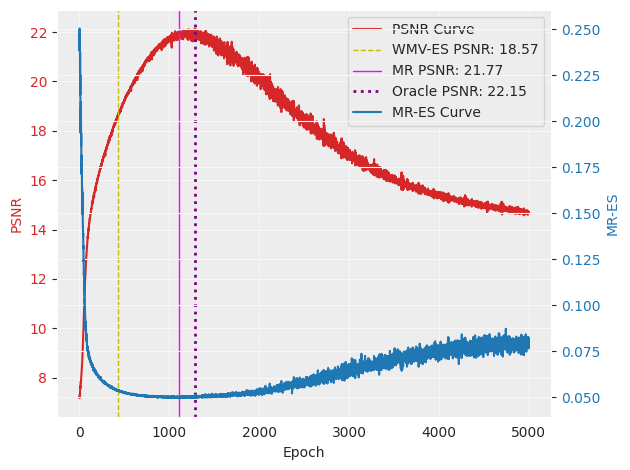}}\vspace{-0.05in}
\centerline{\scriptsize (f) MR-ES curve with MR-ES training}
\end{minipage}
\caption{Image denoising results on Impulse noise. Comparison of WMV-ES and our MR-ES. }\label{fig:t2_acr_gaussian_visual}
\end{figure}

Moreover, we also present the comparison of different training methods among the original DIP, the proposed ACR-ES, and the MR-ES training in Figure \ref{fig:3train}. We test a single image from the CT10 dataset under three different noise types: Gaussian, Poisson, and Impulse. 
% From the results, we know that our ACR-ES and MR-ES do not change the original DIP training.  
Although the reference construction and input perturbations slightly modify the optimization trajectory, the overall PSNR evolution and attainable oracle performance remain comparable to those of the original DIP. 
Thus, the proposed methods improve early stopping mainly through more reliable stopping signals, rather than by substantially changing the DIP training process.

\begin{figure}
\begin{minipage}{0.31\linewidth}
\centerline{\includegraphics[width=1.85in]{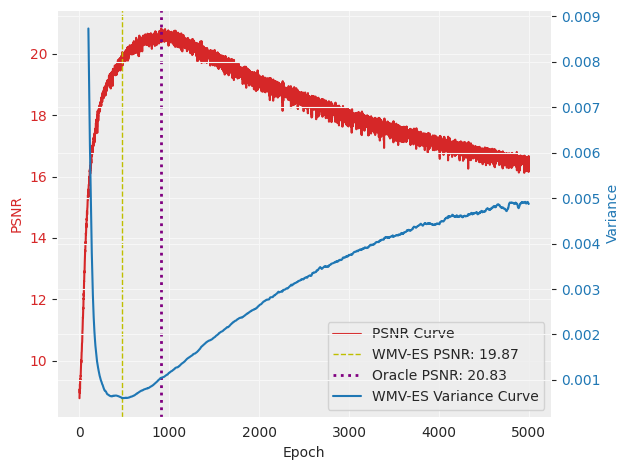}}\vspace{-0.05in}
\centerline{\scriptsize (a) WMV-ES curve with original DIP training
}
\end{minipage}\hspace{0.1in}
\begin{minipage}{0.31\linewidth}
\centerline{\includegraphics[width=1.85in]{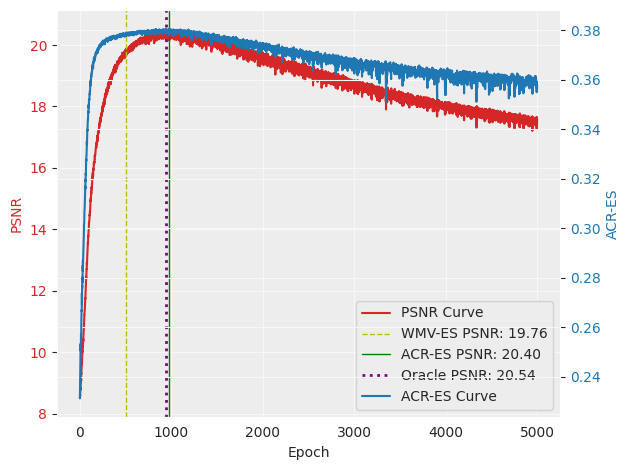}}\vspace{-0.05in}
\centerline{\scriptsize (b) ACR-ES curve with ACR-ES training
}
\end{minipage}\hspace{0.1in}
\begin{minipage}{0.31\linewidth}
\centerline{\includegraphics[width=1.85in]{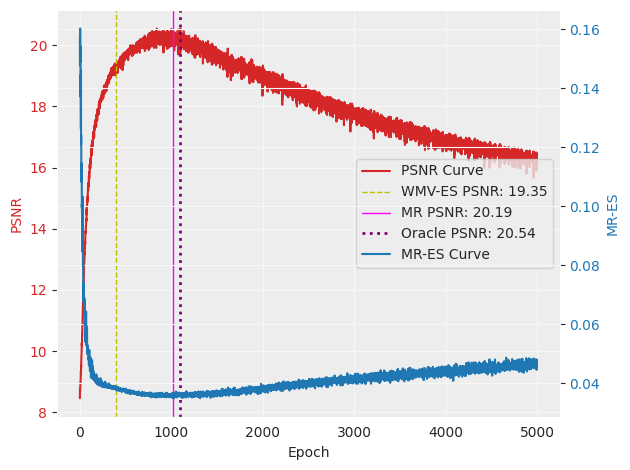}}\vspace{-0.05in}
\centerline{\scriptsize (c) MR-ES curve with MR-ES training
}
\end{minipage}

\begin{minipage}{0.31\linewidth}
\centerline{\includegraphics[width=1.85in]{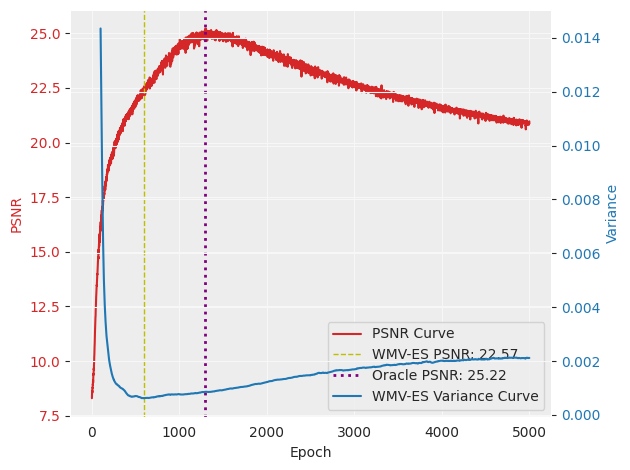}}\vspace{-0.05in}
\centerline{\scriptsize (d) WMV-ES curve with original DIP training
}
\end{minipage}\hspace{0.1in}
\begin{minipage}{0.31\linewidth}
\centerline{\includegraphics[width=1.85in]{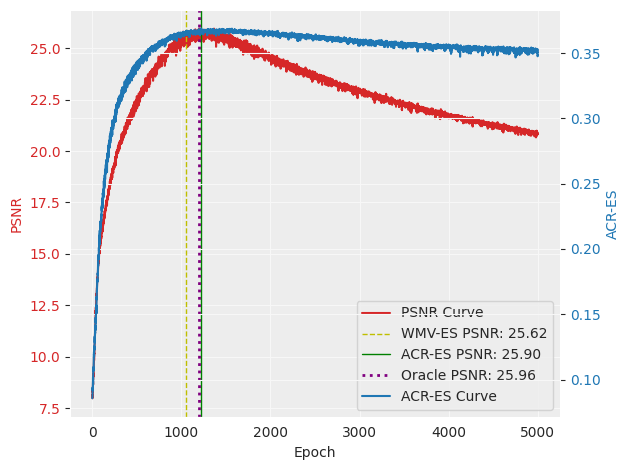}}\vspace{-0.05in}
\centerline{\scriptsize (e) ACR-ES curve with ACR-ES training
}
\end{minipage}\hspace{0.1in}
\begin{minipage}{0.31\linewidth}
\centerline{\includegraphics[width=1.85in]{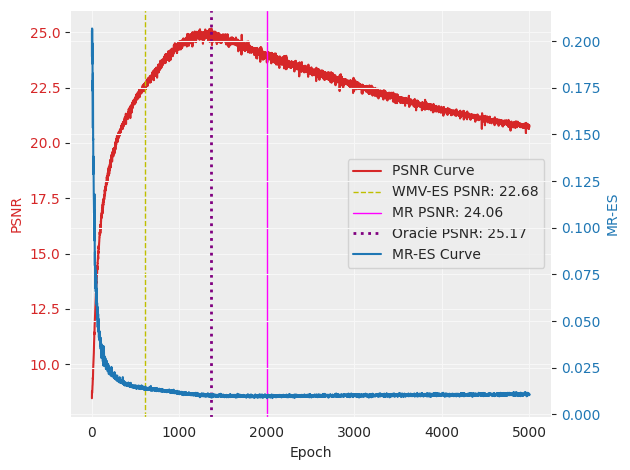}}\vspace{-0.05in}
\centerline{\scriptsize (f) MR-ES curve with MR-ES training
}
\end{minipage}

\begin{minipage}{0.31\linewidth}
\centerline{\includegraphics[width=1.85in]{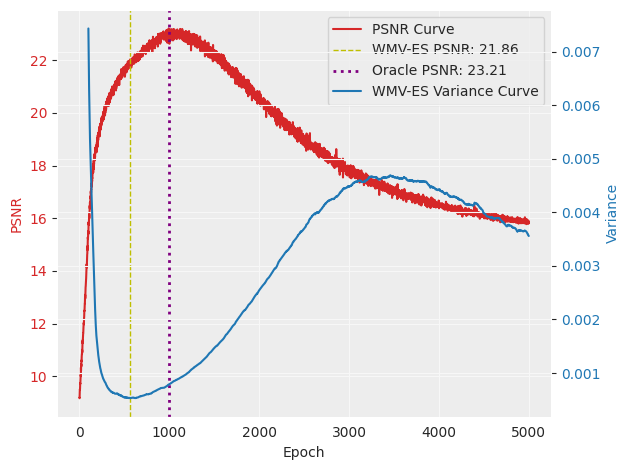}}\vspace{-0.05in}
\centerline{\scriptsize (g) WMV-ES curve with original DIP training
}
\end{minipage}\hspace{0.1in}
\begin{minipage}{0.31\linewidth}
\centerline{\includegraphics[width=1.85in]{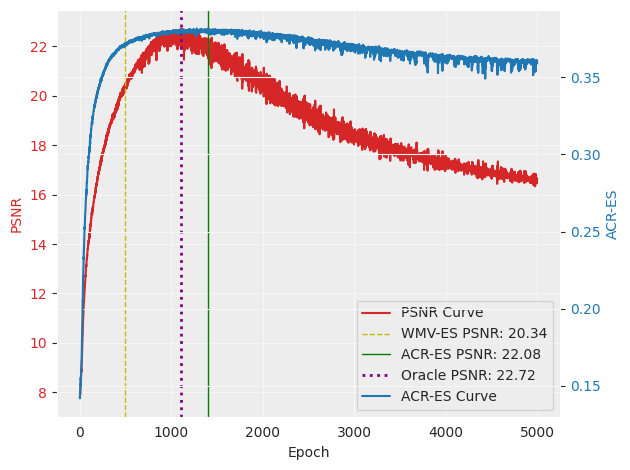}}\vspace{-0.05in}
\centerline{\scriptsize (h) ACR-ES curve with ACR-ES training
}
\end{minipage}\hspace{0.1in}
\begin{minipage}{0.31\linewidth}
\centerline{\includegraphics[width=1.85in]{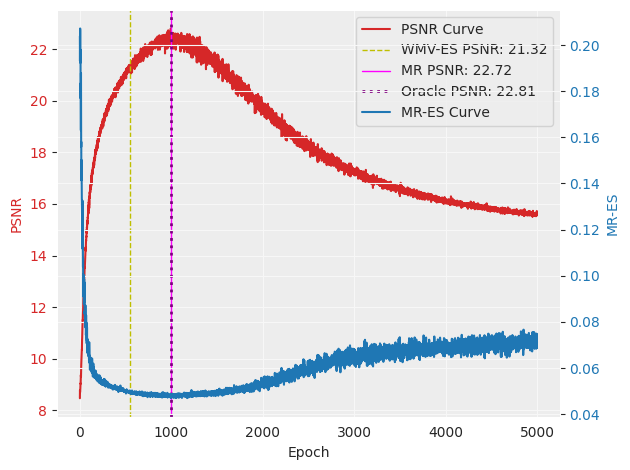}}\vspace{-0.05in}
\centerline{\scriptsize (i) MR-ES curve with MR-ES training
}
\end{minipage}
\caption{Comparison of different training curves of WMV-ES and the proposed reference-based stopping criteria. Test on CT01 image corrupted by Gaussian noise in the first row; Poisson noise in the second row; and Impulse noise in the third row.}\label{fig:3train}
\end{figure}

We further compare the proposed reference-based stopping criteria with SURE-ES and WMV-ES through representative stopping-score trajectories.
Although SURE provides an unbiased estimate of the reconstruction risk under specific Gaussian-noise assumptions, its empirical minimum does not always coincide with the oracle stopping region of DIP optimization.
Figure~\ref{fig:sure_curve_extra} shows examples on a natural image and a medical image corrupted by Gaussian noise.
For the natural image, we compare the SURE curve on the original DIP trajectory with WMV-ES and MR-ES curves on the corresponding reference-based training trajectories.
For the medical image, we compare SURE-ES with the stopping curve obtained under ACR-ES training.
These examples show that the proposed reference-based stopping signals are better aligned with the near-oracle region than SURE-ES or WMV-ES in these representative cases.
This complements the quantitative comparisons in the main paper and further illustrates the robustness of the proposed stopping principle across image domains and training trajectories.

\begin{figure}
\begin{minipage}{0.45\linewidth}
\centerline{\includegraphics[width=2.8in]{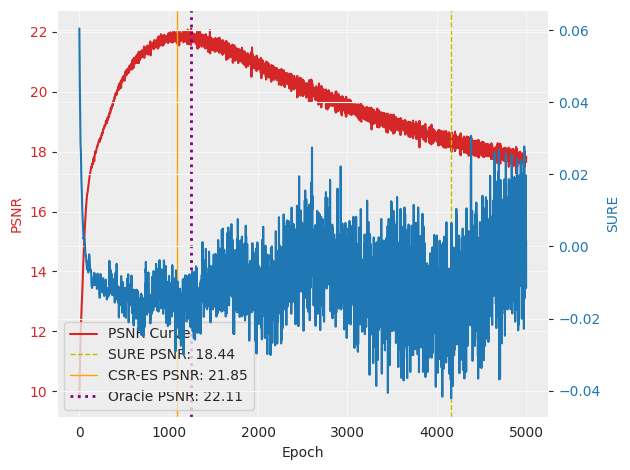}}\vspace{-0.05in}
\centerline{\scriptsize (a) SURE curve with original DIP training
}
\end{minipage}\hspace{0.35in}
\begin{minipage}{0.45\linewidth}
\centerline{\includegraphics[width=2.8in]{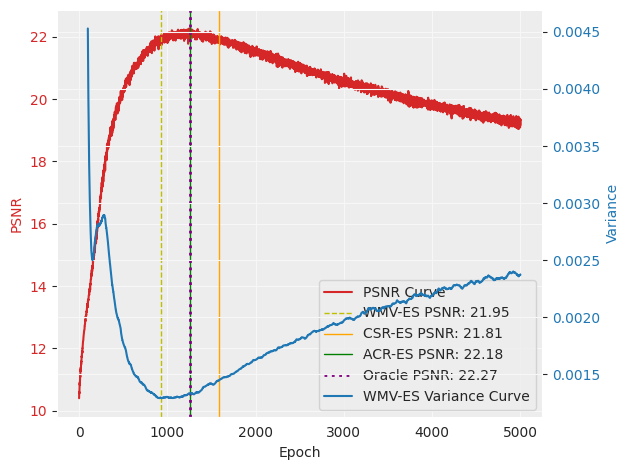}}\vspace{-0.05in}
\centerline{\scriptsize (b) WMV-ES curve with ACR-ES training
}
\end{minipage}

\begin{minipage}{0.45\linewidth}
\centerline{\includegraphics[width=2.8in]{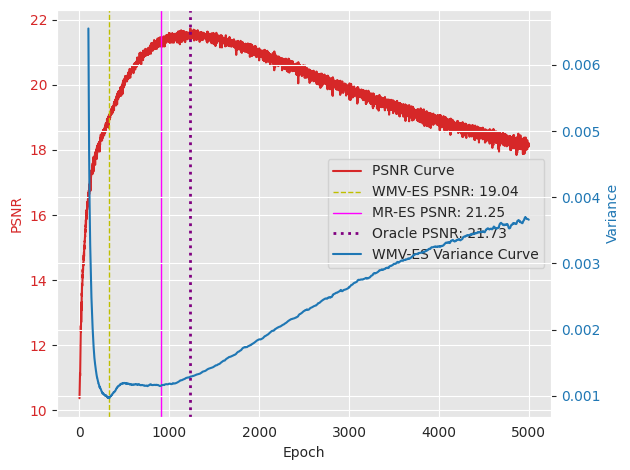}}\vspace{-0.05in}
\centerline{\scriptsize (c) WMV-ES curve with MR-ES training
}
\end{minipage}\hspace{0.35in}
\begin{minipage}{0.45\linewidth}
\centerline{\includegraphics[width=2.8in]{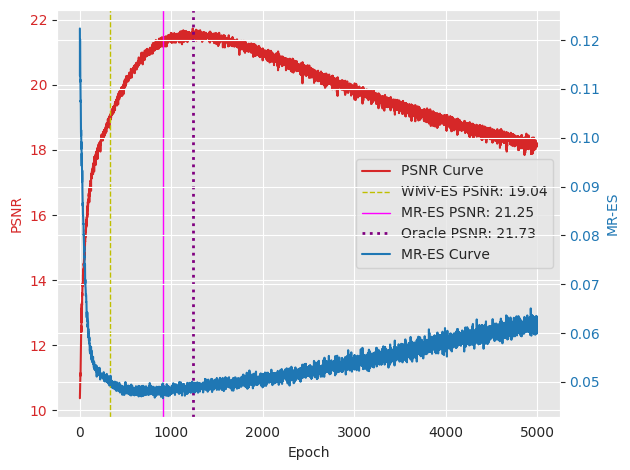}}\vspace{-0.05in}
\centerline{\scriptsize (d) MR-ES curve with MR-ES training
}
\end{minipage}

\begin{minipage}{0.45\linewidth}
\centerline{\includegraphics[width=2.8in]{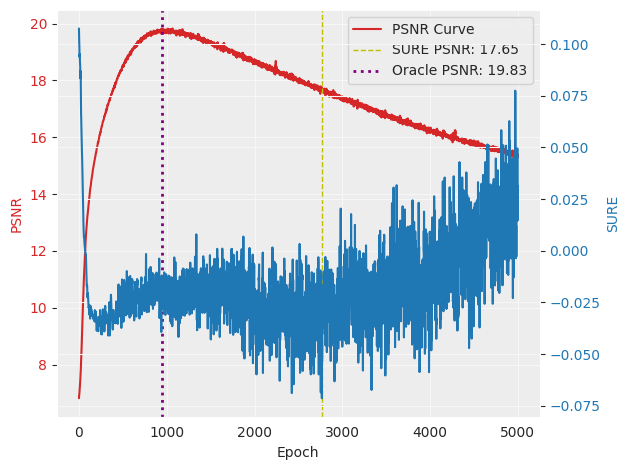}}\vspace{-0.05in}
\centerline{\scriptsize (e) SURE curve with original DIP training
}
\end{minipage}\hspace{0.35in}
\begin{minipage}{0.45\linewidth}
\centerline{\includegraphics[width=2.8in]{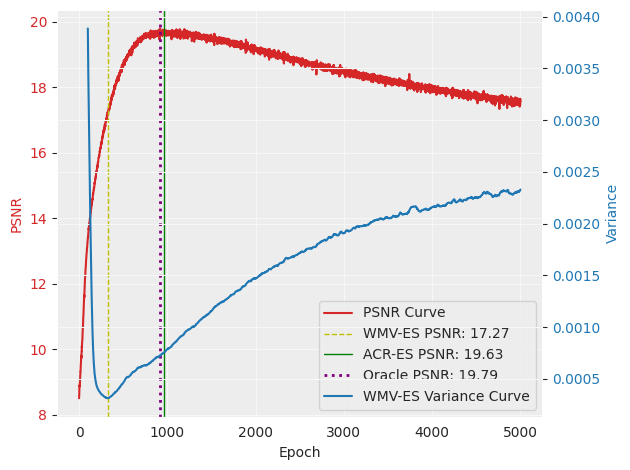}}\vspace{-0.05in}
\centerline{\scriptsize (f) WMS-ES curve with ACR training
}
\end{minipage}
\caption{Comparison of different training curves of SURE, WMV-ES, and the proposed reference-based stopping criteria. First and second row are the nature image from Set3 corrupted by Gaussian noise, the third row is the CT image from T2-20 corrupted by Gaussian noise.}\label{fig:sure_curve_extra}
\end{figure}

\begin{table}[t]
\centering
\caption{
Representative early-stopping comparison on selected natural images. 
Each PSNR entry is reported as ES PSNR (Oracle PSNR), where the oracle PSNR is computed from the corresponding training trajectory of each method. 
Gap denotes Oracle PSNR minus ES PSNR; a smaller Gap indicates a more accurate stopping point.
}
\label{tab:representative_es_gap}
\small
\setlength{\tabcolsep}{4.2pt}
\renewcommand{\arraystretch}{1.18}
\resizebox{\linewidth}{!}{
\begin{tabular}{llcc cc cc cc cc}
\toprule
\multirow{3}{*}{Noise}
& \multirow{3}{*}{Image}
& \multicolumn{2}{c}{SURE-ES}
& \multicolumn{2}{c}{WMV-ES}
& \multicolumn{2}{c}{MR-ES}
& \multicolumn{2}{c}{CSR-ES}
& \multicolumn{2}{c}{ACR-ES} \\
\cmidrule(lr){3-4}
\cmidrule(lr){5-6}
\cmidrule(lr){7-8}
\cmidrule(lr){9-10}
\cmidrule(lr){11-12}
&
& \makecell{PSNR\\(Oracle)} & Gap $\downarrow$
& \makecell{PSNR\\(Oracle)} & Gap $\downarrow$
& \makecell{PSNR\\(Oracle)} & Gap $\downarrow$
& \makecell{PSNR\\(Oracle)} & Gap $\downarrow$
& \makecell{PSNR\\(Oracle)} & Gap $\downarrow$ \\
\midrule

\multirow{10}{*}{Gaussian}
& Leaves
& 20.06 (21.43) & 1.37
& 20.64 (20.89) & 0.25
& 20.62 (20.84) & 0.22
& 20.59 (20.70) & 0.11
& \textbf{20.62 (20.70)} & \textbf{0.08} \\
& Starfish
& 21.31 (22.04) & 0.73
& 18.96 (21.73) & 2.77
& 21.23 (21.72) & 0.49
& 21.39 (21.69) & 0.30
& \textbf{21.52 (21.69)} & \textbf{0.17} \\
& Baboon
& 19.50 (19.92) & 0.42
& 18.30 (19.63) & 1.33
& 19.37 (19.62) & 0.25
& 19.22 (19.50) & 0.28
& \textbf{19.42 (19.50)} & \textbf{0.08} \\
& Barbara
& 21.81 (23.01) & 1.20
& 21.61 (22.66) & 1.05
& 22.30 (22.58) & 0.29
& \textbf{22.46 (22.55)} & \textbf{0.09}
& 22.21 (22.55) & 0.34 \\
& Coastguard
& 22.82 (23.30) & 0.48
& 21.87 (23.10) & 1.22
& 21.95 (22.75) & 0.80
& \textbf{22.77 (22.86)} & \textbf{0.09}
& 22.59 (22.86) & 0.27 \\
& Comic
& 18.46 (20.58) & 2.11
& 17.94 (20.07) & 2.12
& 19.76 (19.86) & 0.11
& 19.81 (19.94) & 0.13
& \textbf{19.91 (19.94)} & \textbf{0.03} \\
& Face
& 22.03 (22.72) & 0.69
& 22.39 (22.66) & 0.27
& 22.26 (22.49) & 0.23
& 22.27 (22.51) & 0.24
& \textbf{22.42 (22.51)} & \textbf{0.09} \\
& Flowers
& 21.93 (22.35) & 0.42
& 21.07 (21.90) & 0.83
& 21.39 (21.72) & 0.34
& \textbf{21.62 (21.82)} & \textbf{0.21}
& 21.61 (21.82) & 0.21 \\
& Foreman
& 20.36 (24.06) & 3.71
& 23.50 (23.88) & 0.38
& 23.47 (23.79) & 0.31
& \textbf{23.71 (23.82)} & \textbf{0.12}
& 23.63 (23.82) & 0.19 \\
& Zebra
& 22.59 (23.04) & 0.45
& 22.09 (22.61) & 0.52
& 22.28 (22.49) & 0.22
& \textbf{22.33 (22.45)} & \textbf{0.12}
& 22.13 (22.45) & 0.32 \\
\midrule

\multirow{10}{*}{Poisson}
& Leaves
& 20.39 (20.83) & 0.44 
& 20.33 (20.47) & 0.14
& 20.32 (20.43) & 0.11
& \textbf{20.92 (21.03)} & \textbf{0.11}
& 20.89 (21.03) & 0.13 \\
& Starfish
& 22.92 (23.16) & 0.24 
& 22.15 (22.69) & 0.54
& 22.69 (22.86) & 0.18
& 23.40 (23.51) & 0.11
& \textbf{23.42 (23.51)} & \textbf{0.09} \\
& Baboon
& 18.90 (20.97) & 2.07
& 18.66 (20.56) & 1.90
& 20.37 (20.54) & 0.17
& \textbf{20.46 (20.54)} & \textbf{0.08}
& 20.41 (20.54) & 0.13 \\
& Barbara
& 23.92 (24.02) & 0.10
& 23.29 (23.56) & 0.27
& 23.21 (23.48) & 0.27
& 24.75 (24.85) & 0.11
& 24.78 (24.85) & 0.07 \\
& Coastguard
& 24.04 (24.47) & 0.43
& 22.93 (23.83) & 0.91
& 23.58 (23.75) & 0.17
& \textbf{24.62 (24.75)} & \textbf{0.13}
& 24.61 (24.75) & 0.14 \\
& Comic
& 19.04 (20.90) & 1.87
& 17.75 (20.35) & 2.59
& 19.84 (20.39) & 0.55
& 21.38 (21.52) & 0.14
& \textbf{21.40 (21.52)} & \textbf{0.12} \\
& Face
& 26.30 (26.49) & 0.19
& 25.50 (26.17) & 0.67
& 25.89 (26.27) & 0.38
& 26.59 (26.79) & 0.20
& \textbf{26.73 (26.79)} & \textbf{0.06} \\
& Flowers
& 21.71 (23.91) & 2.20
& 22.89 (23.55) & 0.65
& 22.99 (23.33) & 0.34
& \textbf{24.47 (24.54)} & \textbf{0.07}
& 24.39 (24.54) & 0.14 \\
& Foreman
& 23.27 (23.62) & 0.34
& 23.15 (23.31) & 0.15
& 23.07 (23.25) & 0.18
& 23.69 (23.84) & 0.15
& \textbf{23.70 (23.84)} & \textbf{0.14} \\
& Zebra
& 22.53 (24.68) & 2.14
& 23.95 (24.11) & 0.17
& \textbf{24.25 (24.31)} & \textbf{0.06}
& 24.90 (25.24) & 0.33
& 25.01 (25.24) & 0.22 \\
\midrule

\multirow{10}{*}{Impulse}
& Leaves
& 17.79 (23.19) & 5.40
& 22.18 (22.34) & 0.16
& 21.59 (22.25) & 0.66
& \textbf{22.12 (22.17)} & \textbf{0.04}
& 21.88 (22.17) & 0.28 \\
& Starfish
& 19.10 (23.92) & 4.82
& 22.68 (23.32) & 0.64
& 22.99 (23.11) & 0.13
& 23.04 (23.38) & 0.34
& \textbf{23.12 (23.38)} & \textbf{0.26} \\
& Baboon
& 20.74 (21.00) & 0.27
& 18.53 (20.50) & 1.97
& 20.03 (20.31) & 0.28
& \textbf{20.34 (20.36)} & \textbf{0.02}
& 20.30 (20.36) & 0.06 \\
& Barbara
& 23.69 (24.17) & 0.49
& 23.32 (23.90) & 0.58
& 23.34 (23.69) & 0.31
& \textbf{23.49 (23.65)} & \textbf{0.16}
& 23.32 (23.65) & 0.33 \\
& Coastguard
& 20.56 (24.65) & 4.09
& 23.05 (24.26) & 1.22
& 23.95 (24.18) & 0.23
& \textbf{24.12 (24.28)} & \textbf{0.15}
& 23.68 (24.28) & 0.60 \\
& Comic
& 19.86 (22.15) & 2.30
& 18.14 (21.45) & 3.31
& 21.00 (21.17) & 0.17
& \textbf{20.98 (21.13)} & \textbf{0.16}
& 20.93 (21.13) & 0.20 \\
& Face
& 16.85 (24.86) & 8.01
& 24.27 (24.81) & 0.54
& 24.18 (24.75) & 0.57
& \textbf{24.59 (24.75)} & \textbf{0.17}
& 23.73 (24.75) & 1.02 \\
& Flowers
& 21.87 (23.78) & 1.91
& 22.63 (23.23) & 0.60
& 22.90 (23.10) & 0.20
& \textbf{22.98 (23.11)} & \textbf{0.13}
& 22.75 (23.11) & 0.36 \\
& Foreman
& 20.17 (26.36) & 6.19
& 25.68 (26.03) & 0.35
& 25.52 (25.90) & 0.38
& \textbf{25.96 (26.07)} & \textbf{0.11}
& 25.41 (26.07) & 0.66 \\
& Zebra
& 22.86 (24.99) & 2.13
& 23.98 (24.34) & 0.36
& 23.90 (24.14) & 0.24
& \textbf{24.00 (24.24)} & \textbf{0.24}
& 23.79 (24.24) & 0.44 \\
\bottomrule
\end{tabular}}
\end{table}

\begin{table}[t]
\centering
\caption{
Representative early-stopping comparison on selected medical images. 
Each PSNR entry is reported as ES PSNR (Oracle PSNR), where the oracle PSNR is computed from the corresponding training trajectory of each method. 
Gap denotes Oracle PSNR minus ES PSNR; a smaller Gap indicates a more accurate stopping point.
}
\label{tab:representative_medical_es_gap}
\small
\setlength{\tabcolsep}{4.5pt}
\renewcommand{\arraystretch}{1.18}
\resizebox{\linewidth}{!}{
\begin{tabular}{llcc cc cc cc}
\toprule
\multirow{3}{*}{Noise}
& \multirow{3}{*}{Image}
& \multicolumn{2}{c}{SURE-ES}
& \multicolumn{2}{c}{WMV-ES}
& \multicolumn{2}{c}{MR-ES}
& \multicolumn{2}{c}{ACR-ES} \\
\cmidrule(lr){3-4}
\cmidrule(lr){5-6}
\cmidrule(lr){7-8}
\cmidrule(lr){9-10}
&
& \makecell{PSNR\\(Oracle)} & Gap $\downarrow$
& \makecell{PSNR\\(Oracle)} & Gap $\downarrow$
& \makecell{PSNR\\(Oracle)} & Gap $\downarrow$
& \makecell{PSNR\\(Oracle)} & Gap $\downarrow$ \\
\midrule

\multirow{10}{*}{Gaussian}
& CT01
& 19.34 (20.92) & 1.59
& 19.46 (20.49) & 1.03
& \textbf{20.35 (20.47)} & \textbf{0.12}
& 20.52 (20.68) & 0.16 \\
% & CT04
% & 15.38 (21.61) & 6.24
% & 21.11 (21.59) & 0.48
% & 21.06 (21.33) & 0.27
% & \textbf{21.29 (21.44)} & \textbf{0.15} \\
& CT05
& 20.42 (20.79) & 0.37
& 19.93 (20.95) & 1.02
& 20.47 (20.67) & 0.20
& \textbf{20.49 (20.63)} & \textbf{0.14} \\
& CT06
& 18.04 (20.50) & 2.46
& 19.95 (20.64) & 0.69
& 19.78 (20.25) & 0.47
& \textbf{20.22 (20.45)} & \textbf{0.23} \\
& CT08
& 15.26 (21.64) & 6.38
& 21.12 (21.41) & 0.28
& 20.12 (21.17) & 1.05
& \textbf{21.11 (21.30)} & \textbf{0.20} \\
& CT09
& 16.39 (20.55) & 4.17
& 20.16 (20.57) & 0.40
& 19.85 (20.21) & 0.36
& \textbf{19.96 (20.20)} & \textbf{0.24} \\
& T2-23
& 19.58 (19.76) & 0.19
& 16.77 (19.57) & 2.80
& 19.23 (19.64) & 0.41
& \textbf{19.37 (19.50)} & \textbf{0.13} \\
& T2-633
& 15.31 (19.93) & 4.62
& 15.87 (19.94) & 4.07
& \textbf{19.72 (19.76)} & \textbf{0.04}
& 19.58 (19.75) & 0.17 \\
& T2-1173
& 15.27 (19.88) & 4.61
& 17.17 (19.87) & 2.70
& 19.68 (19.85) & 0.17
& \textbf{19.63 (19.79)} & \textbf{0.17} \\
& T2-1783
& 15.60 (19.86) & 4.26
& 16.72 (19.72) & 3.00
& 19.23 (19.53) & 0.30
& \textbf{19.40 (19.58)} & \textbf{0.18} \\
& T2-2273
& 18.52 (20.06) & 1.54
& 16.72 (19.95) & 3.23
& 19.71 (19.79) & 0.08
& \textbf{19.86 (19.94)} & \textbf{0.08} \\
\midrule

\multirow{10}{*}{Poisson}
& CT01
& 19.79 (25.66) & 5.87
& 23.84 (26.02) & 2.18
& 24.05 (25.05) & 1.01
& \textbf{26.07 (26.46)} & \textbf{0.40} \\
% & CT04
% & 20.41 (27.33) & 6.91
% & 24.91 (26.71) & 1.80
% & 24.74 (26.87) & 2.12
% & \textbf{27.36 (27.98)} & \textbf{0.62} \\
& CT05
& 22.71 (27.15) & 4.44
& 24.17 (28.03) & 3.86
& 24.63 (27.12) & 2.48
& \textbf{29.26 (29.60)} & \textbf{0.35} \\
& CT06
& 20.88 (26.51) & 5.63
& 22.36 (26.11) & 3.75
& \textbf{25.30 (26.00)} & \textbf{0.70}
& 25.94 (27.18) & 1.24 \\
& CT08
& 22.26 (26.25) & 3.99
& 24.27 (25.94) & 1.67
& 23.99 (25.58) & 1.59
& \textbf{26.28 (26.39)} & \textbf{0.10} \\
& CT09
& 18.82 (26.31) & 7.49
& 22.76 (26.34) & 3.57
& 23.86 (26.11) & 2.25
& \textbf{26.86 (27.63)} & \textbf{0.77} \\
& T2-23
& 24.44 (25.97) & 1.53
& 18.55 (25.96) & 7.41
& 23.35 (25.87) & 2.52
& \textbf{25.60 (26.92)} & \textbf{1.32} \\
& T2-633
& 17.41 (25.60) & 8.18
& 17.95 (26.55) & 8.60
& 23.14 (26.48) & 3.34
& \textbf{27.11 (28.14)} & \textbf{1.04} \\
& T2-1173
& 24.59 (25.99) & 1.39
& 20.30 (27.37) & 7.07
& 23.48 (26.66) & 3.17
& \textbf{27.29 (28.25)} & \textbf{0.96} \\
& T2-1783
& 23.37 (24.80) & 1.43
& 17.91 (25.62) & 7.71
& 23.00 (25.45) & 2.45
& \textbf{25.12 (26.36)} & \textbf{1.24} \\
& T2-2273
& 22.56 (25.98) & 3.42
& 17.94 (25.91) & 7.97
& 22.86 (25.74) & 2.88
& \textbf{25.94 (26.95)} & \textbf{1.01} \\
\midrule

\multirow{10}{*}{Impulse}
& CT01
& 14.95 (22.98) & 8.03
& 21.23 (22.90) & 1.66
& 22.38 (22.74) & 0.36
& \textbf{22.14 (22.38)} & \textbf{0.24} \\
% & CT04
% & 15.25 (24.64) & 9.39
% & 23.18 (24.64) & 1.46
% & \textbf{23.98 (24.38)} & \textbf{0.40}
% & 23.86 (24.38) & 0.52 \\
& CT05
& 14.50 (22.92) & 8.42
& 23.04 (23.95) & 0.92
& \textbf{22.81 (23.06)} & \textbf{0.26}
& 22.45 (23.04) & 0.59 \\
& CT06
& 14.39 (23.26) & 8.87
& 21.15 (23.05) & 1.90
& 22.09 (22.83) & 0.74
& \textbf{22.02 (22.35)} & \textbf{0.32} \\
& CT08
& 15.71 (24.26) & 8.55
& 23.32 (23.85) & 0.53
& 23.84 (24.10) & 0.25
& \textbf{23.43 (23.67)} & \textbf{0.24} \\
& CT09
& 14.63 (23.24) & 8.62
& 21.00 (22.85) & 1.86
& 22.56 (22.86) & 0.30
& \textbf{22.49 (22.75)} & \textbf{0.26} \\
& T2-23
& 14.64 (21.99) & 7.35
& 18.15 (22.01) & 3.86
& 21.40 (21.88) & 0.48
& \textbf{21.55 (21.67)} & \textbf{0.12} \\
& T2-633
& 14.51 (22.73) & 8.22
& 21.90 (22.53) & 0.63
& \textbf{22.19 (22.31)} & \textbf{0.12}
& 22.14 (22.26) & 0.13 \\
& T2-3
& 14.18 (22.60) & 8.42
& 19.81 (22.97) & 3.16
& \textbf{22.10 (22.34)} & \textbf{0.24}
& 21.30 (21.96) & 0.67 \\
& T2-1783
& 14.15 (21.98) & 7.83
& 18.02 (22.02) & 4.00
& 21.51 (21.80) & 0.29
& \textbf{21.21 (21.45)} & \textbf{0.24} \\
& T2-2273
& 14.20 (22.30) & 8.10
& 19.18 (22.21) & 3.03
& \textbf{21.86 (22.07)} & \textbf{0.21}
& 20.48 (21.92) & 1.43 \\
\bottomrule
\end{tabular}}
\end{table}

\subsection{Full Method-specific Oracle Comparisons}
\label{app:oracle_full}

Section~\ref{subsec:oracle_analysis} summarizes the method-specific oracle analysis using group-level averages.
Here, we provide the full per-image comparisons on representative natural and medical images.
The purpose of this analysis is to verify that the proposed criteria reduce the PSNR Gap by selecting stopping points closer to the oracle region, rather than by lowering the attainable oracle PSNR of the corresponding trajectory.

Tables~\ref{tab:representative_es_gap} and~\ref{tab:representative_medical_es_gap} report per-image results on natural and medical images, respectively.
Each entry is written as ES PSNR (Oracle PSNR), where the oracle PSNR is computed from the corresponding training trajectory of each method.
Thus, the selected reconstruction quality and the best attainable quality along the same method-specific trajectory can be compared simultaneously.

Table~\ref{tab:representative_es_gap} reports results on selected natural images under Gaussian, Poisson, and impulse noise.
Across different images and noise types, CSR-ES and ACR-ES generally achieve smaller oracle gaps than SURE-ES and WMV-ES while maintaining comparable oracle PSNR values.
This indicates that the improvement is not caused by degrading the reconstruction trajectory.
Instead, the proposed color-image criteria provide stopping signals that more accurately identify the near-oracle region of the corresponding trajectory.

Table~\ref{tab:representative_medical_es_gap} reports the corresponding results on selected CT and T2 images.
Since CSR-ES relies on RGB channel consistency, it is not applicable to these grayscale medical images.
MR-ES and ACR-ES instead provide reference-based stopping signals through mask-based and auxiliary-corruption-based references.
For Gaussian and impulse noise, MR-ES and ACR-ES generally select iterations close to their own oracle trajectories.
For Poisson noise, where the corruption is signal-dependent and the mask-reference signal can be less stable, ACR-ES often achieves substantially smaller gaps than SURE-ES, WMV-ES, and MR-ES.
These per-image results are consistent with the averaged trends in the main paper and further support the conclusion that the proposed criteria improve stopping accuracy by better identifying the best attainable region along each trajectory.

\subsection{Ablation Study: MSE-based and Bias-based Reference Scores}
\label{app:score_ablation}

We next study the effect of different reference scores. 
The proposed framework uses different reference constructions for natural color images and grayscale medical images. 
For natural color images, ACR-ES uses a bias-based cross-reference score. 
For grayscale medical images, MR-ES uses an MSE-based mask-reference score. 
This section justifies these choices.

For natural color images, we compare bias-based and MSE-based variants of ACR-ES. 
Table~\ref{tab:acr_bias_vs_mse_natural} shows that the bias-based ACR-ES consistently reduces the PSNR Gap compared with the MSE-based variant, especially under Gaussian and impulse noise. 
This suggests that, for RGB images, cross-reference bias information captures useful stopping behavior beyond direct MSE consistency.

\begin{table}[t]
\centering
\caption{Comparison between bias-based and MSE-based ACR-ES on natural image datasets. 
PSNR and PSNR Gap are reported as mean $\pm$ standard deviation across images. 
Positive $\Delta$ Gap indicates that ACR-ES (Bias) is closer to the oracle stopping point.}
\label{tab:acr_bias_vs_mse_natural}
\small
\setlength{\tabcolsep}{5.2pt}
\renewcommand{\arraystretch}{1.15}
\resizebox{\linewidth}{!}{
\begin{tabular}{llcccccc}
\toprule
\multirow{2}{*}{Noise}
& \multirow{2}{*}{Dataset}
& \multicolumn{2}{c}{ACR-ES (Bias)}
& \multicolumn{2}{c}{ACR-ES (MSE)}
& \multirow{2}{*}{$\Delta$ Gap $\uparrow$} \\
\cmidrule(lr){3-4}
\cmidrule(lr){5-6}
&
& PSNR $\uparrow$ & Gap $\downarrow$
& PSNR $\uparrow$ & Gap $\downarrow$
& \\
\midrule

\multirow{4}{*}{Gaussian}
& Set14
& \textbf{22.26($\pm$1.71)} & \textbf{0.25($\pm$0.13)}
& 21.44($\pm$1.67) & 1.06($\pm$0.79)
& \cellcolor{green!8}0.81 \\
& Set18
& \textbf{22.33($\pm$1.52)} & \textbf{0.21($\pm$0.12)}
& 21.65($\pm$0.84) & 0.89($\pm$1.24)
& \cellcolor{green!8}0.68 \\
& CBSD68
& \textbf{22.33($\pm$1.99)} & \textbf{0.30($\pm$0.28)}
& 21.15($\pm$1.58) & 1.48($\pm$1.44)
& \cellcolor{green!8}1.18 \\
& Urban100
& \textbf{21.36($\pm$1.31)} & \textbf{0.20($\pm$0.10)}
& 20.60($\pm$1.01) & 0.97($\pm$0.65)
& \cellcolor{green!8}0.76 \\
\addlinespace[2pt]
\midrule

\multirow{4}{*}{Poisson}
& Set14
& \textbf{24.52($\pm$2.35)} & \textbf{0.11($\pm$0.05)}
& 24.52($\pm$2.35) & 0.11($\pm$0.07)
& \cellcolor{green!8}0.00 \\
& Set18
& \textbf{26.55($\pm$2.33)} & \textbf{0.15($\pm$0.10)}
& 26.53($\pm$2.32) & 0.17($\pm$0.11)
& \cellcolor{green!8}0.02 \\
& CBSD68
& \textbf{24.66($\pm$2.33)} & \textbf{0.14($\pm$0.09)}
& 24.66($\pm$2.32) & 0.14($\pm$0.08)
& \cellcolor{green!8}0.00 \\
& Urban100
& 24.08($\pm$1.66) & 0.18($\pm$0.12)
& \textbf{24.09($\pm$1.64)} & \textbf{0.17($\pm$0.10)}
& \cellcolor{red!6}-0.01 \\
\addlinespace[2pt]
\midrule

\multirow{4}{*}{Impulse}
& Set14
& \textbf{23.65($\pm$1.74)} & \textbf{0.49($\pm$0.23)}
& 22.50($\pm$1.77) & 1.64($\pm$1.56)
& \cellcolor{green!8}1.15 \\
& Set18
& \textbf{24.09($\pm$1.52)} & \textbf{0.53($\pm$0.21)}
& 23.11($\pm$0.97) & 1.50($\pm$0.76)
& \cellcolor{green!8}0.97 \\
& CBSD68
& \textbf{23.63($\pm$2.02)} & \textbf{0.44($\pm$0.25)}
& 22.29($\pm$1.36) & 1.78($\pm$1.08)
& \cellcolor{green!8}1.34 \\
& Urban100
& \textbf{22.86($\pm$1.42)} & \textbf{0.34($\pm$0.17)}
& 22.23($\pm$1.08) & 0.97($\pm$0.66)
& \cellcolor{green!8}0.63 \\

\bottomrule
\end{tabular}}
\end{table}

\begin{table}[t]
\centering
\caption{Comparison between the MSE-based and bias-based mask-reference early stopping criteria on medical datasets. 
PSNR and PSNR Gap are reported as mean $\pm$ standard deviation across images. 
Positive $\Delta$ Gap=Gap(Bias)-Gap(MSE) indicates that the MSE-based criterion is closer to the oracle stopping point.}
\label{tab:mse_vs_bias_medical}
\small
\setlength{\tabcolsep}{5.5pt}
\renewcommand{\arraystretch}{1.15}
% \resizebox{\linewidth}{!}
{
\begin{tabular}{llcccccc}
\toprule
\multirow{2}{*}{Noise}
& \multirow{2}{*}{Dataset}
& \multicolumn{2}{c}{MR-ES (MSE)}
& \multicolumn{2}{c}{MR-ES (Bias)}
& \multirow{2}{*}{$\Delta$ Gap $\uparrow$} \\
\cmidrule(lr){3-4}
\cmidrule(lr){5-6}
&
& PSNR $\uparrow$ & Gap $\downarrow$
& PSNR $\uparrow$ & Gap $\downarrow$
& \\
\midrule

\multirow{4}{*}{Gaussian}
& MRI15
& \textbf{21.60($\pm$0.69)} & \textbf{0.41($\pm$0.29)}
& 19.69($\pm$1.69) & 2.33($\pm$1.74)
& \cellcolor{green!8}1.92 \\
& CT10
& \textbf{20.22($\pm$0.47)} & \textbf{0.35($\pm$0.28)}
& 19.66($\pm$1.36) & 0.92($\pm$1.26)
& \cellcolor{green!8}0.57 \\
& PD20
& \textbf{19.94($\pm$0.36)} & \textbf{0.33($\pm$0.24)}
& 18.03($\pm$0.92) & 2.24($\pm$0.93)
& \cellcolor{green!8}1.91 \\
& T2-20
& \textbf{19.62($\pm$0.19)} & \textbf{0.16($\pm$0.09)}
& 18.33($\pm$1.01) & 1.45($\pm$0.98)
& \cellcolor{green!8}1.29 \\
\addlinespace[2pt]
\midrule

\multirow{4}{*}{Poisson}
& MRI15
& \textbf{24.78($\pm$1.27)} & \textbf{1.35($\pm$1.04)}
& 23.33($\pm$1.83) & 2.80($\pm$1.65)
& \cellcolor{green!8}1.45 \\
& CT10
& \textbf{24.56($\pm$1.30)} & \textbf{2.03($\pm$0.85)}
& 23.36($\pm$1.71) & 3.24($\pm$1.60)
& \cellcolor{green!8}1.21 \\
& PD20
& \textbf{21.15($\pm$0.37)} & 4.49($\pm$0.37)
& 21.29($\pm$0.47) & \textbf{4.34($\pm$0.36)}
& \cellcolor{red!6}-0.15 \\
& T2-20
& 23.26($\pm$0.47) & 3.05($\pm$0.36)
& \textbf{23.53($\pm$0.78)} & \textbf{2.78($\pm$0.72)}
& \cellcolor{red!6}-0.27 \\
\addlinespace[2pt]
\midrule

\multirow{4}{*}{Impulse}
& MRI15
& \textbf{23.49($\pm$0.52)} & \textbf{0.46($\pm$0.47)}
& 18.98($\pm$2.91) & 4.96($\pm$2.96)
& \cellcolor{green!8}4.50 \\
& CT10
& \textbf{22.82($\pm$0.70)} & \textbf{0.31($\pm$0.17)}
& 16.77($\pm$1.81) & 6.36($\pm$1.50)
& \cellcolor{green!8}6.05 \\
& PD20
& \textbf{23.12($\pm$0.28)} & \textbf{0.26($\pm$0.19)}
& 16.51($\pm$1.28) & 6.87($\pm$1.34)
& \cellcolor{green!8}6.61 \\
& T2-20
& \textbf{21.89($\pm$0.26)} & \textbf{0.31($\pm$0.16)}
& 16.72($\pm$1.95) & 5.48($\pm$1.93)
& \cellcolor{green!8}5.17 \\
\bottomrule
\end{tabular}}
\end{table}

For grayscale medical images, we compare MSE-based and bias-based variants of MR-ES. 
Table~\ref{tab:mse_vs_bias_medical} shows that the MSE-based mask-reference criterion is generally more reliable than the bias-based variant. 
This is expected because single-channel medical images do not provide the same multi-channel structure as RGB images, making direct mask-reference consistency a more stable validation signal.

Together, these ablations justify the design used in the main paper: bias-based ACR-ES for natural color images and MSE-based MR-ES for grayscale medical images.

\begin{table}[t]
\centering
\caption{Sensitivity of ACR-ES on Set3 to the additional noise level $c$. 
Results are reported as mean $\pm$ standard deviation across images. 
PSNR Gap denotes the difference between Oracle PSNR and ACR-ES PSNR.
For Gaussian noise, the level of noise in the observation $y$ is $\sigma=0.26$ and the level of noise added to the augmented channels is $\tau=1.25\sigma+0.1c$.
For Poisson noise, the noise  in $y$ follows $\mathcal{P}(\lambda)$  with $\lambda=10$ and the noise  added to augmented channel follows $\mathcal{P}(\tau)$ with $\tau=1.25\lambda+c$.
For impulse noise, the probability of measurement corruption is $p=0.1$ in the observation $y$ and $\tau=1.25p+0.1c$ in the augmented channel.
}
\label{tab:set3c_acr_s_ablation}
\small
\setlength{\tabcolsep}{7pt}
\renewcommand{\arraystretch}{1.15}
\begin{tabular}{llccc}
\toprule
Noise 
& $c$ 
& \makecell{ACR-ES\\PSNR $\uparrow$}
& \cellcolor{gray!10}\makecell{Oracle\\PSNR $\uparrow$}
& \makecell{Gap $\downarrow$} \\
\midrule

\multirow{4}{*}{Gaussian}
& $0$    & 21.44($\pm$0.78) & \cellcolor{gray!10}21.55($\pm$0.80) & \textbf{0.11($\pm$0.05)} \\
& $0.10$ & 21.42($\pm$0.84) & \cellcolor{gray!10}21.56($\pm$0.77) & 0.14($\pm$0.08) \\
& $0.15$ & 21.34($\pm$0.85) & \cellcolor{gray!10}21.54($\pm$0.80) & 0.20($\pm$0.06) \\
& $0.20$ & \textbf{21.44($\pm$0.87)} & \cellcolor{gray!10}21.55($\pm$0.86) & \textbf{0.11($\pm$0.08)} \\
\midrule

\multirow{4}{*}{Poisson}
& $0$    & 22.91($\pm$1.82) & \cellcolor{gray!10}23.05($\pm$1.84) & 0.14($\pm$0.06) \\
& $0.10$ & \textbf{23.00($\pm$1.80)} & \cellcolor{gray!10}23.06($\pm$1.74) & \textbf{0.06($\pm$0.06)} \\
& $0.15$ & 22.81($\pm$1.73) & \cellcolor{gray!10}22.97($\pm$1.74) & 0.17($\pm$0.09) \\
& $0.20$ & 22.89($\pm$1.66) & \cellcolor{gray!10}23.07($\pm$1.69) & 0.18($\pm$0.18) \\
\midrule

\multirow{4}{*}{Impulse}
& $0$    & 22.72($\pm$0.73) & \cellcolor{gray!10}23.05($\pm$0.77) & 0.33($\pm$0.10) \\
& $0.10$ & \textbf{22.84($\pm$0.62)} & \cellcolor{gray!10}22.99($\pm$0.63) & \textbf{0.14($\pm$0.08)} \\
& $0.15$ & 22.81($\pm$0.65) & \cellcolor{gray!10}22.97($\pm$0.60) & 0.16($\pm$0.06) \\
& $0.20$ & 22.73($\pm$0.62) & \cellcolor{gray!10}22.95($\pm$0.63) & 0.22($\pm$0.01) \\

\bottomrule
\end{tabular}
\end{table}
\begin{table}[t]
\centering
\caption{
Sensitivity of ACR-ES on Set3 to auxiliary-noise mismatch under Gaussian noise.
Each entry is reported as ES PSNR (Oracle PSNR). 
The default setting used throughout the paper is $\tau_1=\tau_2=1.25\sigma$, highlighted in purple.
Off-diagonal entries correspond to mismatched auxiliary noise levels $\tau_1\neq \tau_2$. 
{
ACR-ES uses the bias score $\mathrm{bias}(y_2,\hat{x}_1)$.
Gray cells denote matched but non-default settings $\tau_1=\tau_2$.
}
}
\label{tab:n1_n2_mismatch_acr_psnr}
\small
\setlength{\tabcolsep}{8pt}
\renewcommand{\arraystretch}{1.20}
\resizebox{0.72\linewidth}{!}{
\begin{tabular}{c|ccc}
\toprule
\diagbox[width=5.8em,height=2.4em]{$\tau_2$}{$\tau_1$}
& $1.0\sigma$
& $1.25\sigma$
& $1.5\sigma$ \\
\midrule

$1.0\sigma$
& \cellcolor{gray!10}20.73 (21.83)
& 21.46 (21.60)
& \textbf{21.65 (21.65)} \\

$1.25\sigma$
& 21.19 (21.69)
& \cellcolor{blue!8}\textbf{21.52 (21.69)}
& \textbf{21.52 (21.69)} \\

$1.5\sigma$
& 20.73 (21.68)
& \textbf{21.34 (21.68)}
& \cellcolor{gray!10}21.32 (21.65) \\

\bottomrule
\end{tabular}}

\end{table}

\subsection{Sensitivity to Auxiliary Noise Levels}
\label{app:aux_sensitivity}

ACR-ES constructs auxiliary corrupted observations and therefore introduces auxiliary noise levels. 
A natural question is whether the method requires precise tuning of these levels. 
In the main experiments, we use the default noise level setting
\[
\tau_1=\tau_2=1.25\sigma,
\]
where $\sigma$ denotes the original noise level. 
Here, we study the sensitivity of ACR-ES to both auxiliary-noise magnitude and mismatch.

Table~\ref{tab:set3c_acr_s_ablation} studies the effect of varying the auxiliary noise magnitude by setting $\tau_1=\tau_2=1.25\sigma+c$ under Gaussian, Poisson, and impulse noise. 
The PSNR and PSNR Gap remain stable across different values of $c$, showing that ACR-ES is not overly sensitive to the precise auxiliary-noise strength.

We further evaluate mismatched auxiliary noise levels in Table~\ref{tab:n1_n2_mismatch_acr_psnr}, where $\tau_1$ and $\tau_2$ are varied independently. 
The diagonal entries correspond to matched auxiliary noise levels, and the highlighted entry corresponds to the default setting $\tau_1=\tau_2=1.25\sigma$. 
The off-diagonal entries show that ACR-ES remains effective even when $\tau_1\neq \tau_2$.

Moreover, we test whether ACR-ES requires the auxiliary corrupted references to have the same noise type as the observed image.
In this study, the observed image $y$ is fixed as the same Gaussian-corrupted input, while the auxiliary references $y_1$ and $y_2$ are generated using Gaussian, Poisson, or impulse noise.
Since the auxiliary references participate in the ACR-ES construction, changing their noise type may induce a different optimization trajectory; therefore, the oracle is computed separately for each row.
As shown in Table~\ref{tab:acr_noise_type_set3}, Poisson auxiliary references remain effective and achieve a gap comparable to the matched Gaussian setting on average.
Impulse auxiliary references are less stable, mainly due to failures on Starfish and Butterfly, but still work well on Leaves.
These results suggest that ACR-ES is reasonably robust to moderate auxiliary noise-type mismatch, although impulse-type auxiliary corruptions can be less reliable.

Table~\ref{tab:acr_fixed_gaussian_aux_set3} studies a reversed noise-type mismatch setting, where the auxiliary references $y_1$ and $y_2$ are fixed to Gaussian-corrupted observations, while the observed image $y$ is corrupted by Gaussian, Poisson, or impulse noise.
The Gaussian case is the matched setting, and the Poisson and impulse cases test whether ACR-ES remains effective when the observation and auxiliary references have different noise types.
ACR-ES remains reasonably robust under this mismatch: the average gap increases from $0.33$ dB in the matched Gaussian case to $0.49$ dB for Poisson and $0.61$ dB for impulse.
This indicates that ACR-ES does not require an exact match between the noise type of $y$ and that of the auxiliary references, although matched Gaussian references yield the most stable stopping behavior.

These results indicate that ACR-ES is robust to moderate auxiliary-noise scaling and mismatch, and does not require precise tuning of the auxiliary corruption levels.

\begin{table}[t]
\centering
\caption{
Auxiliary noise-type mismatch study of ACR-ES on Set3.
The observed image $y$ is fixed as the same Gaussian-corrupted input with noise level $\sigma=0.26$.
Only the auxiliary corrupted references $y_1$ and $y_2$ vary across rows: the matched Gaussian setting uses noise level $1.25\times0.26$, the Poisson setting uses noise level $1.25\times10$, and the impulse setting uses corruption probability $1.25\times0.1$.
Since different auxiliary references induce different ACR-ES trajectories, the oracle iteration and oracle PSNR are computed separately for each row.
PSNR Gap denotes the difference between the row-specific oracle PSNR and the selected PSNR.
}
\label{tab:acr_noise_type_set3}
\small
\setlength{\tabcolsep}{4.8pt}
\renewcommand{\arraystretch}{1.15}
\resizebox{\linewidth}{!}{
\begin{tabular}{llcccccc}
\toprule
\multirow{2}{*}{Noise type}
& \multirow{2}{*}{Image}
& \multicolumn{3}{c}{ACR-ES}
& \multicolumn{2}{c}{\cellcolor{gray!10}Oracle}
& \multirow{2}{*}{$\Delta$ Gap} \\
\cmidrule(lr){3-5}
\cmidrule(lr){6-7}
&
& Iter.
& PSNR $\uparrow$
& Gap $\downarrow$
& \cellcolor{gray!10}Iter.
& \cellcolor{gray!10}PSNR $\uparrow$
& \\
\midrule

\multirow{4}{*}{Gaussian}
& Starfish
& 1026
& 21.50
& 0.18
& \cellcolor{gray!10}1437
& \cellcolor{gray!10}21.68
& -- \\
& Butterfly
& 872
& 22.00
& 0.25
& \cellcolor{gray!10}1081
& \cellcolor{gray!10}22.25
& -- \\
& Leaves
& 1101
& 20.39
& 0.31
& \cellcolor{gray!10}1580
& \cellcolor{gray!10}20.70
& -- \\
& \textit{Average}
& --
& 21.29($\pm$0.82)
& 0.25($\pm$0.07)
& --
& \cellcolor{gray!10}21.54($\pm$0.78)
& -- \\
\midrule

\multirow{4}{*}{Poisson}
& Starfish
& 2193
& 22.25
& 0.32
& \cellcolor{gray!10}1674
& \cellcolor{gray!10}22.56
& 0.13 \\
& Butterfly
& 1712
& 23.60
& 0.12
& \cellcolor{gray!10}1451
& \cellcolor{gray!10}23.72
& -0.13 \\
& Leaves
& 2300
& 21.63
& 0.22
& \cellcolor{gray!10}1949
& \cellcolor{gray!10}21.85
& -0.09 \\
& \textit{Average}
& --
& 22.49($\pm$1.01)
& 0.22($\pm$0.10)
& --
& \cellcolor{gray!10}22.71($\pm$0.94)
& -0.03 \\
\midrule

\multirow{4}{*}{Impulse}
& Starfish
& 2720
& 20.33
& 1.40
& \cellcolor{gray!10}1328
& \cellcolor{gray!10}21.73
& 1.21 \\
& Butterfly
& 2201
& 21.22
& 1.18
& \cellcolor{gray!10}1170
& \cellcolor{gray!10}22.40
& 0.93 \\
& Leaves
& 1863
& 20.53
& 0.22
& \cellcolor{gray!10}1422
& \cellcolor{gray!10}20.75
& -0.09 \\
& \textit{Average}
& --
& 20.69($\pm$0.47)
& 0.93($\pm$0.63)
& --
& \cellcolor{gray!10}21.63($\pm$0.83)
& 0.69 \\

\bottomrule
\end{tabular}}
\end{table}

\begin{table}[t]
\centering
\caption{
Auxiliary noise-type mismatch study of ACR-ES on Set3 with fixed Gaussian auxiliary references.
The observed image $y$ is corrupted by different noise types: Gaussian noise with level $\sigma=0.26$, Poisson noise with level $10$, and impulse noise with corruption probability $0.1$.
The auxiliary corrupted references $y_1$ and $y_2$ are always generated using Gaussian noise with level $1.25\times0.26$.
Since different observed noise types induce different DIP trajectories, the oracle iteration and oracle PSNR are computed separately for each row.
PSNR Gap denotes the difference between the row-specific oracle PSNR and the selected PSNR.
}
\label{tab:acr_fixed_gaussian_aux_set3}
\small
\setlength{\tabcolsep}{4.8pt}
\renewcommand{\arraystretch}{1.15}
\resizebox{\linewidth}{!}{
\begin{tabular}{llcccccc}
\toprule
\multirow{2}{*}{Noise of $y$}
& \multirow{2}{*}{Image}
& \multicolumn{3}{c}{ACR-ES}
& \multicolumn{2}{c}{\cellcolor{gray!10}Oracle}
& \multirow{2}{*}{$\Delta$ Gap} \\
\cmidrule(lr){3-5}
\cmidrule(lr){6-7}
&
& Iter.
& PSNR $\uparrow$
& Gap $\downarrow$
& \cellcolor{gray!10}Iter.
& \cellcolor{gray!10}PSNR $\uparrow$
& \\
\midrule

\multirow{4}{*}{\makecell{Gaussian\\(matched)}}
& Starfish
& 1034
& 21.38
& 0.16
& \cellcolor{gray!10}1291
& \cellcolor{gray!10}21.54
& -- \\
& Butterfly
& 1013
& 22.05
& 0.09
& \cellcolor{gray!10}1288
& \cellcolor{gray!10}22.15
& -- \\
& Leaves
& 843
& 19.90
& 0.73
& \cellcolor{gray!10}1579
& \cellcolor{gray!10}20.63
& -- \\
& \textit{Average}
& --
& 21.11($\pm$1.10)
& 0.33($\pm$0.35)
& --
& \cellcolor{gray!10}21.44($\pm$0.76)
& -- \\
\midrule

\multirow{4}{*}{\makecell{Poisson\\(mismatched)}}
& Starfish
& 962
& 21.86
& 0.53
& \cellcolor{gray!10}1471
& \cellcolor{gray!10}22.39
& 0.37 \\
& Butterfly
& 894
& 22.53
& 0.55
& \cellcolor{gray!10}1412
& \cellcolor{gray!10}23.08
& 0.45 \\
& Leaves
& 1056
& 19.72
& 0.40
& \cellcolor{gray!10}1856
& \cellcolor{gray!10}20.12
& -0.33 \\
& \textit{Average}
& --
& 21.37($\pm$1.47)
& 0.49($\pm$0.08)
& --
& \cellcolor{gray!10}21.87($\pm$1.55)
& 0.16 \\
\midrule

\multirow{4}{*}{\makecell{Impulse\\(mismatched)}}
& Starfish
& 965
& 22.45
& 0.64
& \cellcolor{gray!10}1649
& \cellcolor{gray!10}23.09
& 0.47 \\
& Butterfly
& 886
& 22.76
& 0.63
& \cellcolor{gray!10}1646
& \cellcolor{gray!10}23.39
& 0.54 \\
& Leaves
& 1073
& 21.44
& 0.56
& \cellcolor{gray!10}1855
& \cellcolor{gray!10}22.00
& -0.17 \\
& \textit{Average}
& --
& 22.22($\pm$0.69)
& 0.61($\pm$0.04)
& --
& \cellcolor{gray!10}22.82($\pm$0.73)
& 0.28 \\

\bottomrule
\end{tabular}}
\end{table}

\section{Generalization to Forward Operators}
\label{app:operator_validation}

The theory in Appendix~\ref{app:reference_validation} is stated for
the denoising case in which the reference and the reconstruction live in the
same space. It is possible to extend the theory to a general inverse problem
$y = A(x) + \xi$, with the validation curve evaluated in the measurement
domain. This section records the generalization, identifies what it does and
does not control, and shows that the mask-reference theorem
(Theorem~\ref{thm:mask_validation}) and the denoising theorem
(Theorem~\ref{thm:single_reference_validation}) are both special cases.

\subsection{Setup}
\label{app:operator_setup}

Let $D$ denote the data used to fit the trajectory
$\{\hat x_t(D)\}_{t=1}^T$. The data $D$ may be a noisy measurement
$y_1 = A(x)+\eta_1$, but more generally it can be any training input from which a DIP trajectory is produced. If the reconstruction procedure uses
additional randomness (e.g.\ random network initialization, dropout,stochastic optimization), we fold this randomness into $D$ or condition on it, so that $\hat x_t$ is a measurable function of $D$. Let
$A:\mathbb{R}^n\to\mathbb{R}^m$ be a linear deterministic validation operator, and suppose an independent reference measurement \begin{equation}\nonumber
    y_2^A \;=\; A(x)+\eta_2
\end{equation}
is available, where $\eta_2 \in \mathbb{R}^m$ has independent, zero-mean,
sub-Gaussian entries with variance proxy $\sigma^2$, satisfies
$\mathbb{E}\eta_{2,i}^2 = \sigma^2$, and is \emph{independent of $D$ and
hence of the fitted trajectory $\{\hat x_t(D)\}$}.
Define the measurement-domain reconstruction error
\begin{equation}\nonumber
    R_A(t) \;:=\; \frac{1}{m}\,\bigl\|A(\hat x_t(D)) - A(x)\bigr\|^2,
\end{equation}
and the operator reference loss
\begin{equation}\nonumber
    V_t^A \;:=\; \frac{1}{m}\,\bigl\|A(\hat x_t(D)) - y_2^A\bigr\|^2.
\end{equation}
The image-domain risk $R(t) = \tfrac{1}{n}\|\hat x_t(D) - x\|^2$ is, in
general, \emph{not} the quantity that the reference loss tracks. We return
to this point in Section~\ref{app:operator_transfer}.

\subsection{Operator-form validation theorem}

\begin{theorem}[Operator reference validation]
\label{thm:operator_validation}
Under the setup above, there exists a universal constant $C>0$ such that,
with probability at least $1-\delta$, uniformly over $1\le t\le T$,
\begin{equation}\nonumber
\bigl|\,V_t^A - R_A(t) - \sigma^2\,\bigr|
\;\le\;
C\!\left[
M_T^A\sqrt{\frac{\log(T/\delta)}{m}}
\;+\;
\sigma^2\,\frac{\log(T/\delta)}{m}
\right],
\end{equation}
where
\begin{equation}\nonumber
M_T^A
\;:=\;
\sigma\sup_{1\le t\le T}\sqrt{R_A(t)} \;+\; \sigma^2.
\end{equation}
\end{theorem}

\begin{proof}
Condition on $D$ and hence on the entire DIP trajectory
$\{\hat x_t(D)\}_{t=1}^T$. For each fixed $t$, write
\begin{equation}\nonumber
a_t \;:=\; A(\hat x_t(D)) - A(x),
\end{equation}
which is deterministic given $D$. Since $y_2^A = A(x) + \eta_2$,
\begin{equation}\nonumber
V_t^A
\;=\;
\frac{1}{m}\|a_t - \eta_2\|^2
\;=\;
\frac{1}{m}\|a_t\|^2
- \frac{2}{m}\langle a_t, \eta_2\rangle
+ \frac{1}{m}\|\eta_2\|^2,
\end{equation}
so that
\begin{equation}\nonumber
V_t^A - R_A(t) - \sigma^2
\;=\;
- \frac{2}{m}\langle a_t, \eta_2\rangle
\;+\;
\frac{1}{m}\bigl(\|\eta_2\|^2 - m\sigma^2\bigr).
\end{equation}

By hypothesis $\eta_2$ is independent of $D$, hence of $a_t$. The first
term is therefore a centered sub-Gaussian linear functional in $\eta_2$
with fixed weights $a_t$. Thus with probability at least $1-\delta/(2T)$,
\begin{equation}\nonumber
\left|\frac{2}{m}\langle a_t, \eta_2\rangle\right|
\;\le\;
C\sigma\sqrt{R_A(t)}\sqrt{\frac{\log(T/\delta)}{m}}.
\end{equation}
The second term is a centered average of sub-exponential random variables
$\eta_{2,i}^2 - \sigma^2$ and is $t$-independent. Bernstein concentration
gives, with probability at least $1-\delta/2$,
\begin{equation}\nonumber
\left|\frac{1}{m}\bigl(\|\eta_2\|^2 - m\sigma^2\bigr)\right|
\;\le\;
C\sigma^2\!\left[\sqrt{\frac{\log(T/\delta)}{m}} + \frac{\log(T/\delta)}{m}\right].
\end{equation}
Union-bounding the first inequality over $t=1,\dots,T$ contributes total
failure probability at most $\delta/2$; combining with the second event,
both hold simultaneously with probability at least $1-\delta$. On this event,
\[
\bigl|V_t^A - R_A(t) - \sigma^2\bigr|
\le
C\sigma\!\left(\!\sqrt{R_A(t)} + \sigma\!\right)\!\sqrt{\frac{\log(T/\delta)}{m}}
+
C\sigma^2\frac{\log(T/\delta)}{m},
\]
uniformly in $t$. Since
$\sigma(\sqrt{R_A(t)}+\sigma) \le \sigma\sup_t\sqrt{R_A(t)} + \sigma^2 = M_T^A$,
the stated bound follows.
\end{proof}

\subsection{Oracle inequality}
\label{app:operator_oracle}

\begin{corollary}[Near-oracle stopping in the measurement domain]
\label{cor:operator_oracle}
Under the event of Theorem~\ref{thm:operator_validation}, define
\begin{equation}\nonumber
\varepsilon_m^A
\;:=\;
C\!\left[M_T^A\sqrt{\frac{\log(T/\delta)}{m}}
+ \sigma^2\frac{\log(T/\delta)}{m}\right].
\end{equation}
Let $\hat t \in \arg\min_t V_t^A$ and $t_A^\star \in \arg\min_t R_A(t)$. Then
\begin{equation}\nonumber
R_A(\hat t) \;\le\; R_A(t_A^\star) + 2\varepsilon_m^A.
\end{equation}
\end{corollary}

\begin{proof}
On the event of Theorem~\ref{thm:operator_validation},
$R_A(t)\le V_t^A - \sigma^2 + \varepsilon_m^A$ and
$V_t^A - \sigma^2 \le R_A(t)+\varepsilon_m^A$
hold simultaneously for all $t$. Hence
\[
R_A(\hat t)\le V_{\hat t}^A - \sigma^2 + \varepsilon_m^A
\le V_{t_A^\star}^A - \sigma^2 + \varepsilon_m^A
\le R_A(t_A^\star) + 2\varepsilon_m^A.\qedhere
\]
\end{proof}

\subsection{Transfer from measurement-domain risk to image-domain risk}
\label{app:operator_transfer}

Corollary~\ref{cor:operator_oracle} controls $R_A$, not the image-domain risk
$R(t)=\tfrac{1}{n}\|\hat x_t - x\|^2$. The two coincide when $A$ is the
identity and the dimensions match, but in general the reference loss cannot
detect overfitting in directions to which $A$ is insensitive. A transfer
condition is therefore needed in the same spirit as
Corollary~\ref{cor:mask_full_risk}. We give a two-sided version that yields
the cleanest comparison to the image-domain oracle, followed by a one-sided
relaxation that is more honest about what is achievable for poorly
conditioned $A$.

\begin{corollary}[Image-domain near-oracle stopping; two-sided transfer]
\label{cor:operator_transfer}
Suppose that, along the trajectory,
\begin{equation}\nonumber
\sup_{1\le t\le T}\bigl|R(t)-c\,R_A(t)\bigr|\le\kappa
\end{equation}
for some operator-dependent constants $c>0$ and $\kappa\ge 0$. Then the
operator reference stopping time $\hat t\in\arg\min_t V_t^A$ satisfies
\begin{equation}\nonumber
R(\hat t)
\;\le\;
\min_{1\le t\le T}R(t)
\;+\;
2\kappa
\;+\;
2c\,\varepsilon_m^A.
\end{equation}
\end{corollary}

\begin{proof}
Let $t^\star\in\arg\min_t R(t)$. By the comparison assumption,
$R(\hat t)\le c\,R_A(\hat t)+\kappa$.
Corollary~\ref{cor:operator_oracle} gives
$R_A(\hat t)\le R_A(t_A^\star)+2\varepsilon_m^A \le R_A(t^\star)+2\varepsilon_m^A$,
and the comparison assumption again gives
$R_A(t^\star)\le \tfrac{1}{c}(R(t^\star)+\kappa)$.
Combining the three inequalities yields the stated bound.
\end{proof}

\begin{remark}[One-sided relaxation]
\label{rem:one_sided_transfer}
The two-sided comparison in Corollary~\ref{cor:operator_transfer} is needed
only for the comparison to the image-domain oracle $R(t^\star)$. If one is
satisfied with an absolute bound on $R(\hat t)$ that uses the
measurement-domain oracle as the benchmark, only the one-sided condition
\begin{equation}\nonumber
\sup_{1\le t\le T}\bigl(R(t)-c\,R_A(t)\bigr)\le\kappa
\end{equation}
is required, in which case
\begin{equation}\nonumber
R(\hat t) \;\le\; c\,R_A(t_A^\star) + 2c\,\varepsilon_m^A + \kappa.
\end{equation}
This one-sided form is the one that follows from a lower singular-value
bound on $A$ alone (see Section~\ref{app:operator_interpretation} below).
\end{remark}

\subsection{Linear-operator interpretation}
\label{app:operator_interpretation}

For a linear $A$ with smallest non-zero singular value $\sigma_{\min}(A)$,
decompose the reconstruction error as
\(\hat x_t - x = u_t^{R} + u_t^{N}\) with
\(u_t^{R}\in\ker(A)^\perp\) and \(u_t^{N}\in\ker(A)\). Then
\(R_A(t) = \tfrac{1}{m}\|A u_t^{R}\|^2\) and
\(R(t) = \tfrac{1}{n}(\|u_t^{R}\|^2 + \|u_t^{N}\|^2)\). Using
\(\sigma_{\min}(A)^2\|u_t^{R}\|^2 \le \|Au_t^{R}\|^2\) gives the
\emph{one-sided} bound
\begin{equation}\nonumber
R(t) \;\le\; \frac{m}{n\,\sigma_{\min}(A)^2}\,R_A(t)
\;+\; \frac{\|u_t^{N}\|^2}{n}.
\end{equation}
This corresponds to the one-sided comparison in
Remark~\ref{rem:one_sided_transfer}, with
\(c = m/(n\sigma_{\min}(A)^2)\) and \(\kappa = \sup_t \|u_t^N\|^2/n\) being
controlled by whatever mechanism limits null-space artifacts along the
trajectory (e.g.\ DIP's spectral bias, or explicit regularization).

The reverse direction of the two-sided comparison required by
Corollary~\ref{cor:operator_transfer} is \emph{not} implied by
$\sigma_{\min}(A)$ alone. It depends on the
ratio $\sigma_{\max}(A)/\sigma_{\min}(A)$ on the active subspace, together
with a bound on $\|u_t^R\|^2$ along the trajectory. In particular, the
two-sided comparison holds when $A$ is approximately isometric on the error
directions traced out by the trajectory. For severely ill-conditioned $A$,
only the weaker one-sided guarantee in Remark~\ref{rem:one_sided_transfer}
is available without additional structural assumptions.

\subsection{Recovery of the denoising and mask-reference theorems}

Theorem~\ref{thm:operator_validation} reduces to the denoising and
mask-reference theorems for two natural choices of $(D, A)$.

\begin{enumerate}
    \item[(i)] \textbf{Denoising.} Take $D = y_1 = x + \eta_1$, $A=I$,
    $m = n$, and $y_2^A = x + \eta_2$. Then $R_A(t) = R(t)$,
    $V_t^A = V_t$, and Theorem~\ref{thm:operator_validation} reduces to
    Theorem~\ref{thm:single_reference_validation}. The transfer condition
    in Corollary~\ref{cor:operator_transfer} is satisfied with $c=1$,
    $\kappa = 0$.

    \item[(ii)] \textbf{Mask reference.} Condition on the mask $M$. Let
    $P_H : \mathbb{R}^n \to \mathbb{R}^m$, with $m = \sum_i H_i$, denote
    the restriction operator that extracts the held-out coordinates. Take the training data to be $D = (M, M\odot y)$, take $A = P_H$ as the
    validation operator, and take
    \begin{equation}\nonumber
        y_2^A = P_H y = P_H x + P_H \xi.
    \end{equation}
    The validation noise $\eta_2 = P_H \xi$ is independent of the
    retained training data $D$ because the pixel noise is independent
    across coordinates and the mask is conditioned on
    (Lemma~\ref{lem:heldout_independence}). With this identification,
    \[
        V_t^A
        = \frac{1}{m}\|P_H \hat x_t - P_H y\|^2
        = \frac{1}{m}\|H \odot (\hat x_t - y)\|^2
        = V_t^{\rm mask},
    \]
    so Theorem~\ref{thm:operator_validation} reduces to
    Theorem~\ref{thm:mask_validation}, and
    Corollary~\ref{cor:operator_transfer} reduces to the representativeness
    transfer in Corollary~\ref{cor:mask_full_risk}.
\end{enumerate}

The recovery of MR-ES uses the fact that the proof of
Theorem~\ref{thm:operator_validation} does not require $D$ to be of the form
$A(x) + \eta_1$; it only requires $\eta_2 \perp D$ and the
sub-Gaussian assumption on $\eta_2$. The training data (which generates
$D$) and the validation operator $A$ may therefore differ, as they do in
MR-ES.

The unifying object is the triple (training data $D$, validation operator
$A$, independent reference $y_2^A$ with conditional mean $A(x)$ and
sub-Gaussian noise $\eta_2$ independent of $D$).

\end{document}